\documentclass[letterpaper,11pt,twoside,leqno]{article}

\usepackage{cmap}
\usepackage[T1]{fontenc}
\usepackage[utf8]{inputenc}
\usepackage[letterpaper,textwidth=5.25in,textheight=8.85in,centering,headheight=13pt,headsep=18pt,footskip=30pt]{geometry}

\usepackage{lmodern}
\usepackage{libertine}

\usepackage{mathrsfs}
\usepackage{amssymb}
\usepackage{amsmath}
\usepackage{amsthm}
\usepackage{microtype}

\usepackage[usenames,dvipsnames]{xcolor}
\definecolor{myblue}{rgb}{0.21, 0.34, 0.74}
\definecolor{mygrey}{rgb}{0.55, 0.57, 0.67}
\definecolor{myred}{rgb}{0.79, 0.0, 0.09}

\definecolor{gblue}{HTML}{418BC0}
\definecolor{gpink}{HTML}{F17EA7}
\definecolor{ggreen}{HTML}{BDE0D2}
\definecolor{gpeach}{HTML}{F3CCB7}

\usepackage{graphicx}
\usepackage{enumitem}
\usepackage{bbm}
\usepackage{mathrsfs}
\usepackage{amsmath}
\usepackage{upgreek}
\usepackage{mathtools}
\usepackage{tabularx}
\usepackage{booktabs}
\usepackage{bm}
\usepackage[scr=boondoxo]{mathalfa}
\usepackage{stmaryrd}

\usepackage{comment}
\usepackage{float}
\usepackage{subcaption}
\usepackage[font=small,labelfont=bf,labelsep=period]{caption}

\numberwithin{figure}{section}
\numberwithin{table}{section}

\usepackage{tikz}
\usetikzlibrary{calc}

\usepackage[nottoc,notlot,notlof]{tocbibind}
\usepackage{cite}
\usepackage[hidelinks,bookmarksopen=true]{hyperref}
\usepackage[noabbrev,capitalize,nameinlink]{cleveref}

\usepackage{fancyhdr}

\newcommand{\CPAMshorttitle}{PERCEPTRONS AND LOCALIZATION OF ATTENTION'S LANDSCAPE}
\newcommand{\CPAMshortauthors}{A. \'{A}LVAREZ-L\'OPEZ, B. GESHKOVSKI AND D. RUIZ-BALET}

\fancypagestyle{cpamfirst}{%
  \fancyhf{}%
  \fancyfoot[L]{\parbox[t]{\textwidth}{\footnotesize}}
}

\usepackage{titlesec}
\titleformat{\section}
  {\centering\normalfont\large\bfseries}{\thesection}{0.75em}{}
\titlespacing*{\section}{0pt}{2.5ex plus 0.8ex minus 0.2ex}{1.3ex plus 0.2ex}

\titleformat{\subsection}
  {\normalfont\bfseries}{\thesubsection}{0.75em}{}
\titlespacing*{\subsection}{0pt}{2.0ex plus 0.6ex minus 0.2ex}{0.8ex plus 0.1ex}

\titleformat{\subsubsection}
  {\normalfont\bfseries}{\thesubsubsection}{0.75em}{}
\titleformat{name=\subsubsection,numberless}
  {\normalfont\bfseries}{}{0pt}{}
\titlespacing*{\subsubsection}{0pt}{1.7ex plus 0.5ex minus 0.2ex}{0.6ex plus 0.1ex}

\titleformat{\paragraph}[runin]
  {\normalfont\bfseries}{}{0pt}{}[.]
\titlespacing*{\paragraph}{0pt}{1.2ex plus 0.4ex minus 0.2ex}{0.5em}

\newtheoremstyle{cpamplain}%
  {8pt plus 2pt minus 2pt}
  {8pt plus 2pt minus 2pt}
  {\itshape}
  {}
  {\normalfont\bfseries}
  {.}
  {0.5em}
  {\MakeUppercase{\thmname{#1}}\ \thmnumber{#2}\thmnote{ {\normalfont(#3)}}}%
\newtheoremstyle{cpamdefinition}%
  {8pt plus 2pt minus 2pt}
  {8pt plus 2pt minus 2pt}
  {\normalfont}
  {}
  {\normalfont\bfseries}
  {.}
  {0.5em}
  {\MakeUppercase{\thmname{#1}}\ \thmnumber{#2}\thmnote{ {\normalfont(#3)}}}%

\theoremstyle{cpamplain}
\newtheorem{theorem}{Theorem}[section]
\newtheorem{proposition}[theorem]{Proposition}
\newtheorem{lemma}[theorem]{Lemma}
\newtheorem{corollary}[theorem]{Corollary}

\theoremstyle{cpamdefinition}

\newtheorem{remark}[theorem]{Remark}

\crefname{theorem}{Theorem}{Theorems}
\crefname{proposition}{Proposition}{Propositions}
\crefname{lemma}{Lemma}{Lemmas}
\crefname{corollary}{Corollary}{Corollaries}
\crefname{definition}{Definition}{Definitions}
\crefname{remark}{Remark}{Remarks}
\crefname{example}{Example}{Examples}
\crefname{problem}{Problem}{Problems}

\makeatletter
\renewcommand{\maketitle}{%
  \thispagestyle{cpamfirst}%
  \begingroup
  \centering
  \vspace*{0.56in}%
  {\Large\bfseries\@title\par}%
  \vspace{0.35in}%
  {\normalsize\@author\par}%
  \vspace{0.28in}%
  \endgroup
}
\makeatother

\newcommand{\CPAMauthor}[1]{{\normalsize #1}}
\newcommand{\CPAMaffiliation}[1]{{\itshape\small #1}}
\newcommand{\CPAMand}{{\normalsize AND}}

\renewenvironment{abstract}{%
  \begin{center}\bfseries Abstract\end{center}%
  \small
  \begin{list}{}{\leftmargin=0.55in\rightmargin=0.55in}%
  \item\relax
}{%
  \end{list}\normalsize
}

\newcommand{\R}{\mathbb{R}}
\renewcommand{\S}{\mathbb{S}}

\newcommand{\proj}{\boldsymbol{\mathsf{P}}^\perp}

\newcommand{\PP}{\mathcal{P}}

\newcommand{\Grad}{\nabla\!\!\!\!\nabla}

\DeclareMathOperator*{\argmax}{arg\,max}

\newcommand{\supp}{\operatorname{supp}}
\newcommand{\diff}{\,\mathrm{d}}

\let\alpha\upalpha
\let\beta\upbeta
\let\gamma\upgamma
\let\delta\updelta
\let\epsilon\upepsilon
\let\varepsilon\upvarepsilon
\let\zeta\upzeta
\let\eta\upeta
\let\theta\uptheta
\let\vartheta\upvartheta
\let\iota\upiota
\let\kappa\upkappa
\let\lambda\uplambda
\let\mu\upmu
\let\nu\upnu
\let\xi\upxi
\let\pi\uppi
\let\varpi\upvarpi
\let\rho\uprho
\let\varrho\upvarrho
\let\sigma\upsigma
\let\varsigma\upvarsigma
\let\tau\uptau
\let\upsilon\upupsilon
\let\phi\upphi
\let\varphi\upvarphi
\let\chi\upchi
\let\psi\uppsi
\let\omega\upomega
\let\Gamma\Upgamma
\let\Delta\Updelta
\let\Theta\Uptheta
\let\Lambda\Uplambda
\let\Xi\Upxi
\let\Pi\Uppi
\let\Sigma\Upsigma
\let\Upsilon\Upupsilon
\let\Phi\Upphi
\let\Psi\Uppsi
\let\Omega\Upomega

\renewcommand{\leq}{\leqslant}
\renewcommand{\geq}{\geqslant}
\renewcommand{\le}{\leqslant}
\renewcommand{\ge}{\geqslant}

\numberwithin{equation}{section}

\title{Perceptrons and localization of attention's mean-field landscape}

\author{%
\CPAMauthor{ANTONIO \'{A}LVAREZ-L\'OPEZ}\\[-1pt]
\CPAMaffiliation{Universidad Aut\'onoma de Madrid}\\[0.7em]
\CPAMand\\[0.7em]
\CPAMauthor{BORJAN GESHKOVSKI}\\[-1pt]
\CPAMaffiliation{Laboratoire Jacques-Louis Lions\\Inria \& Sorbonne Universit\'e}\\[0.7em]
\CPAMand\\[0.7em]
\CPAMauthor{DOM\`ENEC RUIZ-BALET}\\[-1pt]
\CPAMaffiliation{Universitat de Barcelona}%
}

\date{}

\begin{document}
\setlist[itemize,enumerate]{leftmargin=2em,itemsep=0pt,parsep=0pt}

%
%

\maketitle

%
%

\begin{abstract}
The forward pass of a Transformer can be seen as an interacting particle system on the unit sphere: time plays the role of layers, particles that of token embeddings, and the unit sphere idealizes layer normalization. In some weight settings the system can even be seen as a gradient flow for an explicit energy, and one can  make sense of the infinite context length (\emph{mean-field}) limit thanks to Wasserstein gradient flows. In this paper we study the effect of the perceptron block in this setting, and show that critical points are generically atomic and localized on subsets of the sphere.
\end{abstract}

\thispagestyle{cpamfirst}

\setcounter{tocdepth}{2}

\section{Introduction}

The distinctive operation of Transformers \cite{VSP17} is self-attention, which updates each token embedding by aggregating information from the others through a particular trainable weighted average. This mechanism is composed across depth with residual connections, layer normalization, and position-wise feed-forward blocks (typically two-layer perceptrons), producing highly structured yet poorly understood representation dynamics.

For studying signal propagation, it is useful to cast the dynamics of token representations as an interacting particle system \cite{geshkovski2025mathematical, sander2022sinkformers}. Motivated by the perspective of treating layer depth as a continuous time variable \cite{CRBD18} and by the fact that common normalization schemes (e.g., RMSNorm) keep embeddings on a compact manifold, one idealizes token embeddings as particles $x_i(t)$ evolving on the unit sphere $\S^{d-1}$ \cite{geshkovski2025mathematical}. 
Self-attention then becomes a state-dependent coupling: each particle moves toward a weighted average of the others, with weights given by a softmax kernel of pairwise inner products.  This interacting-particle viewpoint---which we adopt throughout---has gained traction in the past few years. It links Transformers to nonlinear consensus and collective-dynamics models~\cite{LLH20,DGCC21,geshkovski2023emergence}, and it also interfaces naturally with statistical mechanics treatments~\cite{cowsik2025geometric,giorlandino2025two,tiberi2024attentionpaths, cui2025phase}. While most of these works study idealized models, we emphasize that scaling laws derived by \cite{cowsik2025geometric}---subsequently studied in \cite{chen2025critical, bruno2026scalinglimitslongcontexttransformers, giorlandino2025two}---have been used in training large language models (OLMo2 7B \& 13B, \cite{olmo20242}). 

One object of interest in modern applications is very large context lengths. This is precisely where the \emph{mean-field limit} of the interacting particle system becomes relevant: as the number of tokens $n$ grows, the empirical measure of particles converges to a measure solving a nonlinear  partial differential equation on the sphere whose velocity field is the attention interaction. One can justify this mean-field limit by classical arguments  \cite{geshkovski2025mathematical}. 
A convenient advantage of the particle perspective is the transparent variational structure: the finite-$n$ particle dynamics can be written as a preconditioned (or weighted-metric) gradient flow of an interaction energy, yielding gradient \emph{descent} when the key-query-value-induced quadratic form is positive semidefinite and gradient \emph{ascent} otherwise \cite{geshkovski2025mathematical}. 
This gradient flow structure persists in the mean-field limit, where the limiting PDE is a (weighted) Wasserstein gradient flow \cite{jordan1998,otto2001}. This framework naturally connects to the study of equilibrium measures for nonlocal interaction energies, including the structural properties of their minimizers \cite{CanizoCarrilloPatacchini2015,CarrilloFigalliPatacchini2017}.

Exploiting this structure has already enabled a rigorous analysis of the \emph{ascent} (attractive) regime, where attention concentrates and synchronized clusters emerge, as well as of the \emph{descent} (repulsive) regime, where particles tend to equi-distribute uniformly~\cite{geshkovski2023emergence,geshkovski2025mathematical,chen2025quantitative,criscitiello2024synchronization,polyanskiy2025synchronization, castin2025unifiedperspectivedynamicsdeep, yu2024white, alcalde2025clustering}. These results have been extended to settings with very general trained weights~\cite{abella2025, burger2025, koubbi2024impact}, to decoder-only architectures with causal masking \cite{2024NikitaClustering, duerinckx2026kinetictheorytransformerslostinthemiddle, barbero2025llms}, to finer landscape phenomena such as metastability~\cite{geshkovski2024dynamic,bruno2025emergence,bruno2025multiscale, alcalde2025attention, altafini2025multistability}, to diffusive variants \cite{gerber2025formation, balasubramanian2025stationary, shalova2024solutions, peletier2025nonlinear}, and to stochastic scaling limits arising from random weights \cite{fedorov2026clustering, koubbi2026homogenized, agazzi2026stochastic}. 

A key architectural ingredient is still largely absent from the above variational mean-field picture: the feed-forward perceptron block. We emphasize that attention-only mean-field limits admit non-atomic stationary densities in repulsive regimes---indeed, the uniform distribution on the sphere is the global minimizer of the resulting energy! Practical evidence also suggests that perceptrons qualitatively change the geometry by counteracting attention-driven collapse, inducing an order--chaos transition \cite{cowsik2025geometric}. From the PDE viewpoint, the perceptron appears as an external drift, and we argue that it indeed reshapes the landscape, but not necessarily countering the above-cited studies: even when pure attention permits continuous equilibria, the stationary states of the coupled dynamics are generically discrete or singular.


\subsection*{Our contributions}

We study mean-field attention as a Wasserstein gradient flow for an energy that couples (unnormalized) attention interactions with a perceptron-induced potential. Our main contributions are:

\begin{itemize}

\item For ReLU perceptrons in $d=2$, every stationary measure has finite support ({\bf\Cref{thm: circle}}); for analytic activations (e.g., GeLU), the same holds for ``stable stationary measures'' ({\bf\Cref{thm: circle.gelu}}). In $d\ge 2$, stationary measures are necessarily singular, and for a dense open set of parameters they are purely atomic ({\bf\Cref{thm: any.d}}). 

\item In the descent (thus repulsive) regime, although the perceptron induces discreteness, we show an anti-concentration phenomenon: the mass of any cluster at the interaction scale $1/\sqrt{\beta}$ is bounded by a numerical constant ({\bf \Cref{thm: bound}}). Consequently, the limit measure (provided it exists) is atomic but cannot collapse to a single atom, in contrast to the attractive regime. We also show that the number of heavy atoms scales with $\sqrt{\beta}$ ({\bf\Cref{cor: bound}}).

\begin{figure}[!h]
    \centering

    \begin{subfigure}[b]{0.32\linewidth}
        \centering
        \includegraphics[width=\linewidth]{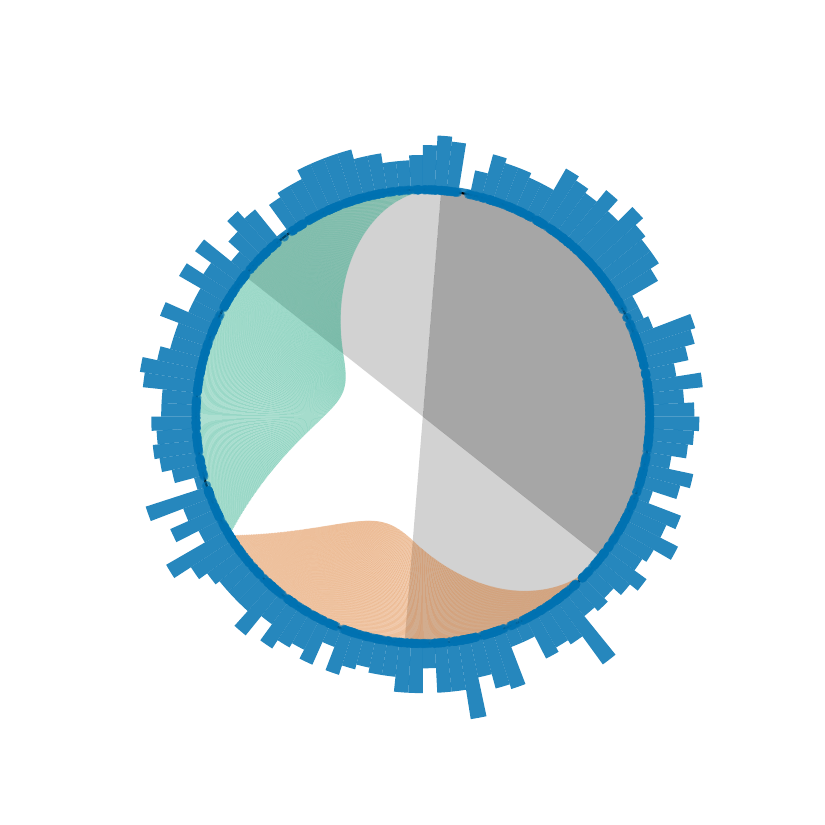}
    \end{subfigure}
    \begin{subfigure}[b]{0.32\linewidth}
        \centering
        \includegraphics[width=\linewidth]{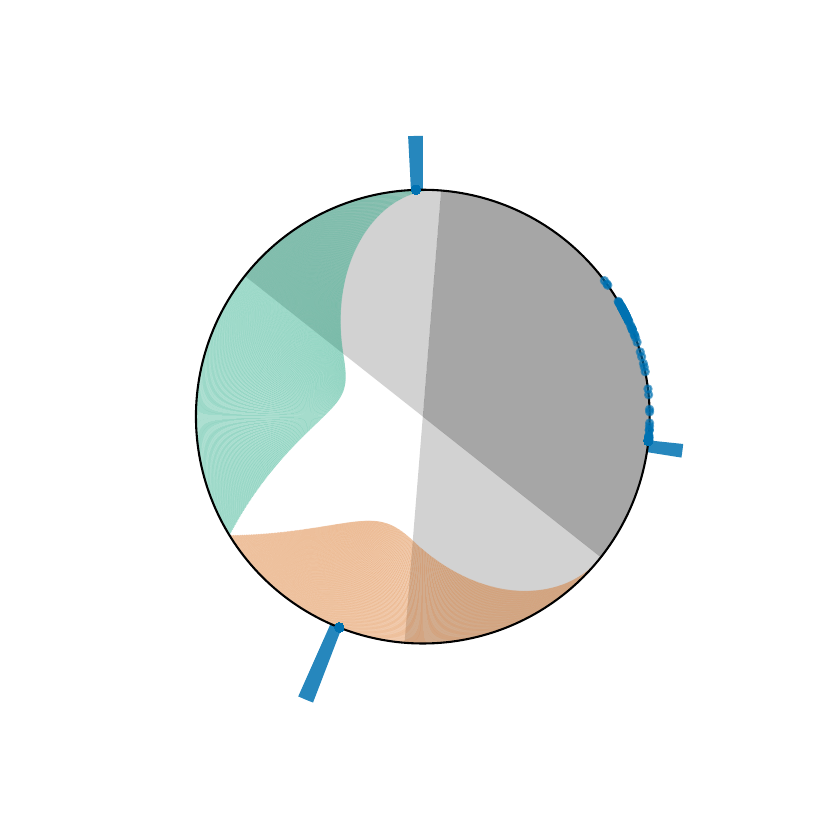}
    \end{subfigure}
    \begin{subfigure}[b]{0.32\linewidth}
        \centering
        \includegraphics[width=\linewidth]{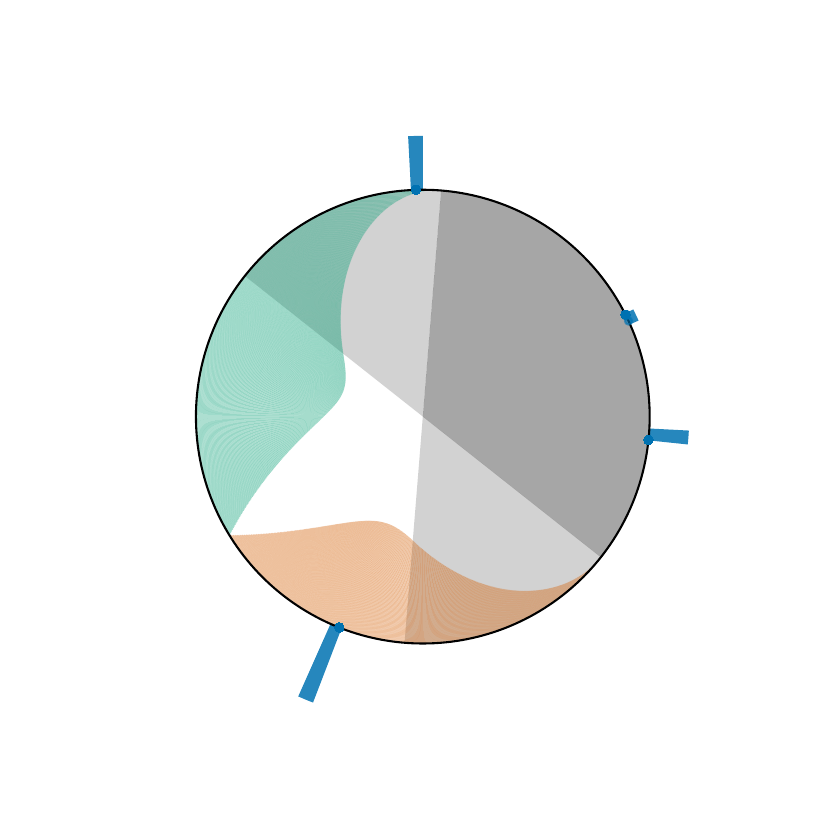}
    \end{subfigure}

    \vspace{0.5em}
    \begin{subfigure}[b]{0.32\linewidth}
        \centering
        \includegraphics[width=\linewidth]{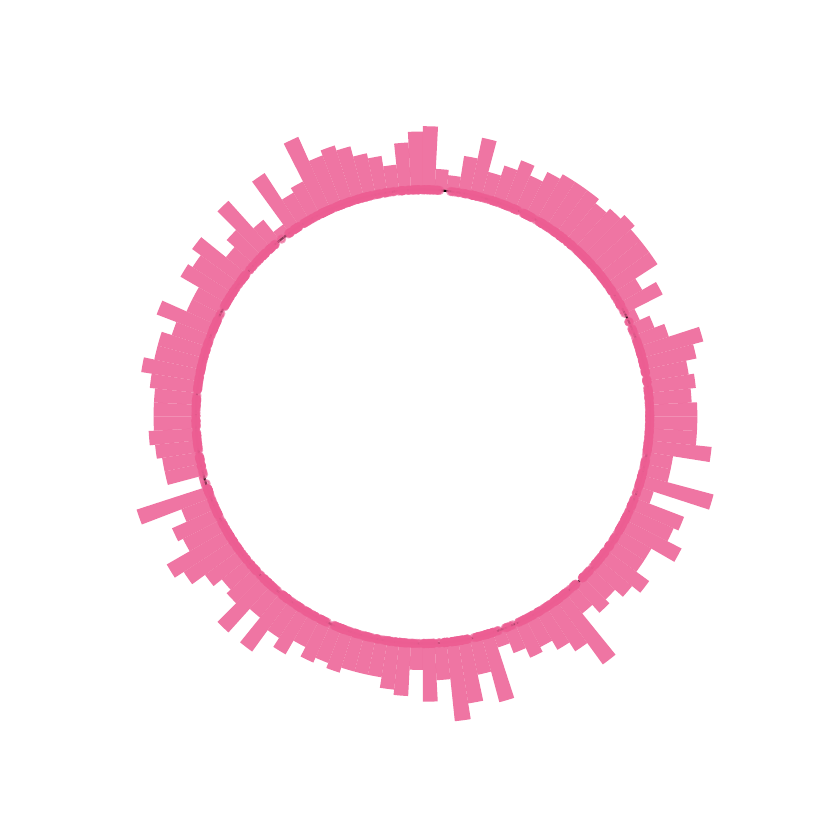}
    \end{subfigure}
    \hfill
    \begin{subfigure}[b]{0.65\linewidth}
        \centering
        \includegraphics[width=\linewidth]{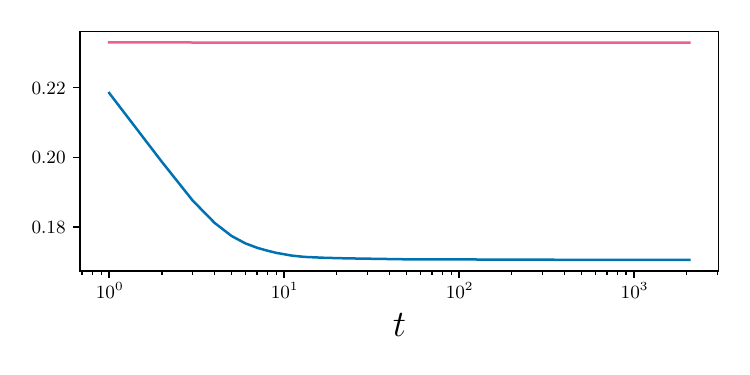}
    \end{subfigure}

    \caption{Gradient descent with ReLU perceptron in $d=2$, starting from the uniform measure, with $\beta=1$. 
    \textbf{Top:} particle histograms for the dynamics at initial, intermediate and final times.
    \textbf{Bottom left:} final configuration for pure self-attention without a perceptron.
    \textbf{Bottom right:} the energy $\mathsf E_{\beta,\upvartheta}$ with ({\color{gblue}blue}) and without ({\color{gpink}pink}) the perceptron.
    Background shading shows the perceptron landscape ({\color{ggreen}green}: $>0$; {\color{gpeach}orange}: $<0$). Areas where $a_j \cdot x + b_j > 0$ for some $j$ are shown in lighter gray, while the darkest gray regions correspond to the ``dead zones'' where the potential vanishes.
    }
    \label{fig:descent_relu_hist_energy}
\end{figure}

\item We characterize the extremal points of the energy in {\bf\Cref{prop: min.max}}. For ReLU perceptrons, maximizers reduce to a finite family of quadratic programs, yielding explicit solutions in some cases.
\end{itemize}

\section{Setup}\label{s:setup}

Let $\mathcal{P}(\mathbb{S}^{d-1})$ denote the space of Borel probability measures 
on $\mathbb{S}^{d-1}$ and $\sigma_d$ the uniform measure 
on $\mathbb{S}^{d-1}$. 
For a smooth $f: \mathbb{S}^{d-1} \to \mathbb{R}$ and $x\in\S^{d-1}$, the spherical gradient is 
\begin{equation*}
    \nabla f(x) \coloneqq \proj_x \nabla_{\mathbb{R}^d} \tilde{f}(x)
\end{equation*}
where $\proj_x \coloneqq I_d - xx^\top$ is the orthogonal projection onto $\mathsf{T}_x\S^{d-1}$ and $\tilde{f}:\R^d\to\R$ is a(ny) smooth extension of $f$; $-\mathrm{div}$ denotes the adjoint of $\nabla$. All integrals are taken over $\S^{d-1}$, which we drop to ease reading.

\subsection{Wasserstein gradient flows}

We study the evolution of measures driven by a functional $\mathsf{E}: \mathcal{P}(\mathbb{S}^{d-1}) \to \mathbb{R}$. 
A curve $(\mu(t))_{t \ge 0} \subset \mathcal{P}(\mathbb{S}^{d-1})$ is a \emph{Wasserstein gradient flow}---in the steepest descent convention---for $\mathsf{E}$ if it is locally absolutely continuous in $(\mathcal P(\mathbb{S}^{d-1}), W_2)$ and there exists a velocity field $v(t)\in L^2(\mu(t);\mathsf{T}\mathbb{S}^{d-1})$ such that
\begin{equation}\label{eq:continuity_eq}
  \partial_t \mu(t) + \mathrm{div}\left(\mu(t) v(t)\right)=0,
\end{equation}
in the sense of distributions $\mathcal{D}'(\R_{\geq0}\times\mathbb{S}^{d-1})$, with velocity given by the gradient of the first variation,
\[
v(t) = -\nabla \frac{\delta \mathsf{E}}{\delta \mu}[\mu(t)].
\]
We recall that the first variation $\frac{\delta \mathsf{E}}{\delta \mu}[\mu]$ is defined (up to an additive constant) by the relation
\[
\frac{\mathrm{d}}{\mathrm{d}\varepsilon} \mathsf{E}(\mu + \varepsilon \chi)\Big|_{\varepsilon=0}
= \int\frac{\delta \mathsf{E}}{\delta \mu}[\mu]\diff \chi,
\]
for signed measures $\chi$ with $\chi(\mathbb{S}^{d-1})=0$ and such that $\mu+\varepsilon\chi\in\mathcal P(\mathbb{S}^{d-1})$ for $|\varepsilon|$ small. In the Wasserstein geometry, the relevant perturbations are transport ones $\chi=-\mathrm{div}(\mu\xi)$ with $\xi\in L^2(\mu;\mathsf{T}\mathbb{S}^{d-1})$, which identifies the Wasserstein gradient as above; see \cite{ambrosioGradientFlowsMetric2008}.

\subsection{Critical points}

We are primarily interested in the asymptotic behavior of these flows, which relates to the critical points of $\mathsf{E}$. A measure $\mu \in \mathcal{P}(\mathbb{S}^{d-1})$ is a \emph{critical point} of $\mathsf{E}$, or stationary solution to \eqref{eq:continuity_eq}, if
\begin{equation} \label{eq:critical_point_def}
    \nabla\frac{\delta\mathsf{E}}{\delta\mu}[\mu](x)=0 \qquad \text{for } \mu\text{-a.e. } x.
\end{equation}
We are also interested in second-order Wasserstein critical points. Recall that
\begin{equation}\label{eq:otto.tangent}
 \mathsf T_\mu \mathcal P(\mathbb S^{d-1}) \coloneqq \overline{\{\nabla \phi\colon \phi\in C^\infty(\mathbb S^{d-1})\}}^{\,L^2(\mu)}    
\end{equation}
is the tangent space at $\mu$, a Hilbert subspace of $L^2(\mu;\mathsf{T}\mathbb S^{d-1})$. Suppose $\mu$ lies in the regular (Riemannian) part of $(\mathcal P_2(\mathbb S^{d-1}),W_2)$ and $\mathsf{E}$ is twice differentiable in the $W_2$-sense at $\mu$ \cite[Ch.~15]{villani2009optimal}. Then the \emph{Wasserstein Hessian} at $\mu$ in the direction $\xi \in \mathsf T_\mu \mathcal P(\mathbb S^{d-1})$ is defined by
\begin{equation*}
\mathrm{Hess}_\mu \mathsf{E}(\xi,\xi)
\coloneqq \left.\frac{\diff ^2}{\diff  t^2}\right|_{t=0} \mathsf{E}(\mu(t)),
\end{equation*}
where $(\mu(t))$ is the $W_2$-geodesic emanating from $\mu$ with initial velocity field represented by $\xi=\nabla\phi$, i.e. (for $|t|$ small) $\mu(t)=(T(t))_\#\mu$ with $T(t)(x)=\exp_x(t\nabla\phi(x))$.\footnote{Here $\exp_x$ denotes the exponential map on $\mathbb S^{d-1}$. In general (e.g. for singular $\mu$), geodesics need not be uniquely determined by an $L^2(\mu)$ velocity field; in that case one may instead work with directional second variations along Monge perturbations, but we will not need this level of generality here.}

A critical point $\mu$ is \emph{second-order positive-definite} (SOPD) if the Wasserstein Hessian at $\mu$ is well-defined and
\begin{equation}\label{eq:2order}
 \mathrm{Hess}_\mu \mathsf{E}(\xi,\xi)\ge 0\qquad\text{ for all }\xi\in\mathsf T_\mu \mathcal P(\mathbb S^{d-1}),
\end{equation}
and it is \emph{strictly} SOPD if, moreover, there is $\kappa>0$ such that
\begin{equation}\label{eq:strict2order}
\mathrm{Hess}_\mu \mathsf E(\xi,\xi)\ge 
\kappa\|\xi\|_{L^2(\mu)}^2.
\end{equation}
A SOPD critical point is \emph{degenerate} if it is not strictly SOPD, i.e. if there exists $\xi\neq 0$ with $\mathrm{Hess}_\mu \mathsf{E}(\xi,\xi)=0$. If \eqref{eq:2order} fails (there exists $\xi$ with $\mathrm{Hess}_\mu \mathsf{E}(\xi,\xi)<0$), we call $\mu$ unstable (a saddle/maximum direction).

We focus on SOPD Wasserstein critical points because they ought to capture the appropriate notion of \emph{stable} equilibria for Wasserstein gradient (descent) flows and because quantitative convergence estimates are typically driven by second-order information (e.g. PL-type inequalities) around such points \cite{villani2009optimal,otto2001,chen2025quantitative}. 
In our coupled attention--perceptron setting, this perspective also highlights a surprising phenomenon: even in the repulsive regime, stable stationary states are forced to be discrete, so the mean-field flow can evolve from a smooth initialization toward a genuinely clustered configuration rather than a diffuse equilibrium.

\subsection{Mean-field formulation of Transformers}

As in \cite{geshkovski2025mathematical}, we model the layer-wise evolution of token embeddings  as the flow of an interacting particle system on $\mathbb{S}^{d-1}$, where ``time'' indexes the layers of the architecture.  At the mean-field level, the distribution of a particle under pure self-attention dynamics follows the Wasserstein gradient flow of
\begin{equation*}
\mathsf{E}_\beta[\mu] \coloneqq \frac{1}{2\beta}\iint e^{\beta x\cdot y}\diff \mu(x)\diff \mu(y).
\end{equation*}
Its first variation is
\begin{equation*}
    \frac{\delta \mathsf{E}_\beta}{\delta \mu}[\mu](x)=\frac{1}{\beta}\int e^{\beta x\cdot y}\diff \mu(y),
\end{equation*}
and the Wasserstein gradient in \eqref{eq:continuity_eq} reads
\begin{equation*}
    \nabla\frac{\delta \mathsf{E}_\beta}{\delta \mu}[\mu](x)=\proj_x \int e^{\beta x\cdot y} y\diff\mu(y).
\end{equation*}
This setting corresponds to the toy model analyzed in \cite{geshkovski2025mathematical} in which the key-query weights $K, Q$ and value weights $V$ are constant in time and satisfy $B\coloneqq Q^\top K=V=I_d$.
The framework and results can however be generalized to very general weights \cite{abella2025, geshkovski2025mathematical, karbevski2025key}---see also \Cref{ss:implem} below.
    
As discussed in \cite{geshkovski2025mathematical}, the perceptron component can be incorporated at the mean-field level as an additional drift field, yielding 
\begin{equation}\label{eq:WGF-main}
\partial_t \mu(t) \pm \mathrm{div}\left(\mu(t)\left(\nabla \frac{\delta \mathsf{E}_\beta}{\delta \mu}[\mu(t)] + \mathsf{u}_\upvartheta\right)\right)=0,
\end{equation}
where
\begin{equation*}
\mathsf{u}_\upvartheta(x) \coloneqq \proj_x \sum_{j\in\llbracket1,d\rrbracket} \boldsymbol{\omega}_j\sigma\left(a_j\cdot x + b_j\right),
\end{equation*}
with weights $\upvartheta=(a_j, \boldsymbol{\omega}_j, b_j)_{j\in\llbracket1,d\rrbracket}$, where $a_j,\boldsymbol{\omega}_j\in\R^d$ and $b_j \in \mathbb{R}$. We refer to the $+$ case in \eqref{eq:WGF-main} as \emph{ascent}, and $-$ as \emph{descent}.
One typically uses 
\[
\sigma(s)=s_+\text{ (ReLU) or }\sigma(s)=\frac{s}{2}\left(1+\operatorname{erf}\left(\tfrac{s}{\sqrt{2}}\right)\right)\text{ (GeLU)}.
\]
Our analysis distinguishes between these two cases. For simplicity, we henceforth omit the biases $b_j$, discussing their inclusion later in \Cref{rem: ext}.

When $\mathsf{u}_\upvartheta$ derives from a scalar potential, the coupled dynamics remain a Wasserstein gradient flow. This holds if and only if the output weights $\boldsymbol{\omega}_j$ are collinear with the input weights $a_j$, meaning $\boldsymbol{\omega}_j=\omega_j a_j$ for some scalars $\omega_j\in\R$. Henceforth, we assume this is the case and denote the parameters by $\upvartheta=(a_j,\omega_j)_{j\in\llbracket 1,d \rrbracket}\in(\R^{d+1})^d$.

Under this condition, we can choose a primitive  $\upvarphi$ with $\upvarphi'(s)=2\sigma(s)$ and set
\begin{equation} \label{eq: primitive.field}
\mathsf{v}_\upvartheta(x) \coloneqq \sum_{j\in\llbracket1,d\rrbracket} \omega_j\upvarphi\left(a_j\cdot x\right),
\end{equation}
so that $\frac{1}{2}\nabla\mathsf{v}_\upvartheta(x)=\mathsf{u}_\upvartheta(x)$. The full energy is then
\begin{equation*}
\mathsf{E}_{\beta,\upvartheta}[\mu]
\coloneqq \frac{1}{2\beta}\iint e^{\beta x\cdot y}\diff \mu(x)\diff \mu(y)
+\frac{1}{2}\int \mathsf{v}_\upvartheta(x)\diff \mu(x),
\end{equation*}
and \eqref{eq:WGF-main} is precisely its Wasserstein gradient flow, since
\begin{equation*}
\nabla\frac{\delta\mathsf{E}_{\beta,\upvartheta}}{\delta\mu}[\mu](x) = \int e^{\beta x\cdot y}\proj_x y\diff\mu(y) + \mathsf{u}_\upvartheta(x).
\end{equation*}
Whenever $\sigma$ is continuous, \eqref{eq:critical_point_def} extends to the entire support:
\begin{equation} \label{eq: steady.state}
\nabla\frac{\delta\mathsf{E}_{\beta,\upvartheta}}{\delta\mu}[\mu](x)=0\qquad \text{for all } x\in\supp\mu.
\end{equation}

\subsection{Practical considerations} \label{ss:implem}

While our toy model captures the core interaction dynamics, we distinguish three differences with respect to Transformers used in real-world applications.

\paragraph{Perceptrons}

The drift given by the perceptron is a gradient field only under the specific weight symmetries assumed in \eqref{eq: primitive.field}. More general weights $\boldsymbol{\omega}_j$ could be interpreted as preconditioners, much like what is discussed in Section 2.2 in \cite{alcalde2025attention}.

\paragraph{Self-attention}

\textbf{\emph{(i)}} Practical implementations replace  $\beta x \cdot y$ with a general bilinear form $x^\top Q^\top K y$ involving query ($Q$) and key ($K$) matrices. 
For technical clarity, we work in the isotropic case $Q^\top K=\beta I_d$ with $ V=\pm I_d$, but our main results can be extended to the setting where $Q^\top K$ is symmetric and invertible; see \Cref{rem: ext} and the discussion after \Cref{lem: quadpol}. The sign of $V$ determines whether the dynamics follow gradient ascent or descent, but this does not change the stationary points.

\textbf{\emph{(ii)}} In practice, attention scores are also normalized via a softmax as
\begin{equation*}
\nabla\log\int e^{\beta x\cdot y}\diff\mu(y)
= \frac{1}{\frac{\delta\mathsf E_\beta}{\delta\mu}[\mu](x)} \nabla\frac{\delta\mathsf E_\beta}{\delta\mu}[\mu](x).
\end{equation*}
This field corresponds to a weighted Wasserstein gradient $\Grad$ obtained by rescaling the standard variation by the strictly positive weight
$\mathsf{w}[\mu](x) \coloneqq \frac{\delta \mathsf{E}_\beta}{\delta \mu}[\mu](x).$ 
The resulting dynamics are
\begin{equation}\label{eq:full-WGF-main}
\partial_t\mu(t) + \mathrm{div}\left(\mu(t) \left( \Grad\mathsf{E}_{\beta}[\mu(t)] + \mathsf{u}_\upvartheta \right)\right) = 0,  
\end{equation}
where $\Grad\mathsf{E}\coloneqq\frac{1}{\mathsf{w}[\mu]} \nabla \frac{\delta \mathsf{E}}{\delta \mu}[\mu]$. They
retain the structure of a gradient flow (under a conformally equivalent metric). Since $\mathsf{w}>0$, the stationarity condition \eqref{eq: steady.state} becomes:
\begin{equation}\label{eq: fulltrans.stat}
   \nabla \frac{\delta \mathsf{E}_{\beta}}{\delta \mu}[\mu] + \mathsf{w}[\mu]\mathsf{u}_\upvartheta = 0 \qquad \text{on } \supp\mu.
\end{equation} 
To facilitate the analysis, we focus on unnormalized attention in the main text; the analogous results for \eqref{eq:full-WGF-main}, which are qualitatively similar, are detailed in {\bf\Cref{ss:normalized}}.

\section{Results}

We now state our main results, which characterize the stationary configurations of the mean-field dynamics \eqref{eq:WGF-main}.

\subsection{Atomicity of critical points}

We begin with the simplest case of the ReLU perceptron in $d=2$. In this setting, the non-analyticity of the potential forces any stationary measure to have finite support.

\begin{theorem} \label{thm: circle}
    Let $d=2$ and $\beta>0$, and fix $\sigma(s)=s_+$. Assume that the weights $\upvartheta$ are such that $\mathsf{v}_\upvartheta$ is not real-analytic\footnote{This discards the pathological cases where $\mathsf{v}_\upvartheta\equiv0$, or specific weight symmetries for which $\mathsf{v}_\upvartheta$ effectively becomes a quadratic trigonometric polynomial.} on $\mathbb{S}^1$. 
    Then any $\mu \in \mathcal{P}(\mathbb{S}^{1})$ satisfying the stationarity condition \eqref{eq: steady.state} is purely atomic and has finite support.
\end{theorem}

When the activation is real-analytic (e.g. GeLU), simply looking at solutions of \eqref{eq: steady.state} is not enough due to the analyticity of the potential. However, the same conclusions hold for strict SOPD Wasserstein critical points.

\begin{theorem} \label{thm: circle.gelu}
  Let $d=2$, $\beta>0$. Fix weights $\upvartheta$ and a real-analytic function $\sigma:\mathbb{R}\to\mathbb{R}$. If $\mu \in \mathcal{P}(\mathbb{S}^{1})$ is a strict SOPD Wasserstein critical point of $\mathsf{E}_{\beta,\upvartheta}$ in the sense of \eqref{eq:strict2order}, then $\mu$ is purely atomic and has finite support. 
    
    In other words, if a SOPD Wasserstein critical point has infinite support, then it must be degenerate.
\end{theorem}

In higher dimensions the landscape is more complicated, but the following theorem establishes that stationary measures are always singular and atomicity remains generic.

\begin{theorem}\label{thm: any.d}
    Let $d\ge 2$ and fix $\mu\in\mathcal P(\S^{d-1})$.

    \smallskip
    \noindent \textbf{(i)} Fix $\beta>0$. If $\mu$ is stationary for $\sigma(s)=s_+$ and $\upvartheta$ is such that $\mathsf{v}_\upvartheta$ is not real-analytic, or $\mu$ is a strict SOPD critical point for a real-analytic $\sigma$, then $\sigma_d(\supp\mu)=0$. In particular, $\mu$ is singular with respect to $\sigma_d$.

    \smallskip
    \noindent \textbf{(ii)} If $\sigma$ is real-analytic and $\sigma(s)\neq0$ for $s\neq0$, then there exists an open and dense set $U_\mu \subset \R_{>0}\times (\R^{d+1})^d$ such that, if $(\beta, \upvartheta)\in U_\mu$ and $\mu$ is stationary with respect to $(\beta, \upvartheta)$, then $\mu$ is purely atomic with finite support.

    \smallskip
    \noindent \textbf{(iii)} If $\sigma(s)=s_+$, then there exists a dense set $U_\mu \subset \R_{>0}\times (\R^{d+1})^d$ such that, if $(\beta,\upvartheta)\in U_\mu$ and $\mu$ is stationary with respect to $(\beta,\upvartheta)$, then the restriction of $\mu$ to the active regions $\bigcup_{j} \{ x:a_j \cdot x > 0\}$ is purely atomic with at most countably many atoms.
\end{theorem}

\begin{remark}\label{rem: ext}
The results above easily extend  to:

\smallskip\noindent\emph{(i)} biases $a_j \cdot x + b_j$ inside the perceptron, provided at least one of the hyperplanes $\{x: a_j \cdot x + b_j = 0 \}$ intersects $\S^{d-1}$ in more than one point (equivalently, $|b_j| < \|a_j\|$ for some $j$).

\smallskip\noindent\emph{(ii)} symmetric invertible $B$ instead of $\beta I_d$ in the interaction energy; see \Cref{rem: general-attention,rem:unifiedlog-B}.
\end{remark}

\subsection{Anti-concentration bound}


While stationary measures have a discrete nature in the repulsive case, the strict concavity of the kernel\footnote{This has also been used in \cite{geshkovski2025mathematical, geshkovski2024number} in related but different contexts.} $\theta\mapsto e^{\beta \cos\theta}$ on $(-\beta^{-\frac12}, \beta^{-\frac12})$ for $\beta\to\infty$ ensures that they cannot be too concentrated.

\begin{figure}[h!]
    \centering
    \includegraphics[scale=0.65]{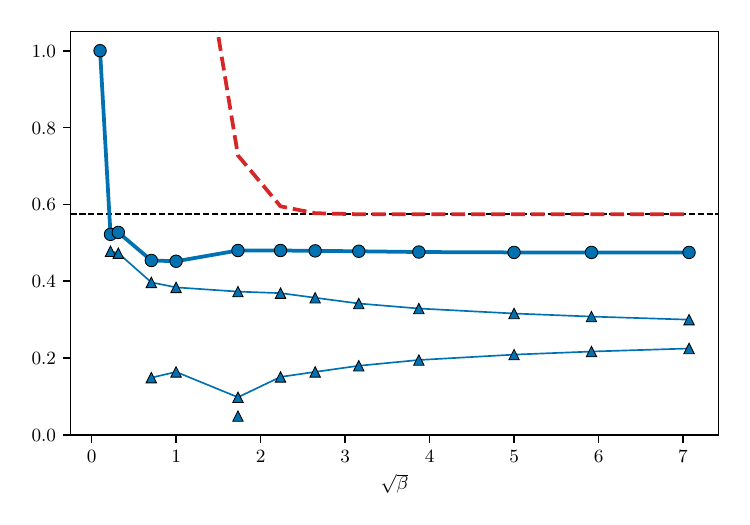}
    \caption{Cluster masses (in {\color{gblue}blue}, the largest being the thickest) at final time across  $\sqrt{\beta}$ for gradient descent with GeLU perceptron, initialized with $N=1000$ points of mass $10^{-3}$ (see \Cref{sec:numerics} for setup). The horizontal and {\color{red}red} dashed lines represent the numerical term and the full upper bound in \eqref{eq: conclusion}, respectively.}
\end{figure}

\begin{theorem}\label{thm: bound}
Let $d=2$, $\beta>0$ and $\sigma$ globally Lipschitz with constant $1$ and $\sigma(0)=0$ for simplicity. Consider any SOPD Wasserstein critical point of the form
\begin{equation}\label{eq: atomic.thm.bound}
    \mu = \sum_{i\in\llbracket 1, N\rrbracket} m_i \delta_{\theta_i}\in\mathcal{P}(\mathbb{S}^{1})
\end{equation}
where $N\ge 2$, the weights $m_i>0$ sum to $1$, and the points $\theta_i\in[0,2\pi)$ are pairwise distinct. 

Suppose there exists a subset $\mathcal S \subseteq \llbracket 1, N\rrbracket$ with $n=|\mathcal S|\geq2$ satisfying the cluster condition
\begin{equation} \label{eq: pairwise.distance}
\max_{i,j\in \mathcal S}\min_{k\in\mathbb Z}|\theta_i-\theta_j+2\pi k| \le \frac{1}{2\sqrt{\beta}}.
\end{equation}
Then, for $\beta$ sufficiently large, the total mass of this cluster satisfies
\begin{equation}\label{eq: conclusion}
\sum_{i\in \mathcal S} m_i \le 0.5742+ O\left(e^{-\beta}\right).
\end{equation}
Moreover,  if $\upvartheta=(\omega_j,a_j)_j$ satisfies
\begin{equation}\label{eq: theta.bound}
|\omega_1|\cdot\|a_1\|^2 + |\omega_2|\cdot\|a_2\|^2< 0.16547,
\end{equation}
then $\mathcal S=\llbracket 1, N\rrbracket$ satisfying \eqref{eq: pairwise.distance} cannot hold, for any $\beta>0$.
\end{theorem}

It ensues that

\begin{corollary}\label{cor: bound}
Let $\mu$ be as in \Cref{thm: bound}. Suppose that
\[
\operatorname{supp}\mu \subset \bigcup_{j\in\llbracket1, M\rrbracket} I_j, \qquad \text{where } \max_{j\in\llbracket1, M\rrbracket}|I_j| \eqqcolon L < 2\pi,
\]
for some integer $M \ge 1$. Let $N_\varepsilon$ be the number of atoms with mass $\ge \varepsilon$ for $\varepsilon>0$. Then, for $\beta$ large enough,
\[
N_\varepsilon \le \frac{M}{\varepsilon}\left(1+2L\sqrt{\beta}\right)\left( 0.5742 + O\left(e^{-\beta}\right) \right).
\]
\end{corollary}

We finish with the natural question of \emph{extremal} points.

\begin{proposition}\label{prop: min.max}
Let $d\geq 2$, $\beta>0$, $\upvartheta=(a_j, \boldsymbol{\omega}_j)_{j\in\llbracket1,d\rrbracket}$ and $\upvarphi'=2\sigma$. 

\noindent \textbf{(i)} The global maximizers of $\mathsf{E}_{\beta,\upvartheta}$ are exactly those $\mu=\delta_x$ with
\begin{equation*}
x\in\argmax_{y\in\S^{d-1}}\sum_{j\in\llbracket1,d\rrbracket} \omega_j\upvarphi(a_j\cdot y).
\end{equation*}
    
\noindent \textbf{(ii)} $\mathsf{E}_{\beta,\upvartheta}$ has a unique global minimizer $\mu_\ast$ which is 
    invariant under rotations that fix every $a_j$ such that $\omega_j\ne0$. 
\end{proposition}

\subsection{Maximizers for ReLU perceptrons}

We compute the maximizers in \Cref{prop: min.max} for $\sigma(s)=s_+$. Define
\begin{equation}\label{eq: defZd}
    \mathscr Z\coloneqq\{x\in\S^{d-1}:  a_j\cdot x=0 \text{ for some } j\in\llbracket1,d\rrbracket\}
\end{equation}
and let $I$ be any connected component of $\S^{d-1}\setminus\mathscr Z$ with active index set
\begin{equation}\label{eq: defJI}
    J_I\coloneqq\{j\in\llbracket1,d\rrbracket: a_j\cdot x>0 \text{ for all } x\in I\}.
\end{equation}
The signs of $x\mapsto a_j\cdot x$ are fixed on $I$, hence
\[
\mathsf{v}_\upvartheta(x)=x^\top B_{I}\,x,\qquad 
B_{I}\coloneqq \sum_{j\in J_{I}} \omega_j\, a_j a_j^\top\quad. 
\] Thus \Cref{prop: min.max} reduces to solving and comparing a finite collection of constrained quadratic programs:
\begin{equation}\label{eq:v-decomp}
\max_{x\in\mathbb S^{d-1}}\mathsf{v}_\upvartheta(x)
=\max_\alpha\ \max_{x\in \overline{I}} x^\top B_{I}\,x.
\end{equation}
This yields finitely many candidates, one for each connected component of $\S^{d-1}\setminus \mathscr{Z}$. On each $\overline I$, either a principal eigenvector of $B_I$ that lies in $I$ maximizes $x^\top B_I x$, or else all maximizers lie on $\partial I$. 

Across regions, the maximizers of $\mathsf v_\upvartheta$ may form a singleton, a finite set, or a continuum. In special cases they can be described explicitly:

\smallskip
\noindent \textbf{(i) Collinear directions.} Suppose \(a_j=\alpha_j\,a\) with \(a\in\R^d\) and \(\alpha_j\in\R\). Then
\[
\mathsf v_\upvartheta(x)
= (a \cdot  x)_+^2  \sum_{\alpha_j>0}\omega_j\alpha_j^2
+
(a \cdot  x)_-^2  \sum_{\alpha_j<0} \omega_j\alpha_j^2.
\]
Hence
\[
\max_{x\in\mathbb S^{d-1}}\mathsf{v}_\upvartheta(x)
=\|a\|^2\max\left\{\sum_{\alpha_j>0}\omega_j\alpha_j^2,\  \sum_{\alpha_j<0}\omega_j\alpha_j^2, 0\right\}.
\]
If $\max\mathsf{v}_\upvartheta>0$, the maximizer is \(\delta_{a/\|a\|}\) or \(\delta_{-a/\|a\|}\), according to which of the two sums is larger; both are maximizers in the tie case. Otherwise, any $\delta_x$ with $a\cdot x=0$ is a maximizer.

\smallskip
\noindent \textbf{(ii) Diagonal directions.} If \(a_j=\alpha_j e_j\) with $\alpha_j\in\R$, and \(\omega_j\equiv 1\), then
\[
\max_{x\in\mathbb S^{d-1}}\mathsf{v}_\upvartheta(x)=\max_{x\in\S^{d-1}}\sum_{j\in\llbracket1,d\rrbracket} (\alpha_j x_j)_+^2=\max_{j\in\llbracket1,d\rrbracket}|\alpha_j|^2,
\]
and the maximizers are those $\delta_x$ with \(x_k=0\) if \( |\alpha_k|<\max_j|\alpha_j|\) and \( \alpha_k x_k\ge0\) if \( |\alpha_k|=\max_j|\alpha_j|\). If the $\max_j$ is attained at a unique index $k$, then $x=\mathrm{sign}(\alpha_k)e_k$.

\smallskip
\noindent \textbf{(iii) Nonnegative.} Assume \(a_j\ge 0\) entrywise and \(\omega_j\equiv 1\). For any \(x\in\mathbb S^{d-1}\),
\begin{equation*}
\sum_{j\in\llbracket1,d\rrbracket} (a_j\cdot x)_+^2 \le \sum_{j\in\llbracket1,d\rrbracket} (a_j\cdot x)^2 = \|Ax\|_2^2\ \le\ \|A\|_2^2,    
\end{equation*}
where $A$ is the matrix with rows $a_j$. By Perron–Frobenius on the entrywise nonnegative $A^\top A$, we may choose a top right singular vector $v_1\ge0$, ensuring $Av_1\ge0$. Thus $x=v_1$ saturates both inequalities, so $\delta_{v_1/\|v_1\|}$ is a maximizer (unique if the top singular value is simple).

\subsection{Discussion}\label{ss: disc}

\smallskip\noindent\emph{(i) Domain geometry.} The analysis is conducted on $\mathbb{S}^{d-1}$. Despite the unique-continuation arguments naturally extend to other real-analytic compact manifolds, ruling out continuous measures globally (as in \Cref{thm: circle} and \Cref{thm: any.d}) heavily relies on the spectral inversion of the isotropic interaction kernel $e^{\beta x\cdot y}$ via spherical harmonics (\Cref{lem: quadpol}). Extending this contradiction mechanism to general domains remains an open problem.

\smallskip\noindent\emph{(ii) Non-conservative drifts.} When the assumption of collinear weights is dropped, the MLP drift is no longer a gradient. The stationary condition is the weaker transport equation \(\operatorname{div}(\mu v)=0\), rather than the pointwise condition of \eqref{eq: steady.state}. Handling this generalized condition falls outside the scope of our current techniques.

\smallskip\noindent\emph{(iii) General attention kernels.} The atomicity results require $\mu \mapsto \int K(\cdot, y)\diff\mu(y)$ to be real-analytic and injective (or provide a comparable spectral inversion). On the other hand, the quantitative anti-concentration bounds in \Cref{thm: bound} depend on the specific local concavity of $e^{\beta\cos\theta}$. Extending these conclusions to completely general attention kernels requires corresponding analytic and local-curvature properties.

\section{Simulations}\label{sec:numerics}

We complement our theory with simulations\footnote{Code can be found at \href{https://github.com/antonioalvarezl/2026-MLP-Attention-Energy}{https://github.com/antonioalvarezl/2026-MLP-Attention-Energy}.}. 

\begin{figure}[!h]
    \centering

    \begin{subfigure}[b]{\linewidth}
        \centering
        \includegraphics[width=0.32\linewidth]{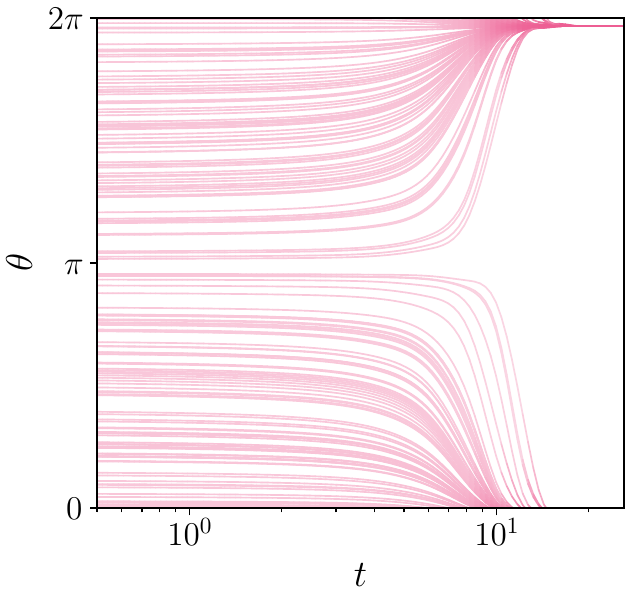}
    \hfill
        \centering
        \includegraphics[width=0.32\linewidth]{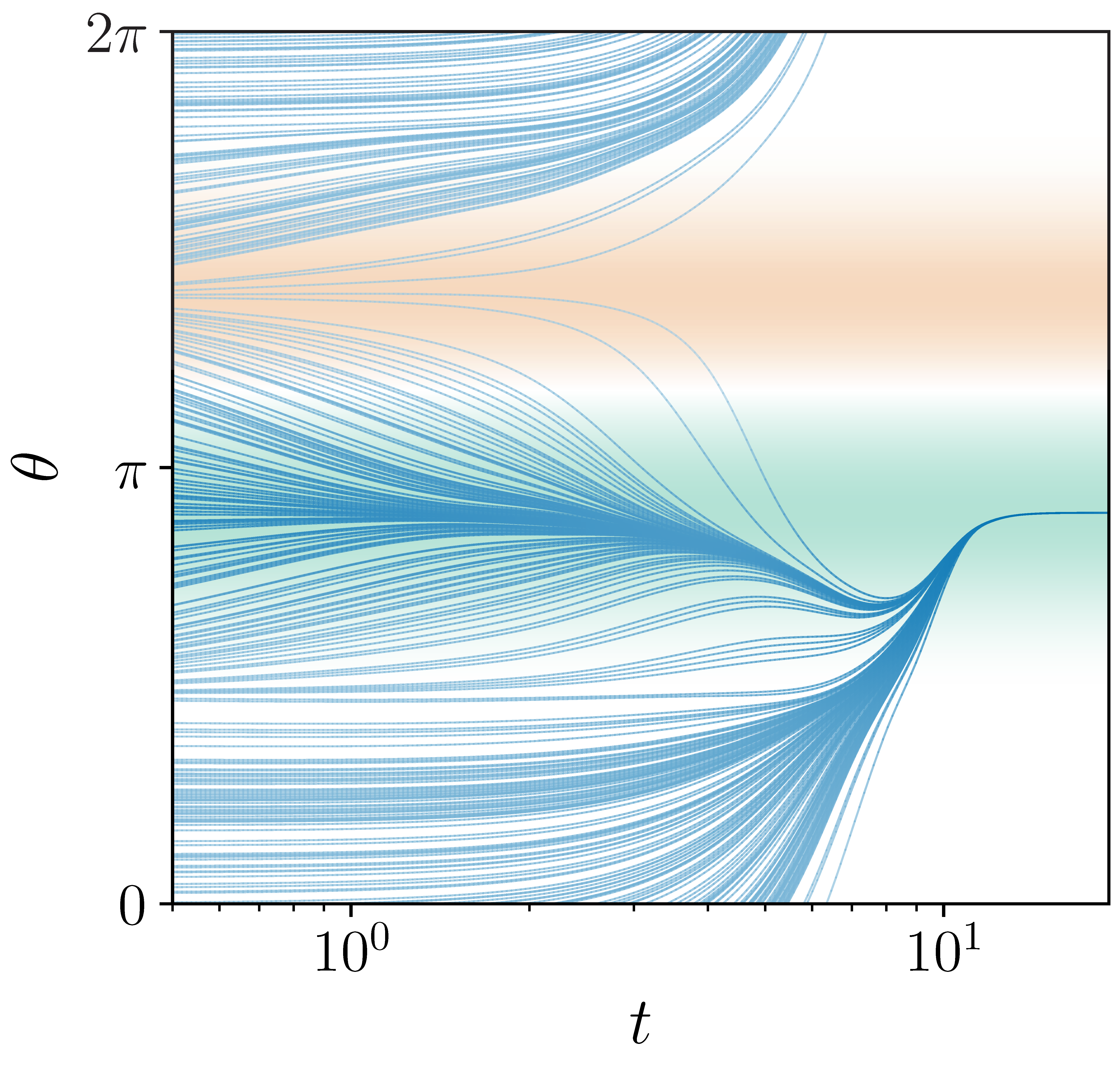}
    \hfill
        \centering
        \includegraphics[width=0.32\linewidth]{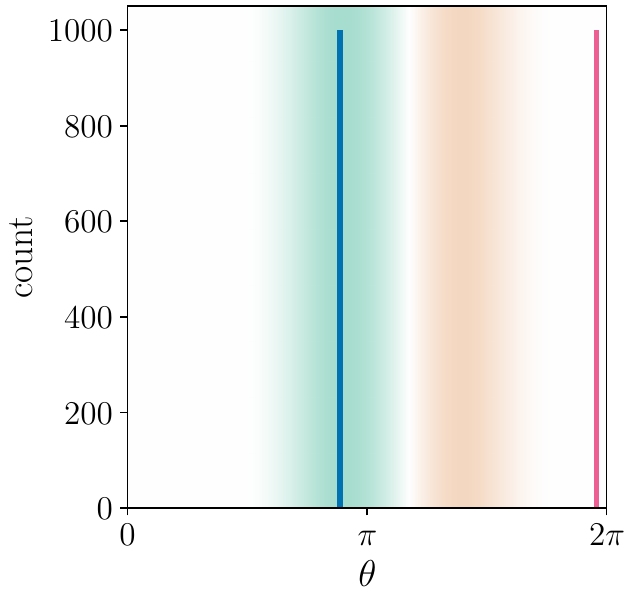}
        \caption{$\beta=0.05$} 
    \end{subfigure}

    \vspace{0.5em}

    \begin{subfigure}[b]{\linewidth}
        \centering
        \includegraphics[width=0.32\linewidth]{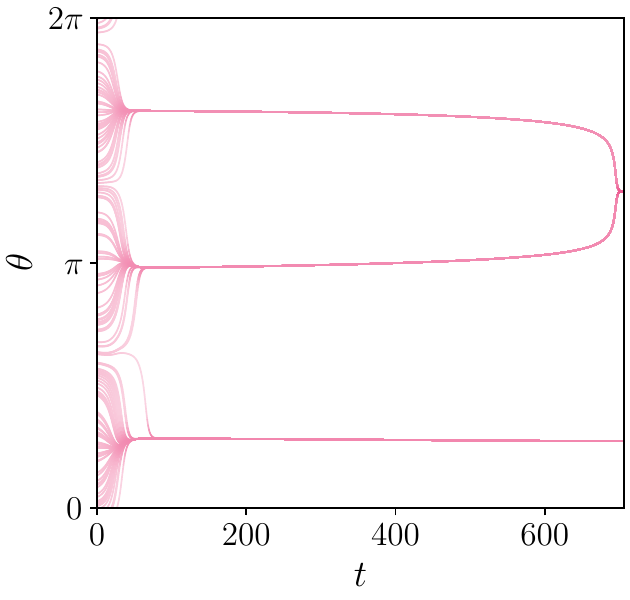}
    \hfill
        \centering
        \includegraphics[width=0.32\linewidth]{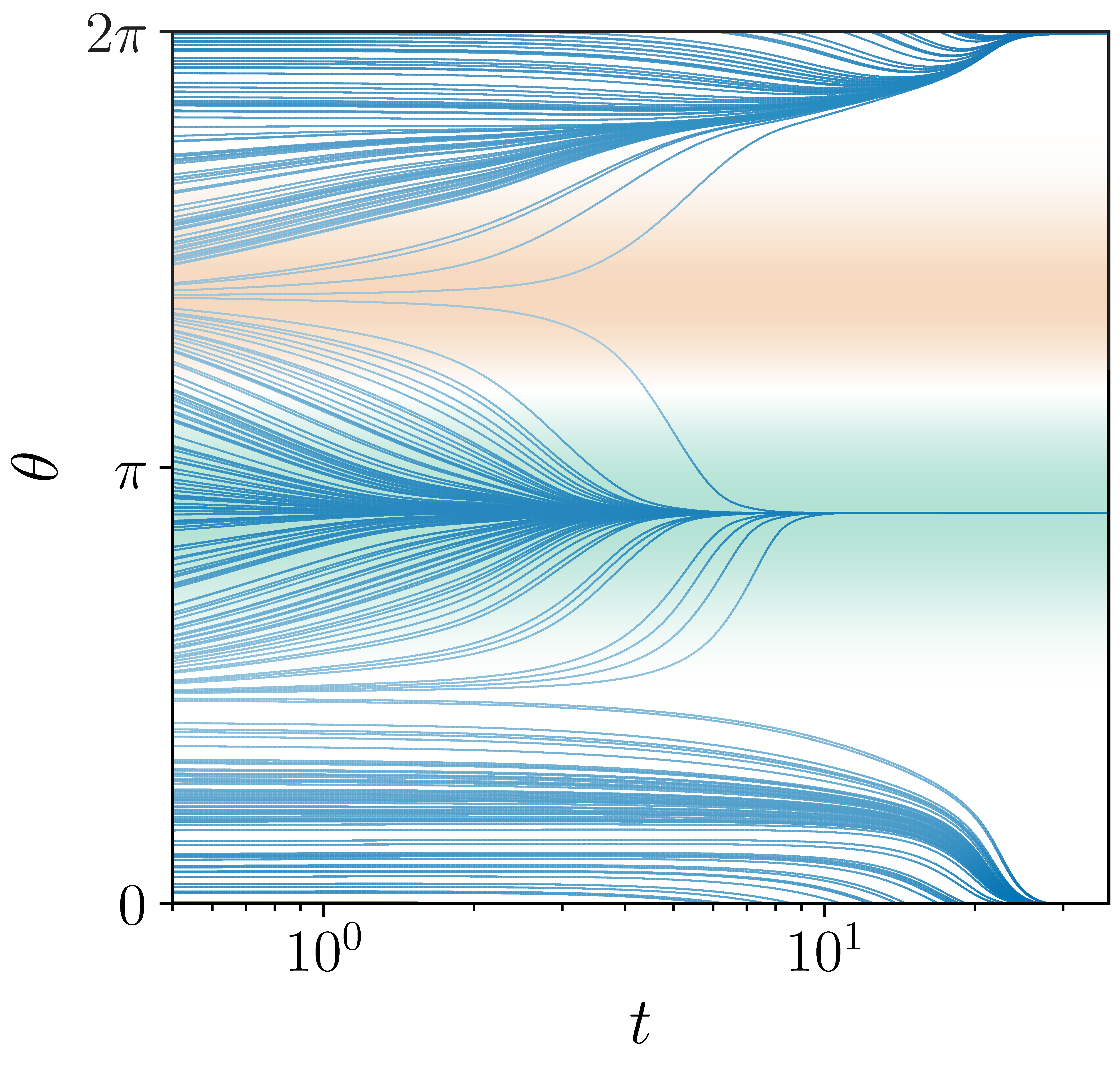}
    \hfill
        \centering
        \includegraphics[width=0.32\linewidth]{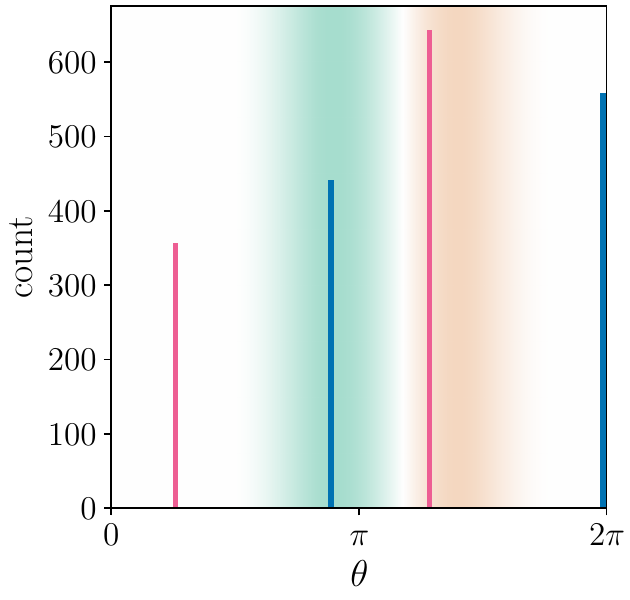}
        \caption{$\beta=5$}
    \end{subfigure}

    \vspace{0.5em}

    \begin{subfigure}[b]{\linewidth}
        \centering
        \includegraphics[width=0.32\linewidth]{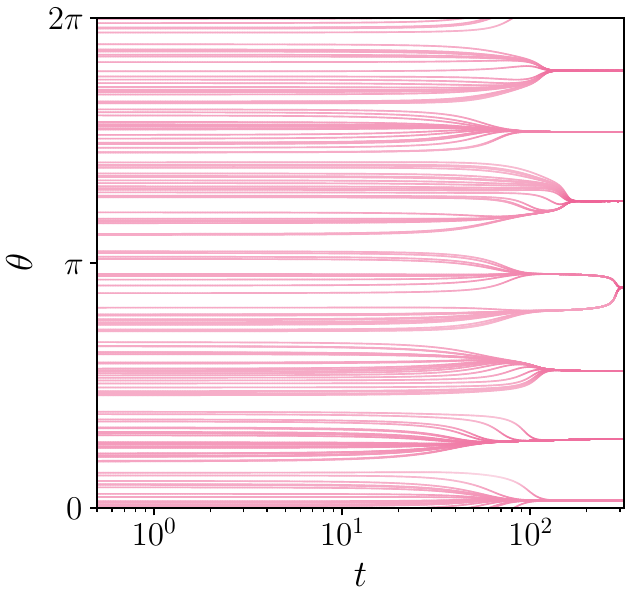}
    \hfill
        \centering
        \includegraphics[width=0.32\linewidth]{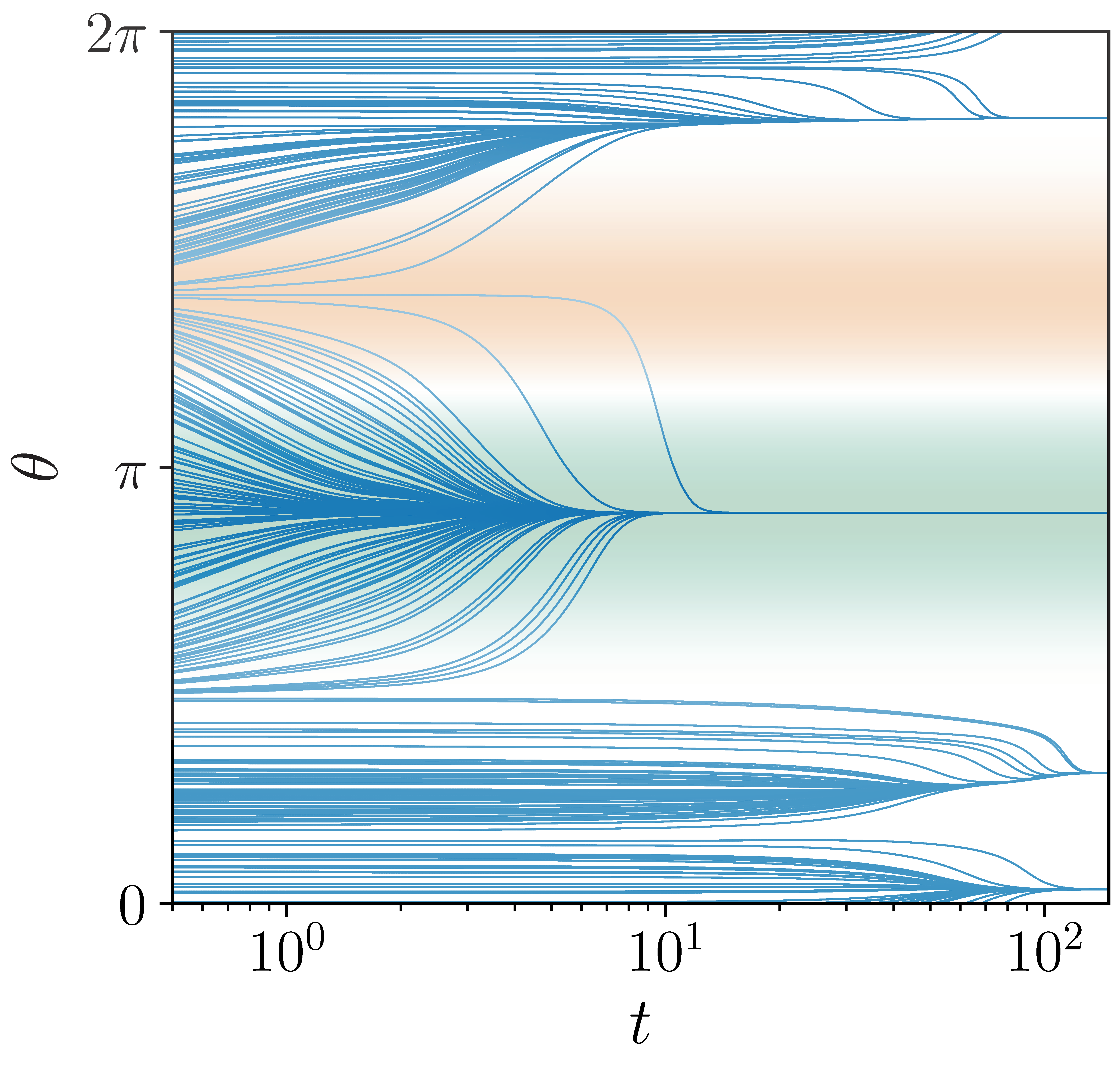}
    \hfill
        \centering
        \includegraphics[width=0.32\linewidth]{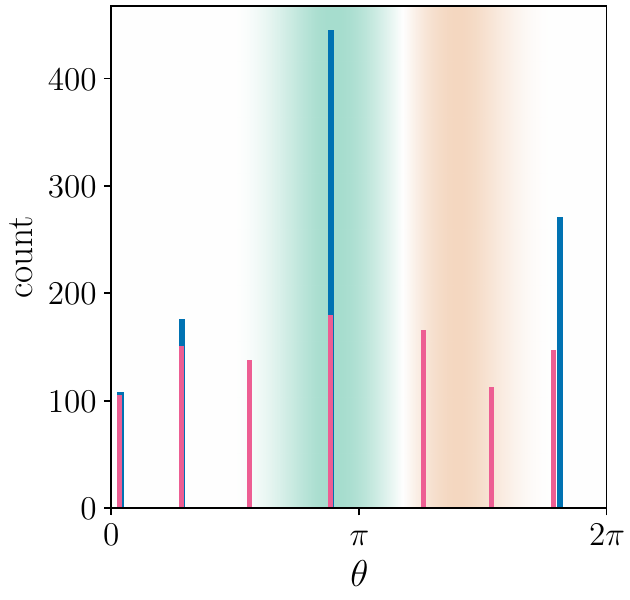}
        \caption{$\beta=50$}
    \end{subfigure}

    \caption{Gradient ascent on $\mathbb{S}^1$ with ReLU perceptron.
    \textbf{Left:} pure self-attention. 
    \textbf{Middle:} self-attention with a ReLU perceptron.
    \textbf{Right:} measure at final time. Background shading represents the potential landscape ({\color{ggreen}green}: positive; {\color{gpeach}orange}: negative values).
    }
    \label{fig:atoms_three_betas}
\end{figure}

\subsection*{Methodology}

Unless stated otherwise, we simulate \eqref{eq:WGF-main}--\eqref{eq:full-WGF-main} on $\mathbb{S}^1$ using $1000$ particles initialized uniformly on $[0,2\pi)$ using a fixed random seed. We use an explicit Euler scheme with $\Delta t=0.1$.
We sweep the inverse temperatures
\[
\beta \in \{0.01,0.05,0.1,0.5,1,3,5,7,10,15,25,35,50\}.
\]
The perceptron weights $\upvartheta$ are sampled from a standard normal distribution and fixed across all runs. 

For a given $\beta$, motivated by \Cref{lem: concavity}, we identify clusters on $\mathbb{S}^1$ by grouping particles whose pairwise geodesic distance is at most $\min\{1/(2\sqrt{\beta}),\pi/4\}$.

The simulation stops once the cluster count remains constant over a window of $5$ snapshots (with a snapshot every 10 steps) and $\max_i |\dot\theta_i| \le  10^{-4}$. Metastability can occur \cite{geshkovski2024dynamic, bruno2025emergence, bruno2025multiscale}: the dynamics may enter very slow regimes in which the convergence criteria are satisfied over long windows even though residual motion persists at larger timescales. 

\subsection*{Emergence of clusters}

We begin with gradient ascent  on $\mathbb{S}^1$ with ReLU perceptron. \Cref{fig:atoms_three_betas} displays three representative runs.

This protocol is repeated on $\mathbb{S}^2$. For visualization, we map each particle $x_i(t)\in\S^2$ to spherical angles $(\theta_i(t),\phi_i(t))$ and build a two-dimensional histogram on a uniform $(\theta,\phi)$--grid.

\begin{figure}[!h]
    \centering

    \begin{subfigure}[b]{0.32\linewidth}
        \centering
        \includegraphics[width=\linewidth]{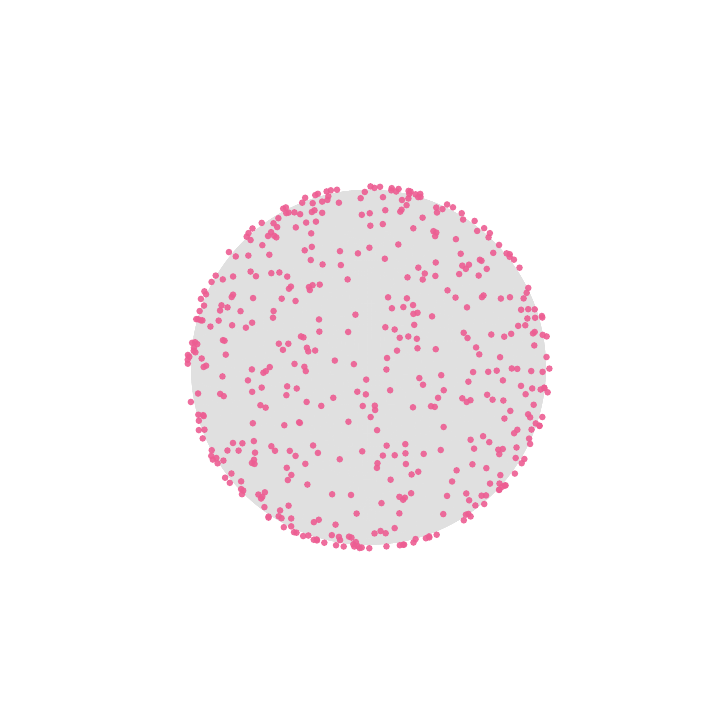}
    \end{subfigure}
    \hfill
    \begin{subfigure}[b]{0.32\linewidth}
        \centering
        \includegraphics[width=\linewidth]{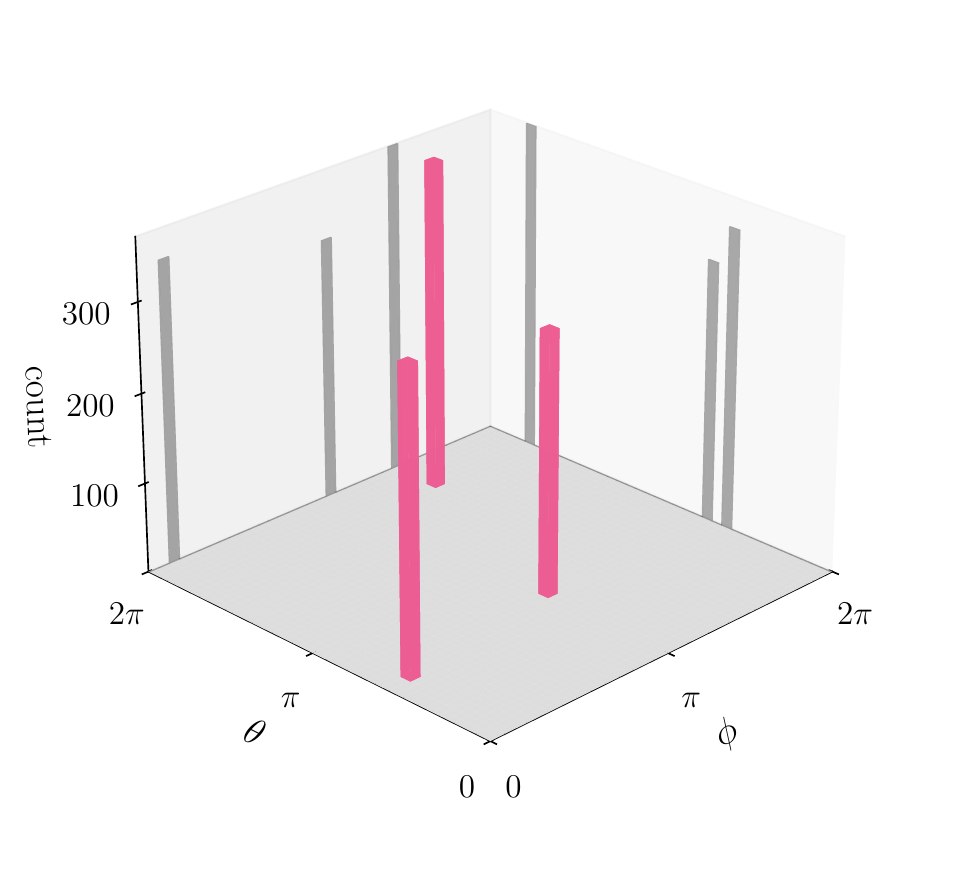}
    \end{subfigure}
    \hfill
    \begin{subfigure}[b]{0.32\linewidth}
        \centering
        \includegraphics[width=\linewidth]{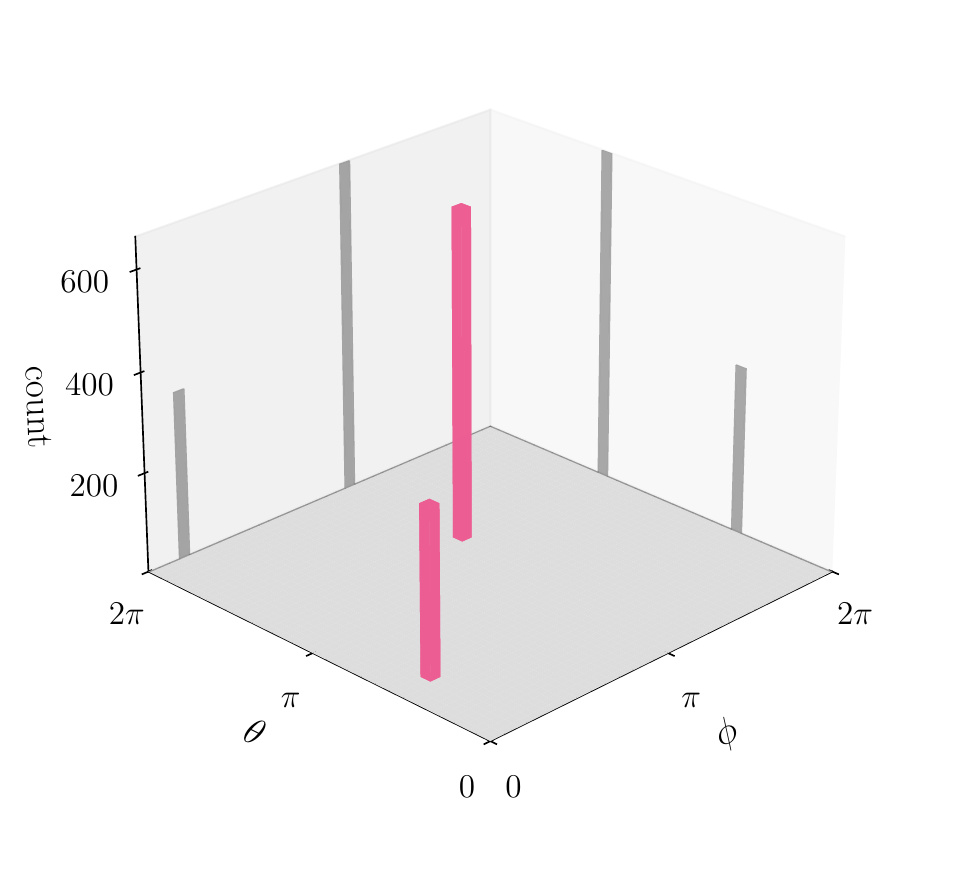}
    \end{subfigure}

    \vspace{0.5em}

    \begin{subfigure}[b]{0.32\linewidth}
        \centering
        \includegraphics[width=\linewidth]{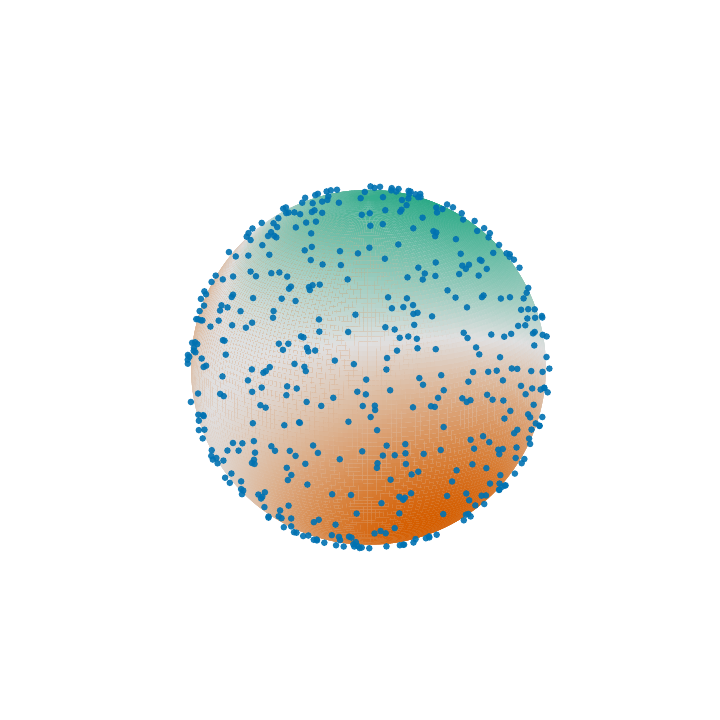}
        \caption{Initial}
    \end{subfigure}
    \hfill
    \begin{subfigure}[b]{0.32\linewidth}
        \centering
        \includegraphics[width=\linewidth]{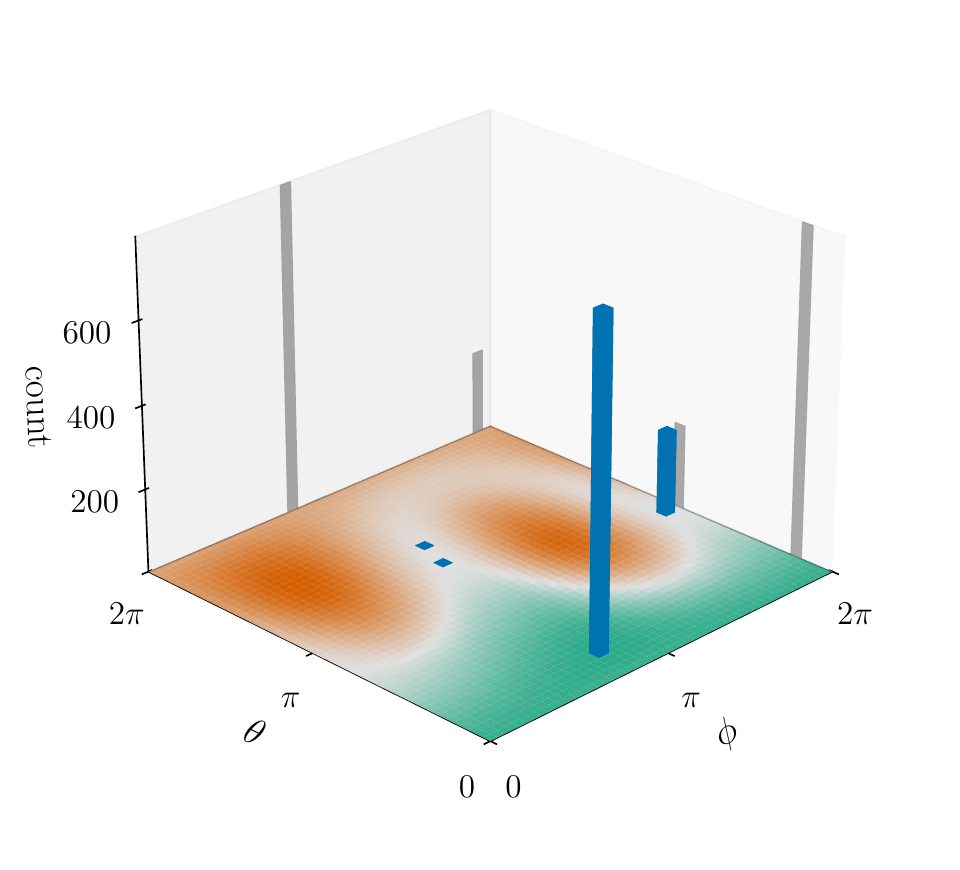}
        \caption{Intermediate}
    \end{subfigure}
    \hfill
    \begin{subfigure}[b]{0.32\linewidth}
        \centering
        \includegraphics[width=\linewidth]{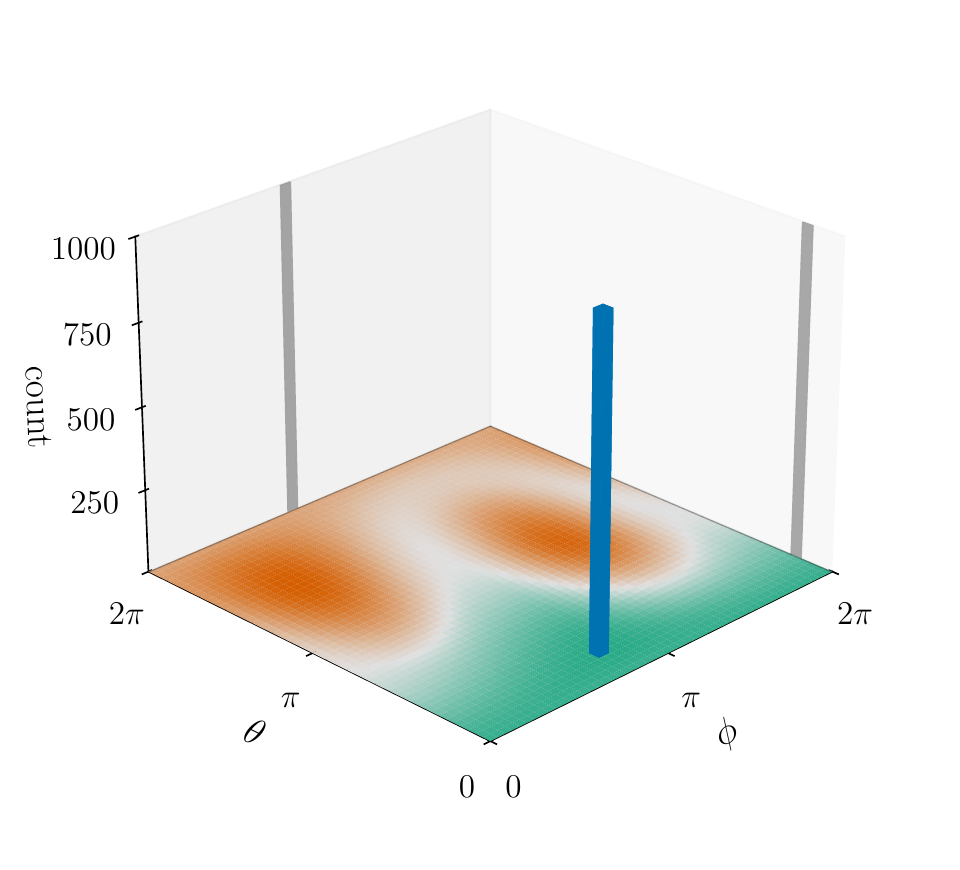}
        \caption{Final}
    \end{subfigure}
\caption{Gradient ascent on $\S^2$ with  $\beta=1$. \textbf{Top row:} pure self-attention. \textbf{Bottom row:} self-attention with ReLU perceptron. An animation is available at \href{https://github.com/antonioalvarezl/2026-MLP-Attention-Energy/blob/main/examples/USAS2.gif}{https://github.com/antonioalvarezl/2026-MLP-Attention-Energy/blob/main/examples/USAS2.gif}.
    } 
\end{figure}

\subsection*{Gradient descent}

We now turn to the minimization/descent dynamics, which is \eqref{eq:WGF-main} with a $-$ sign.
Here pure self-attention is repulsive and promotes spreading of particles, while the perceptron enforces singular stationary configurations. As established in \Cref{prop: min.max} \textbf{(ii)}, this system admits a unique global minimizer $\mu_\ast$ for every choice of parameters, and \Cref{thm: any.d} applies to its stationary points.

We first run gradient descent with a ReLU perceptron: $\sigma(s)=s_+$. \Cref{fig:descent_relu_traj} shows particle trajectories for $\beta \in \{0.05, 5, 50\}$ (recall that the convergence histogram for $\beta=1$ was shown in \Cref{fig:descent_relu_hist_energy}). A visible fraction of the mass remains trapped where the ReLU is identically zero: in those regions, the perceptron vanishes. Numerically, this produces a mixed stationary pattern: atomic clusters in active regions coexisting with frozen components in dead zones, consistent with \Cref{thm: any.d}\textbf{(i)}--\textbf{(iii)}.

\begin{figure}[!h]
    \centering
    \begin{subfigure}[b]{0.32\linewidth}
        \centering
        \includegraphics[width=\linewidth]{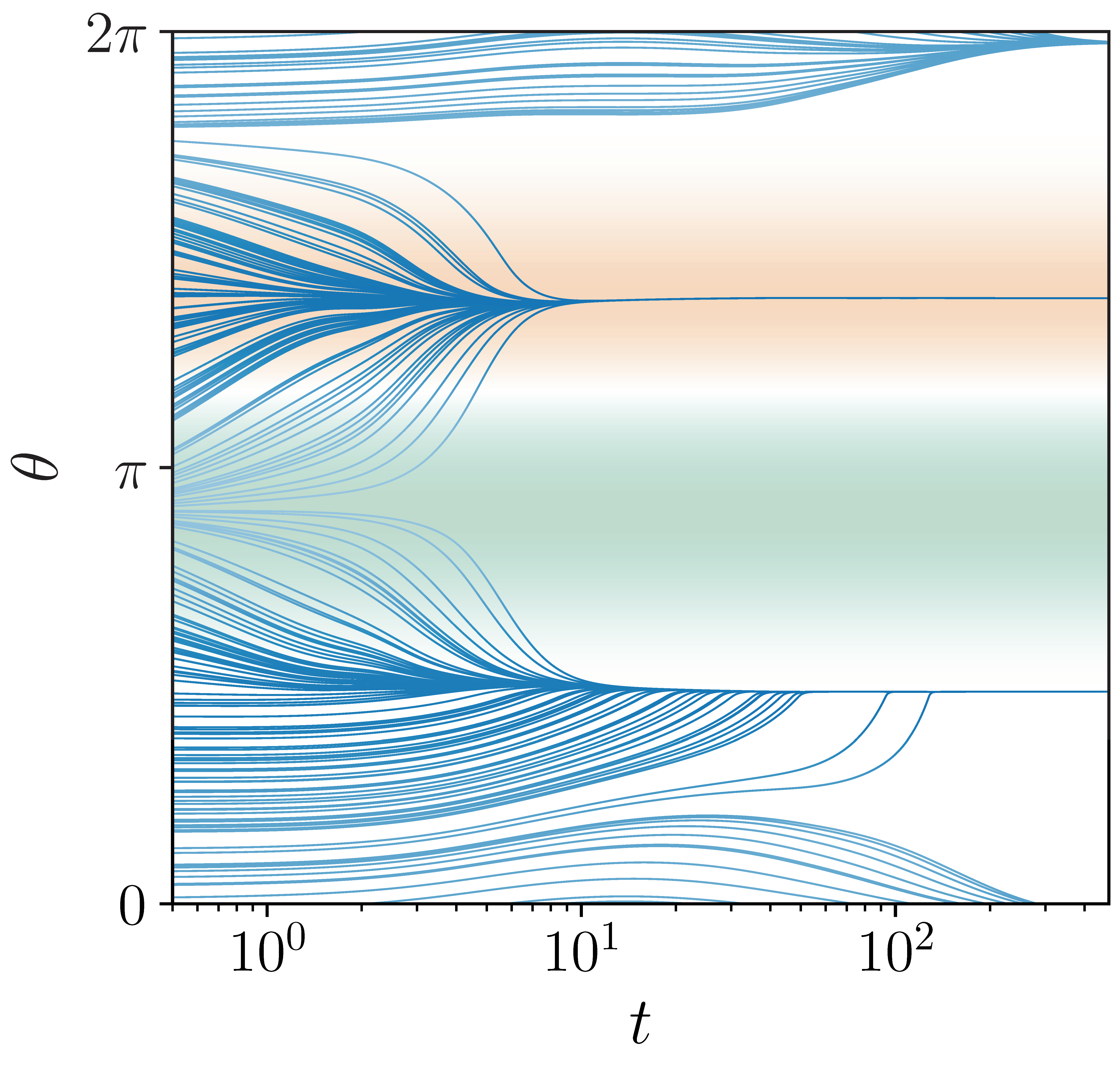}
        \caption{$\beta=0.05$}
    \end{subfigure}
    \hfill
    \begin{subfigure}[b]{0.32\linewidth}
        \centering
        \includegraphics[width=\linewidth]{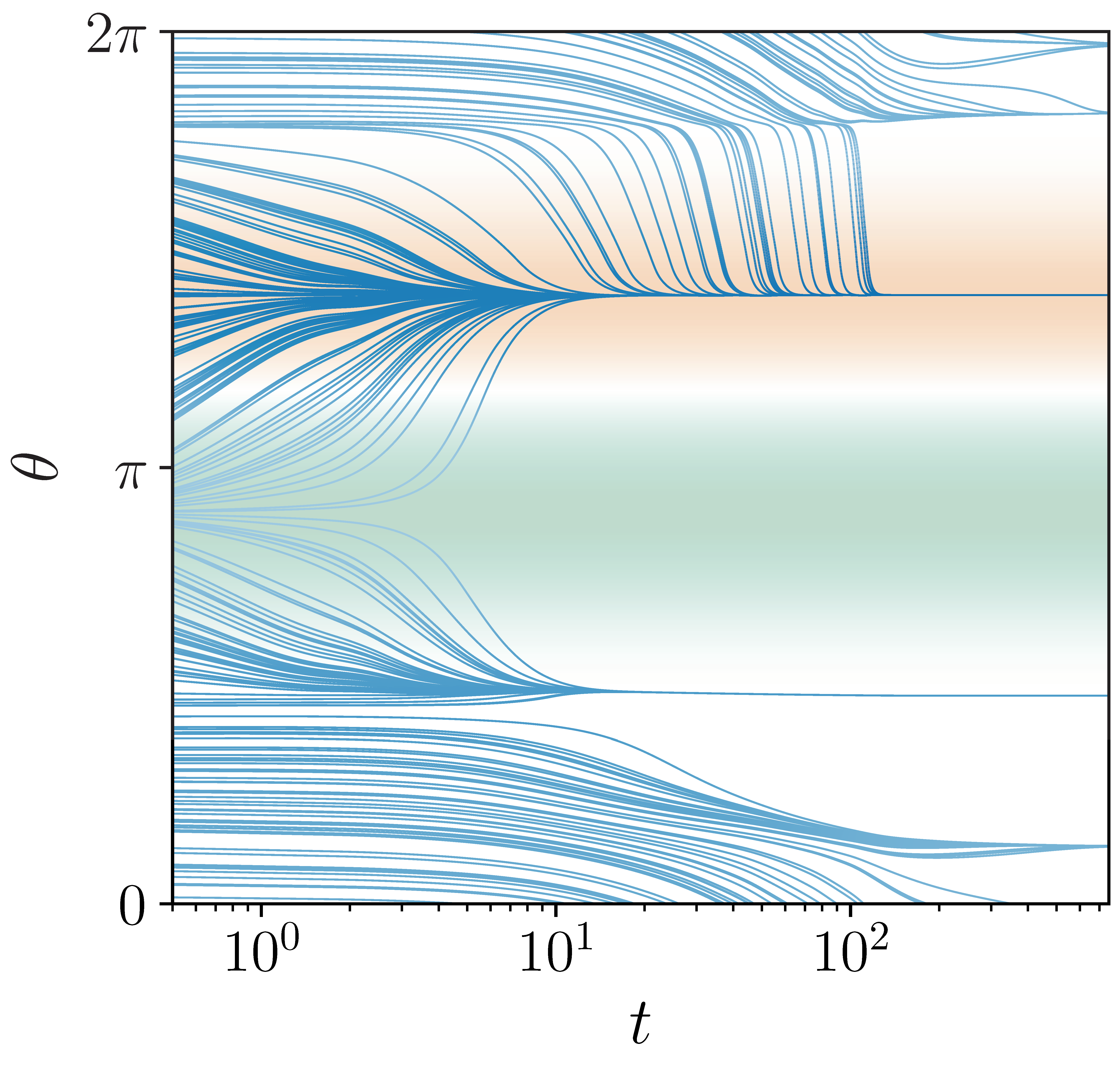}
        \caption{$\beta=5$}
    \end{subfigure}
    \hfill
    \begin{subfigure}[b]{0.32\linewidth}
        \centering
        \includegraphics[width=\linewidth]{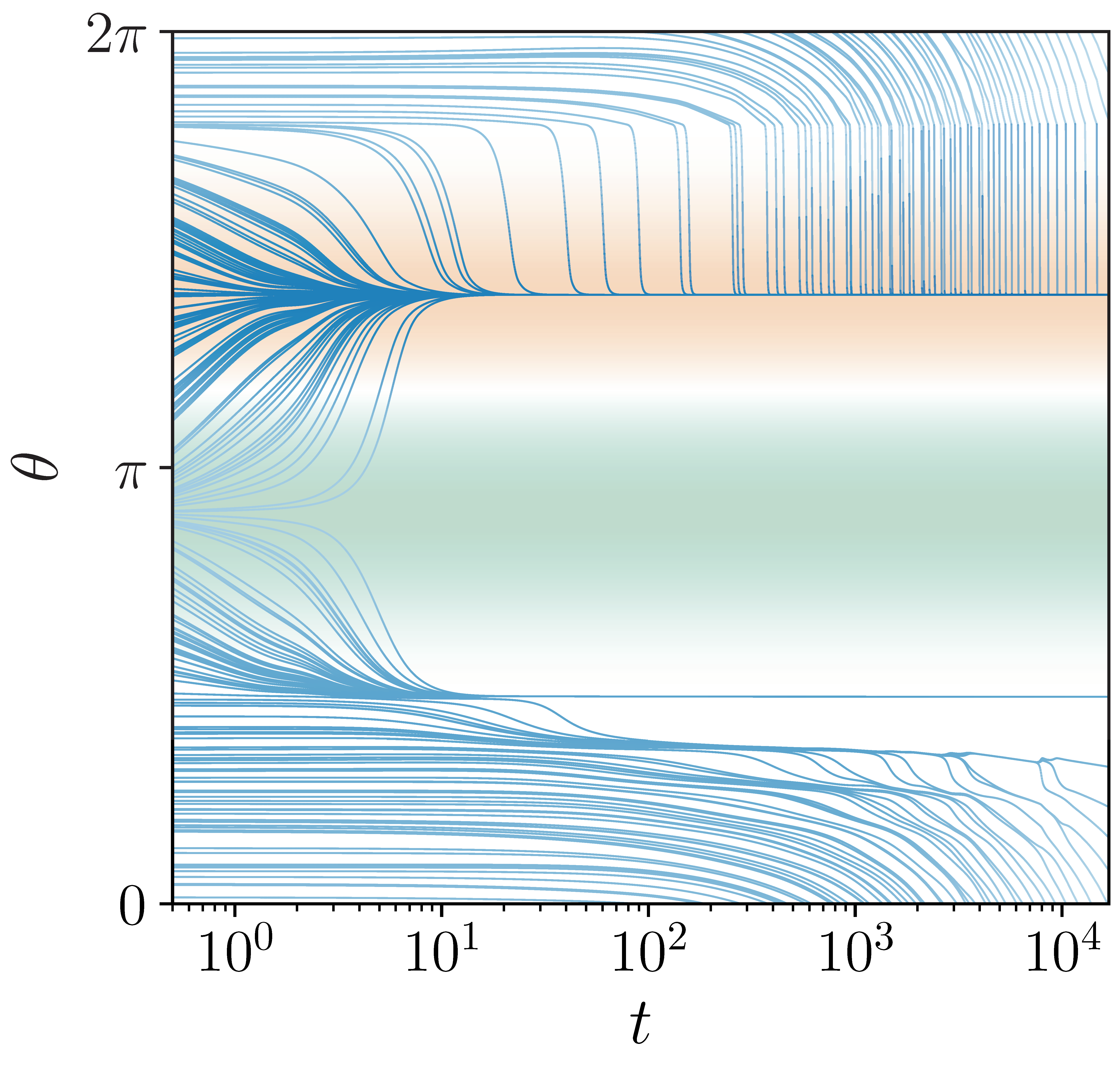}
        \caption{$\beta=50$}
    \end{subfigure}
\caption{Gradient descent on $\mathbb{S}^1$ with ReLU perceptron. 
}
    \label{fig:descent_relu_traj}
\end{figure}

We repeat the same experiment using GeLU. \Cref{fig:descent_gelu_traj} shows representative histograms at convergence. 

\begin{figure}[!h]
    \centering

    \begin{subfigure}[!h]{0.32\linewidth}
        \centering
        \includegraphics[width=\linewidth]{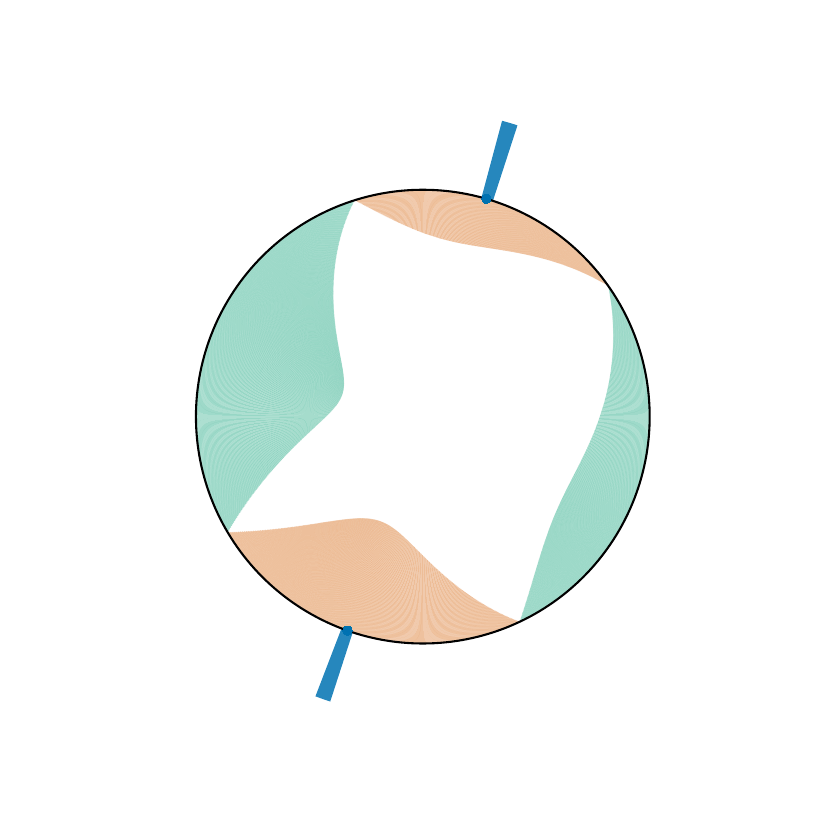}
        \caption{$\beta=0.05$}
    \end{subfigure}
    \hfill 
    \begin{subfigure}[!h]{0.32\linewidth}
        \centering
        \includegraphics[width=\linewidth]{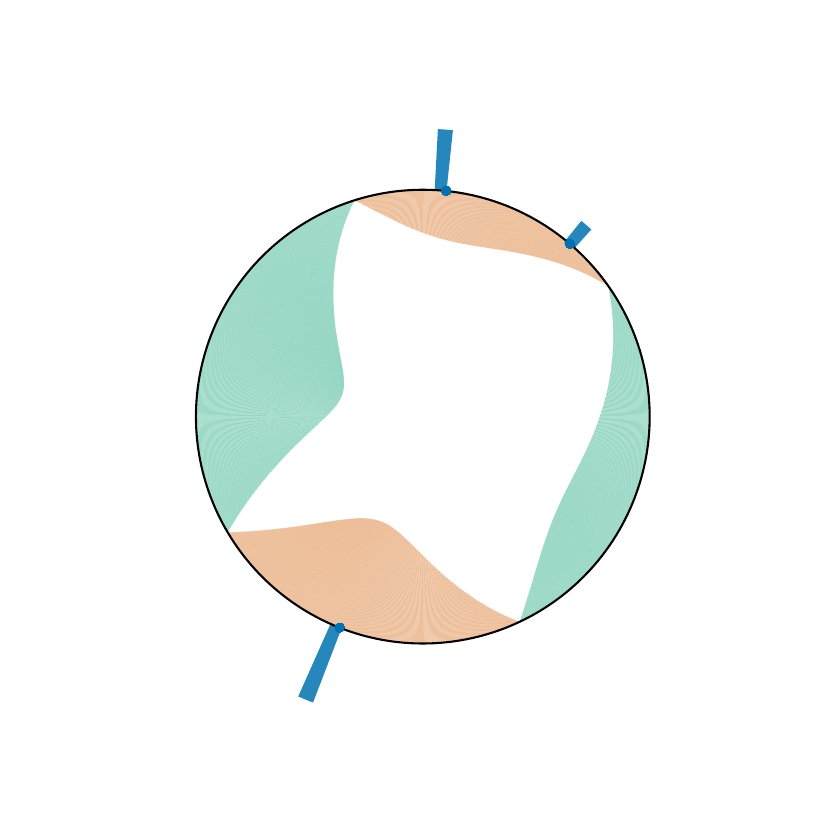}
        \caption{$\beta=5$}
    \end{subfigure}
    \hfill 
    \begin{subfigure}[!h]{0.32\linewidth}
        \centering
        \includegraphics[width=\linewidth]{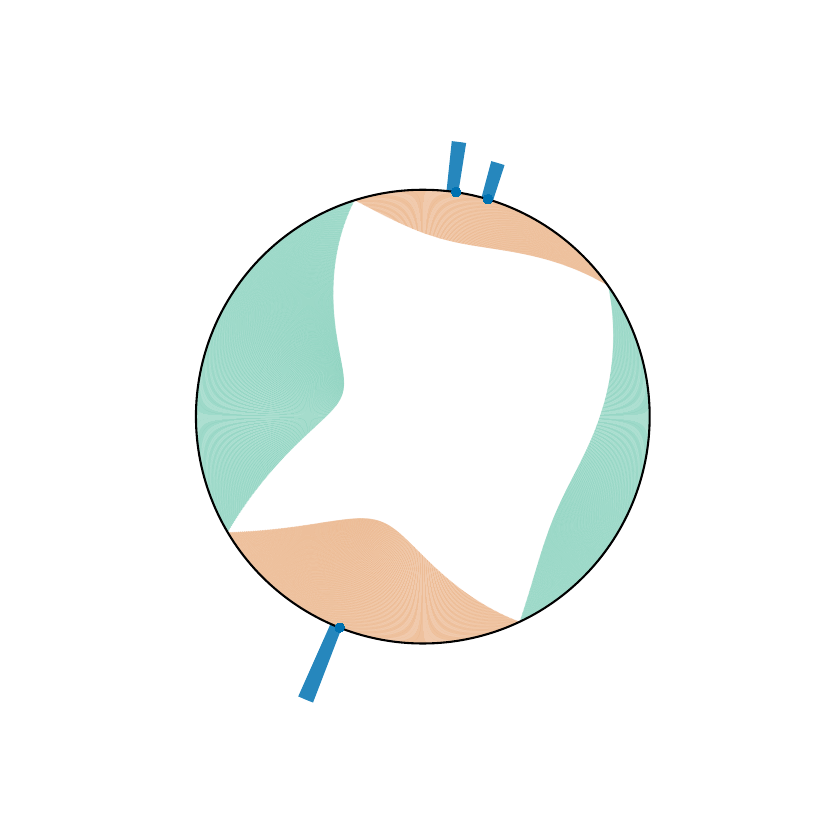}
        \caption{$\beta=50$}
    \end{subfigure}
    \caption{Histograms at final time for gradient descent with GeLU perceptron.}
    \label{fig:descent_gelu_traj}
\end{figure}

The analogous setup on $\S^2$ appears in \Cref{fig:s2_hist_beta_comparison}.

\begin{figure}[!h]
    \centering

  \begin{subfigure}[b]{0.32\linewidth}
        \centering
        \includegraphics[width=\linewidth]{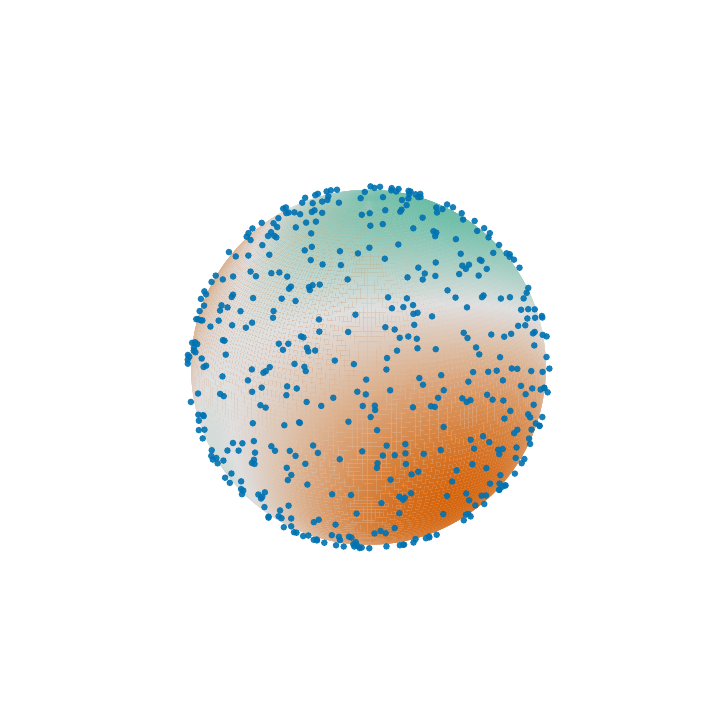}
    \end{subfigure}
    \hfill
    \begin{subfigure}[b]{0.32\linewidth}
        \centering
        \includegraphics[width=\linewidth]{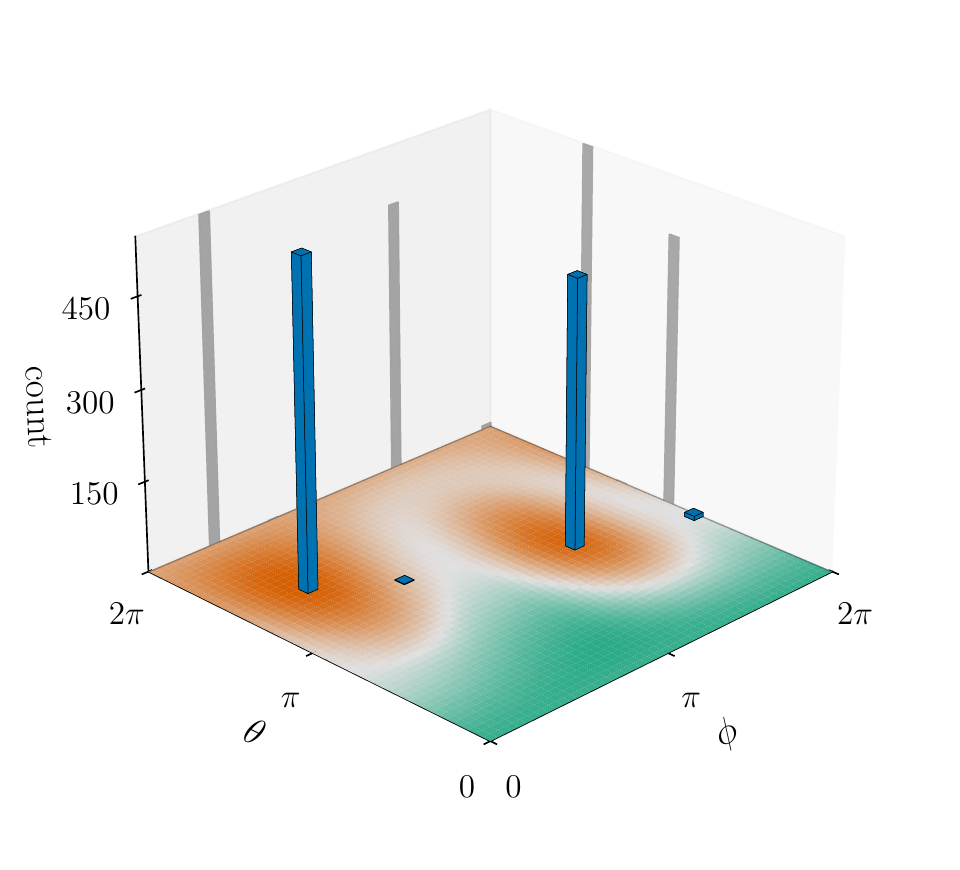}
    \end{subfigure}
    \hfill
    \begin{subfigure}[b]{0.32\linewidth}
        \centering
        \includegraphics[width=\linewidth]{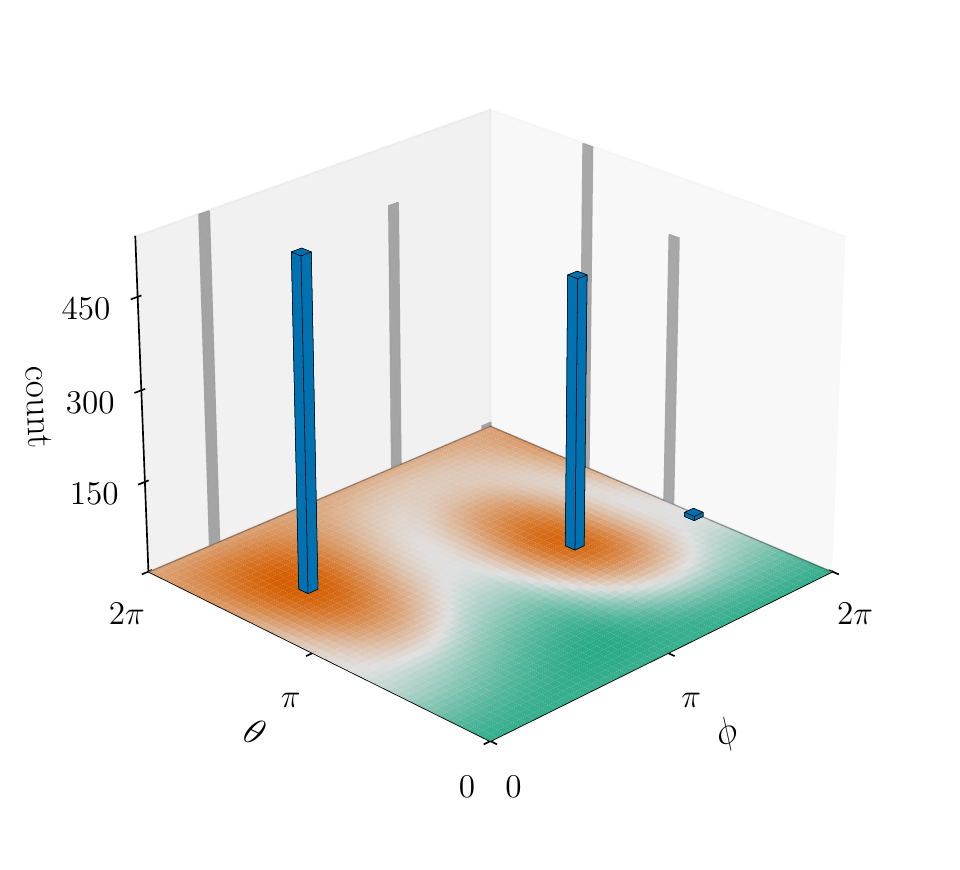}
    \end{subfigure}
    
    \vspace{0.5em}

    \begin{subfigure}[b]{0.32\linewidth}
        \centering
        \includegraphics[width=\linewidth]{figures/init_mlp.pdf}
    \end{subfigure}
    \hfill
    \begin{subfigure}[b]{0.32\linewidth}
        \centering
        \includegraphics[width=\linewidth]{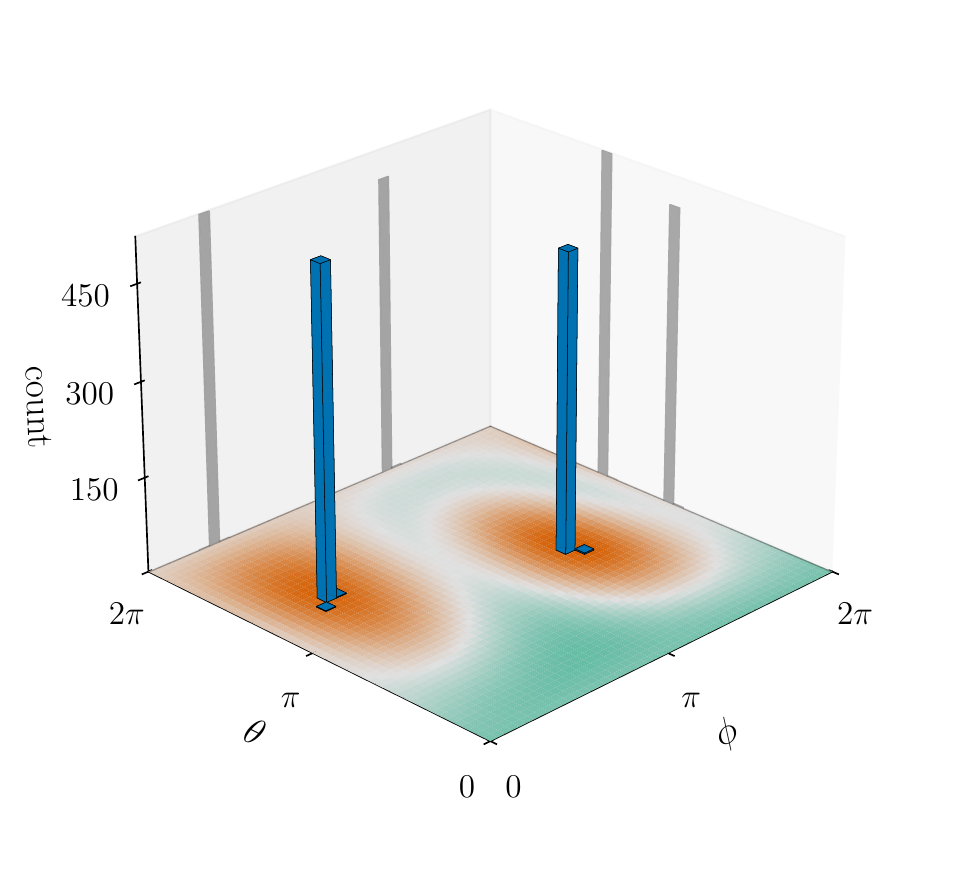}
    \end{subfigure}
    \hfill
    \begin{subfigure}[b]{0.32\linewidth}
        \centering
        \includegraphics[width=\linewidth]{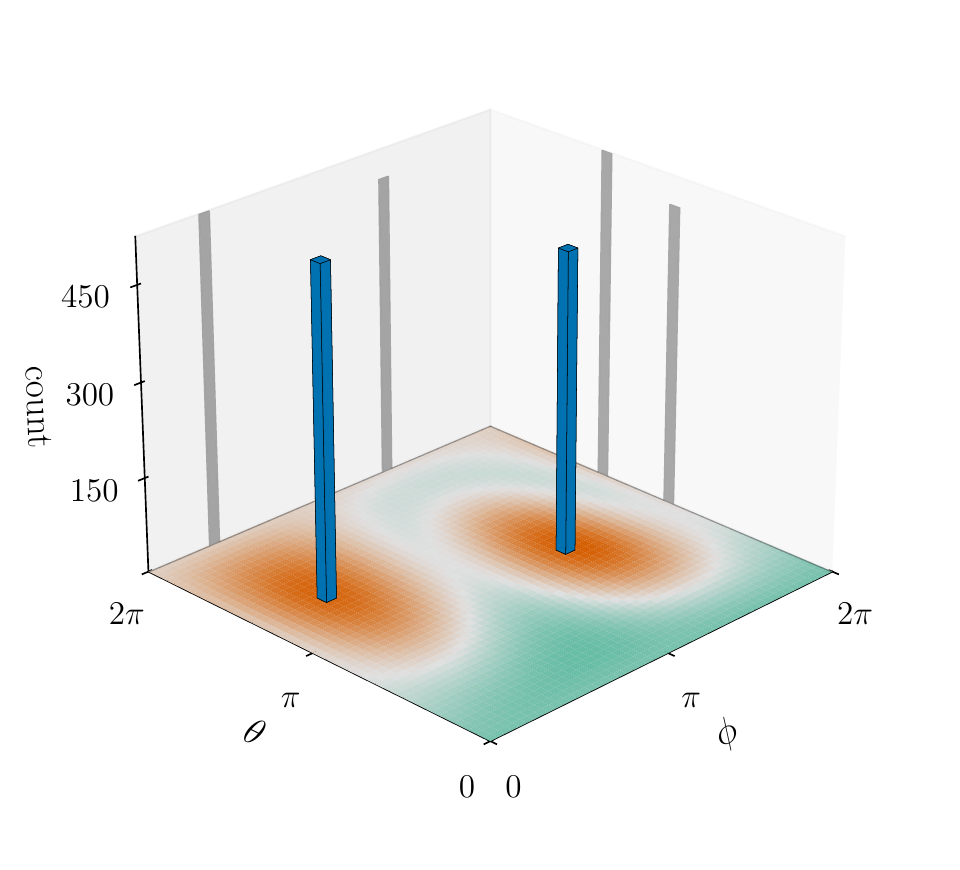}
    \end{subfigure}
    
    \caption{Gradient descent on  $\S^2$ with $\beta=1$. 
    \textbf{Top row:} ReLU perceptron. \textbf{Bottom row:} GeLU perceptron. 
    An animation is available at \href{https://github.com/antonioalvarezl/2026-MLP-Attention-Energy/blob/main/examples/USAdS2.gif}{https://github.com/antonioalvarezl/2026-MLP-Attention-Energy/blob/main/examples/USAdS2.gif}.
    }
    \label{fig:s2_hist_beta_comparison}
\end{figure}

\subsection*{Softmax normalization}

We now study the dynamics governed by \eqref{eq:full-WGF-main}---referred to as SA---across the settings previously considered on $\S^1$ and $\S^2$. \Cref{fig:sa_trajectories} illustrates the particle trajectories on $\S^1$ for three representative values of $\beta$.

\begin{figure}[!h]
    \centering

    \begin{subfigure}[b]{0.32\linewidth}
        \centering
        \includegraphics[width=\linewidth]{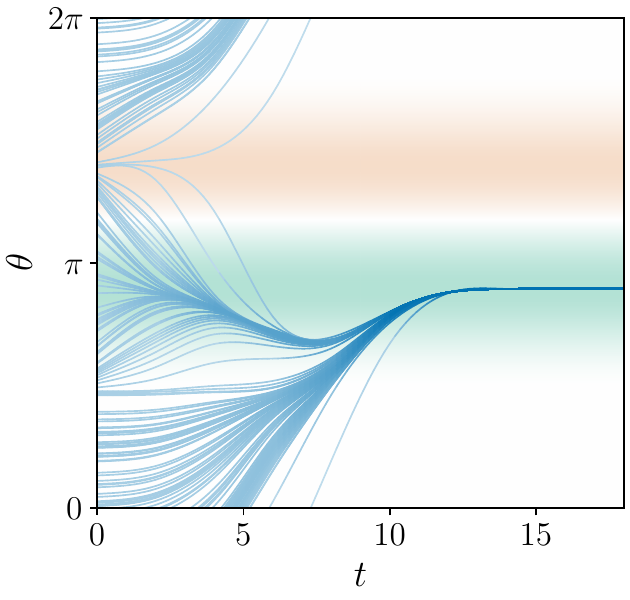}
    \end{subfigure}
    \hfill
    \begin{subfigure}[b]{0.32\linewidth}
        \centering
        \includegraphics[width=\linewidth]{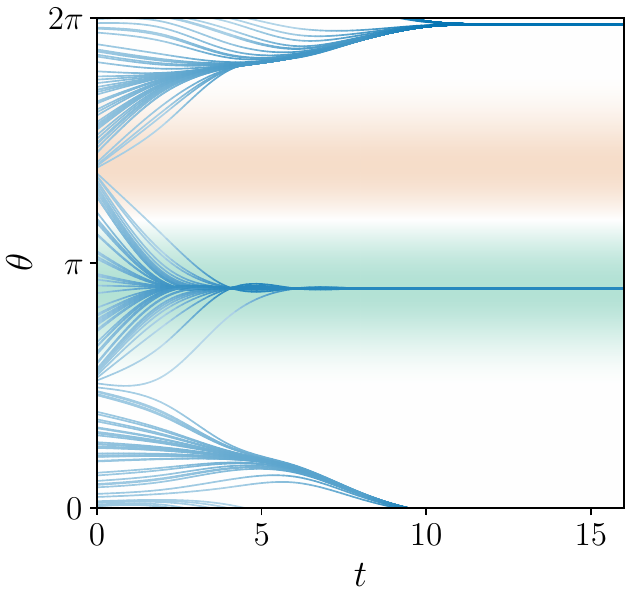}
    \end{subfigure}
    \hfill
    \begin{subfigure}[b]{0.32\linewidth}
        \centering
        \includegraphics[width=\linewidth]{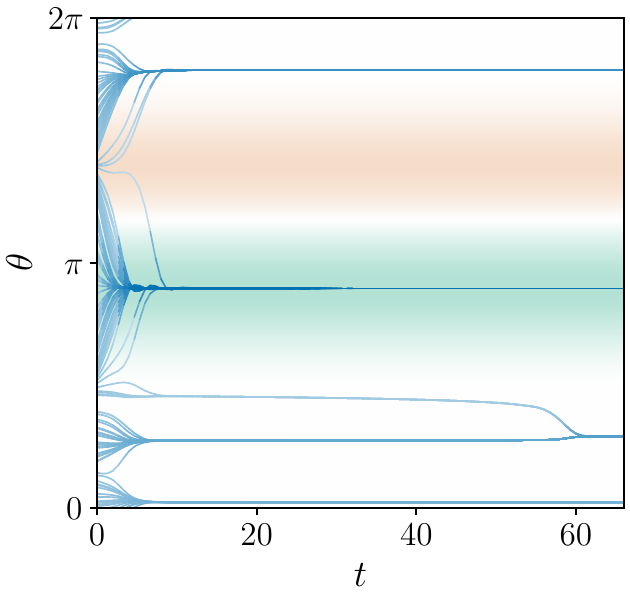}
    \end{subfigure}

    \vspace{0.5em}

    \begin{subfigure}[b]{0.32\linewidth}
        \centering
        \includegraphics[width=\linewidth]{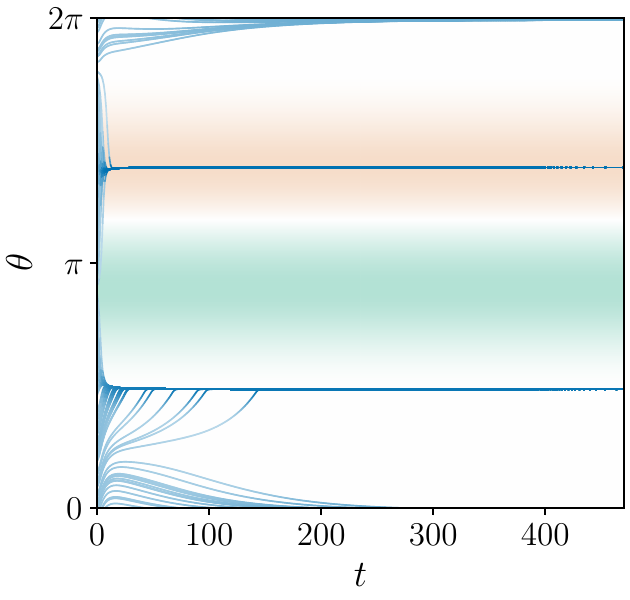}
    \end{subfigure}
    \hfill
    \begin{subfigure}[b]{0.32\linewidth}
        \centering
        \includegraphics[width=\linewidth]{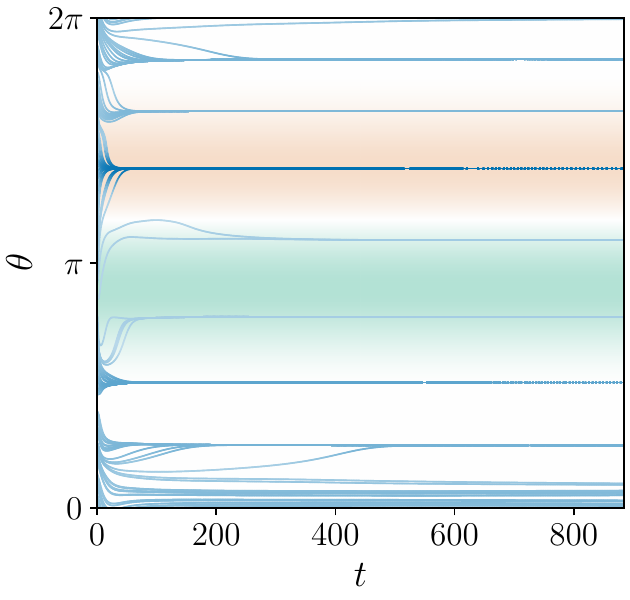}
    \end{subfigure}
    \hfill
    \begin{subfigure}[b]{0.32\linewidth}
        \centering
        \includegraphics[width=\linewidth]{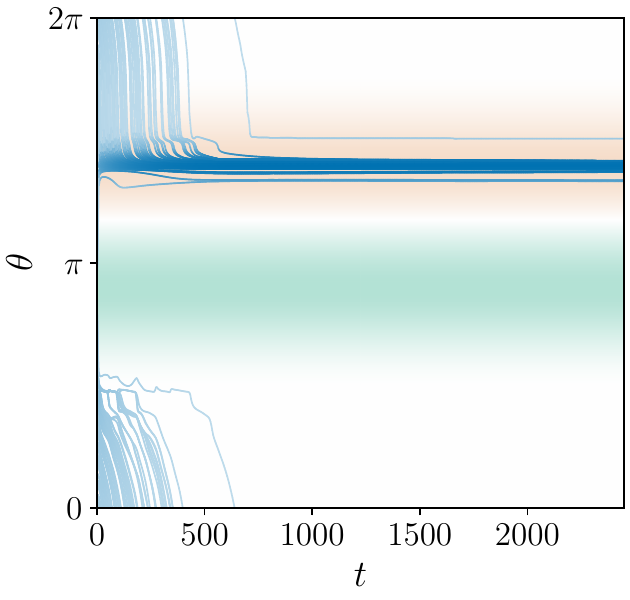}
    \end{subfigure}

    \vspace{0.5em}

    \begin{subfigure}[b]{0.32\linewidth}
        \centering
        \includegraphics[width=\linewidth]{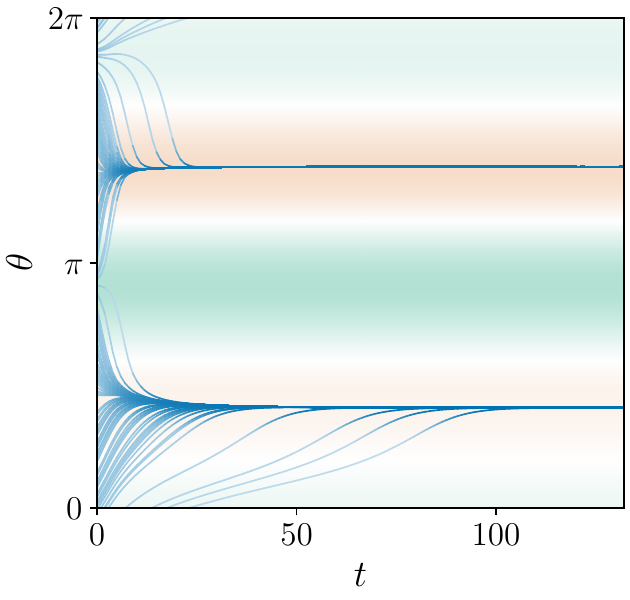}
    \end{subfigure}
    \hfill
    \begin{subfigure}[b]{0.32\linewidth}
        \centering
        \includegraphics[width=\linewidth]{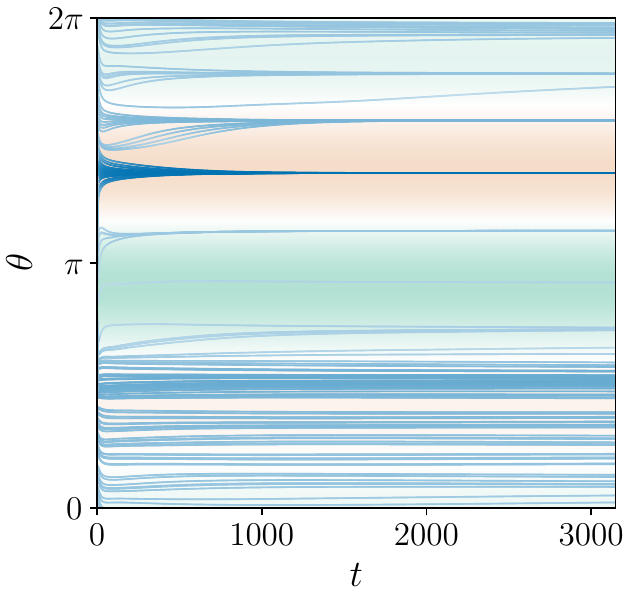}
    \end{subfigure}
    \hfill
    \begin{subfigure}[b]{0.32\linewidth}
        \centering
        \includegraphics[width=\linewidth]{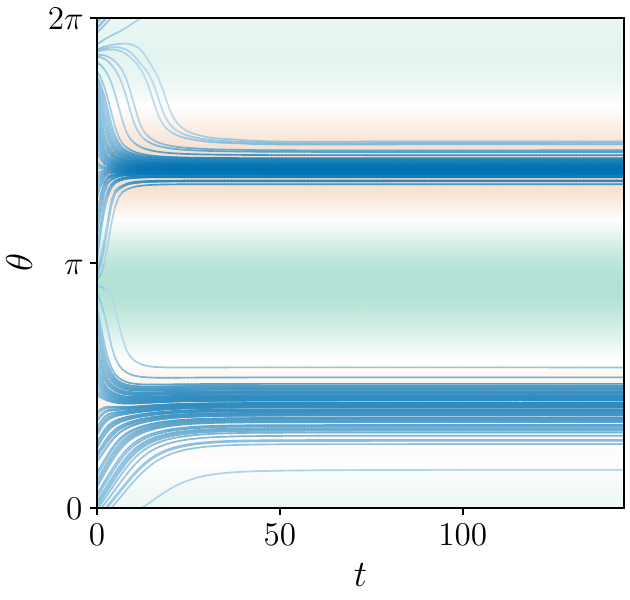}
    \end{subfigure}

    \caption{Trajectories following softmax-normalized attention.
    \textbf{Top row:} gradient ascent with ReLU perceptron.
    \textbf{Middle row:} gradient descent with ReLU perceptron.
    \textbf{Bottom row:} gradient descent with GeLU perceptron.}
    \label{fig:sa_trajectories}
\end{figure}

\subsection*{Scaling of support size}

We count the number of clusters (atoms) at convergence for every setup considered across the swept values of $\beta$. Results are summarized in \Cref{fig:cluster_scaling}. For pure self-attention (setup S1-0), the reported cluster counts correspond to the number of atoms in the metastable configuration, rather than the long-time limit. Indeed, these dynamics are known to collapse the initial measure to a single Dirac mass \cite{geshkovski2025mathematical,chen2025quantitative, geshkovski2024measure}.

The same caveat applies when adding a perceptron (setup~S1): for the fixed $\upvartheta$ used throughout, the perceptron potential exhibits a single local maximum (cf.\ \Cref{fig:atoms_three_betas}), so $\mathsf{E}_{\beta,\upvartheta}$ has a unique maximizer by \Cref{prop: min.max}. Yet, we still likely observe metastable configurations with more than one atom on computational time horizons. 

\begin{figure}[!h]
    \centering

\includegraphics[width=\linewidth]{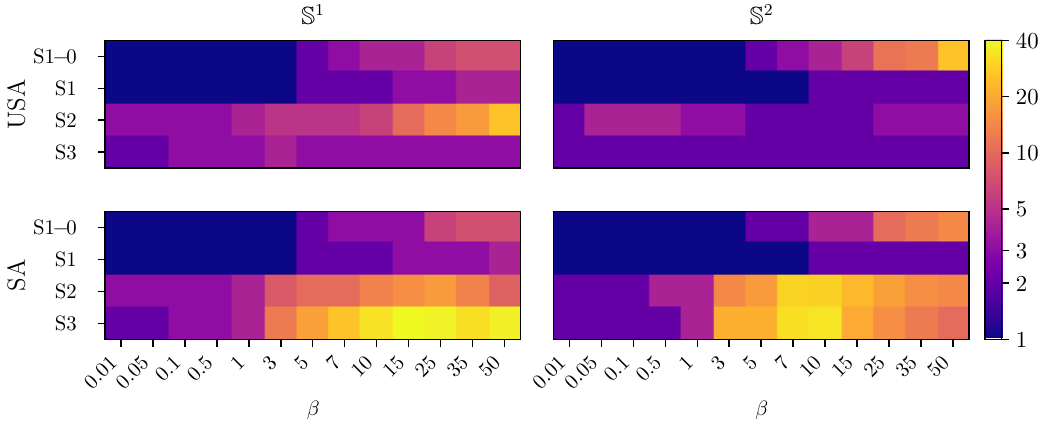}
    \caption{Number of clusters at convergence across setups. S1/S1-0: Ascent with/without ReLU perceptron. S2: Descent with ReLU perceptron. S3: Descent with GeLU perceptron.}
    \label{fig:cluster_scaling}
\end{figure}

\subsection*{Sensitivity to initial configuration}

We run $10$ independent simulations at $\beta=10$, using GeLU, with identical weights $\upvartheta$ and resampled uniform initializations. \Cref{fig:initializations} compares three settings: (i) pure self-attention (no perceptron), (ii) gradient ascent with a perceptron, and (iii) descent with the same perceptron. In the absence of a perceptron, the cluster locations are highly seed-dependent. In contrast, the perceptron drift breaks this symmetry, anchoring the clusters to specific locations that are largely independent of the initial configuration.

\begin{figure}[ht!]
    \centering
    \begin{subfigure}[!h]{0.32\linewidth}
        \centering
        \includegraphics[width=\linewidth]{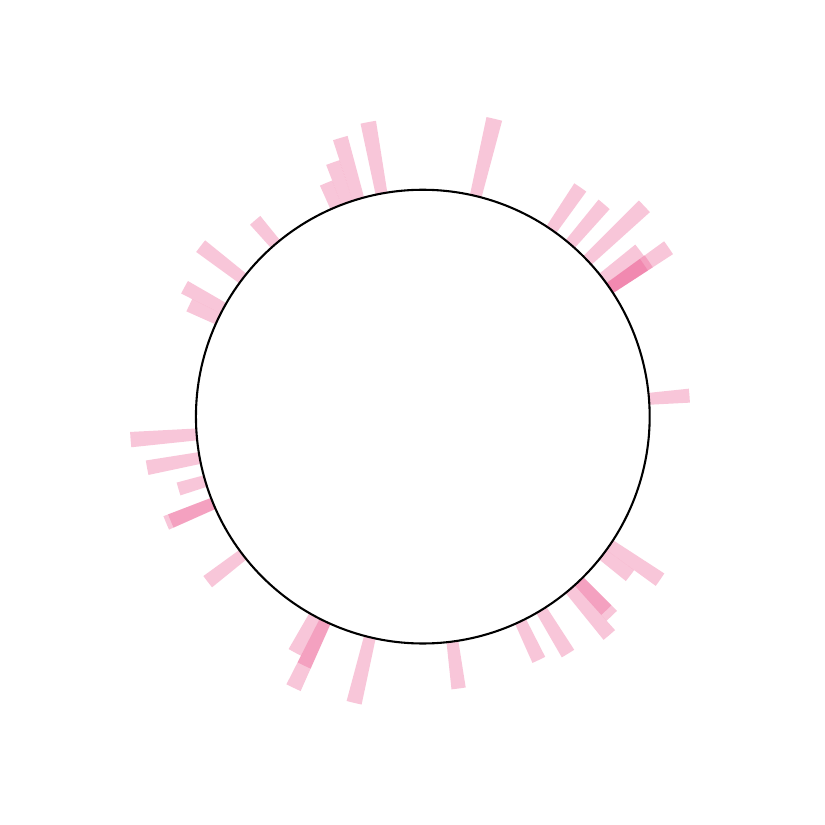}
    \end{subfigure}
    \hfill
    \begin{subfigure}[!h]{0.32\linewidth}
        \centering
        \includegraphics[width=\linewidth]{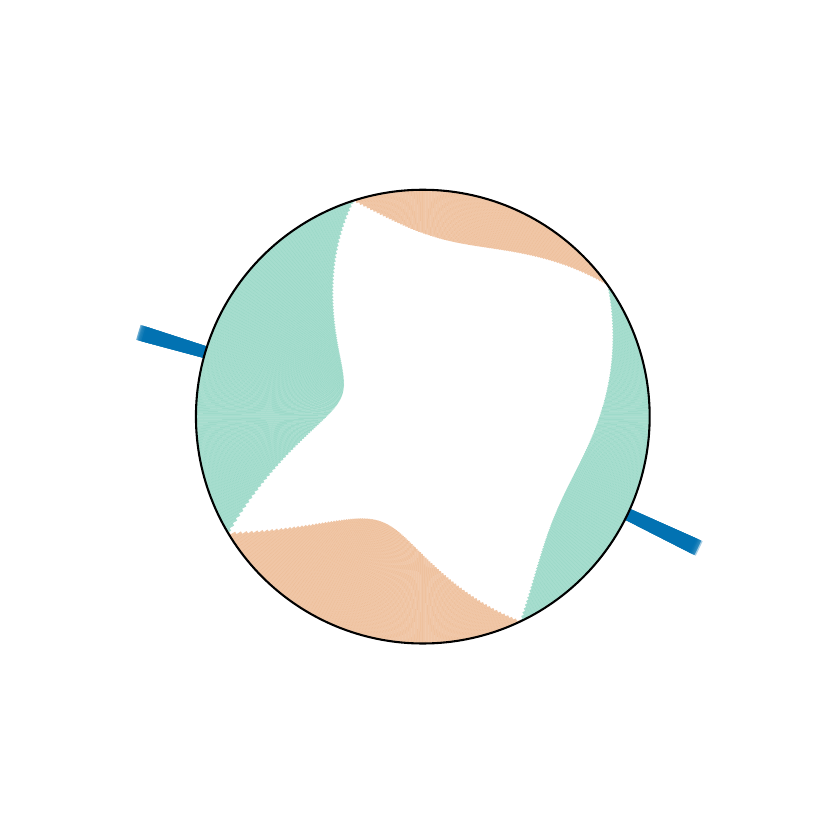}
    \end{subfigure}
    \hfill
    \begin{subfigure}[!h]{0.32\linewidth}
        \centering
        \includegraphics[width=\linewidth]{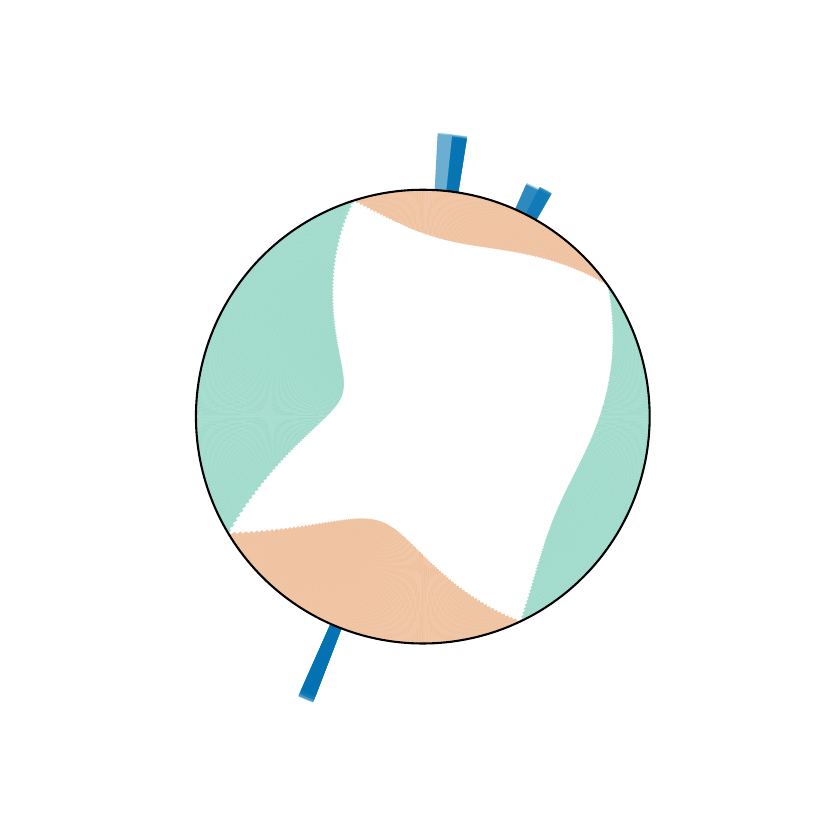}
    \end{subfigure}
    \caption{Superposition of stationary measures from $10$ independent uniform initializations with $\beta=10$ and $\sigma=$ GeLU.
    \textbf{Left:} pure self-attention descent yields seed-dependent cluster locations.
    \textbf{Middle:} the perceptron drift in ascent breaks the symmetry and selects seed-independent clusters.
    \textbf{Right:} in descent, the perceptron still anchors a seed-independent stationary support, which is less concentrated and may include spurious atoms.}
    \label{fig:initializations}
\end{figure}

\subsection*{Higher dimensions}\label{s: supp.sims}

We now present numerical experiments for higher dimensions, specifically $d\in\{4,5,7,10,20,50\}$. We update the cluster identification rule so that particles are grouped if their pairwise geodesic distance is at most $\min\{1/(2\sqrt{\beta}),\ \pi/(2d)\}$. This heuristic choice accounts for concentration of measure on high-dimensional spheres, where random points tend to be nearly orthogonal. 


\Cref{fig:cluster_scaling_highdim} and \Cref{fig:all_masses_highdim} summarize the cluster counts and their respective masses across dimensions $d \in \{4, 5, 7, 10, 20, 50\}$. Notably, we observe empirically that the cluster masses in these higher dimensions remain strictly below the same numerical upper bound derived for $d=2$ in \Cref{thm: bound}.

\begin{figure}[h!]
    \centering
    \includegraphics[width=\linewidth]{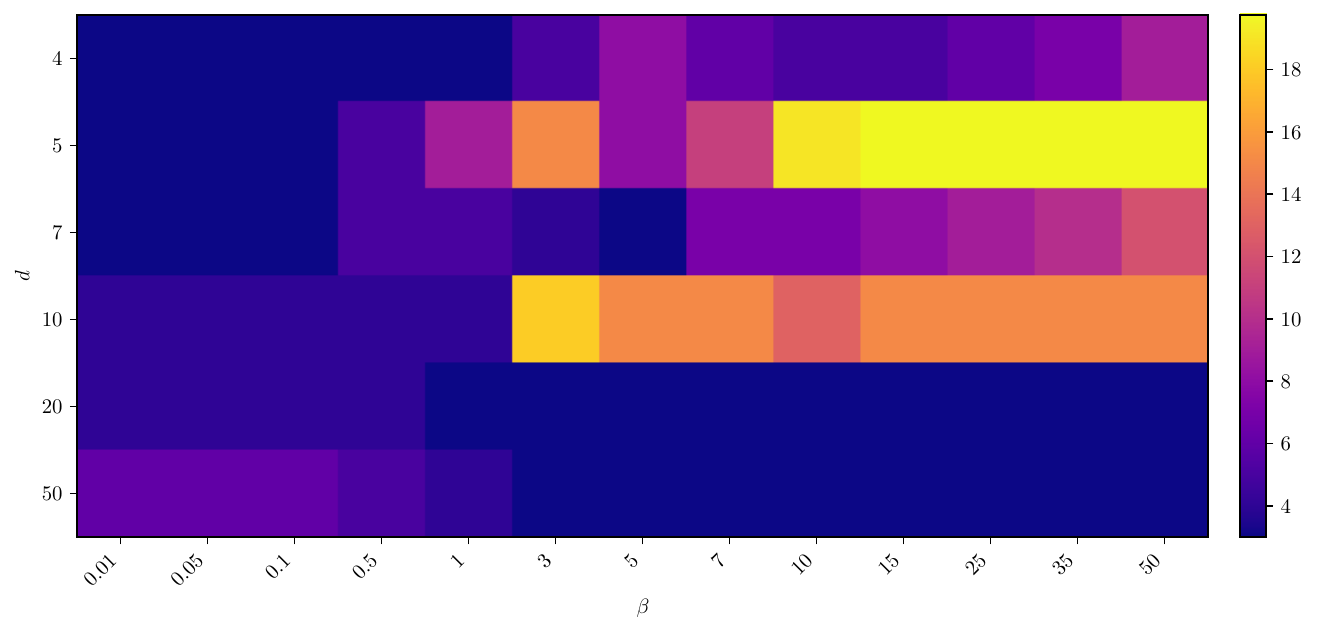}
    \caption{Number of clusters at convergence across dimensions $d\in\{4,5,7,10,20,50\}$ for gradient descent using unnormalized self-attention with a ReLU perceptron.}
    \label{fig:cluster_scaling_highdim}
\end{figure}

\begin{figure}[t!]
    \centering
    \begin{subfigure}[b]{0.32\linewidth}
        \centering
        \includegraphics[width=\linewidth]{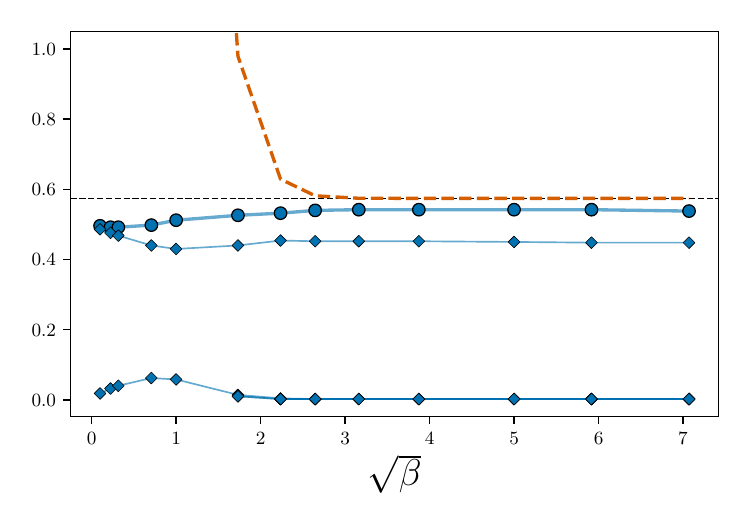}
    \end{subfigure}
    \hfill
    \begin{subfigure}[b]{0.32\linewidth}
        \centering
        \includegraphics[width=\linewidth]{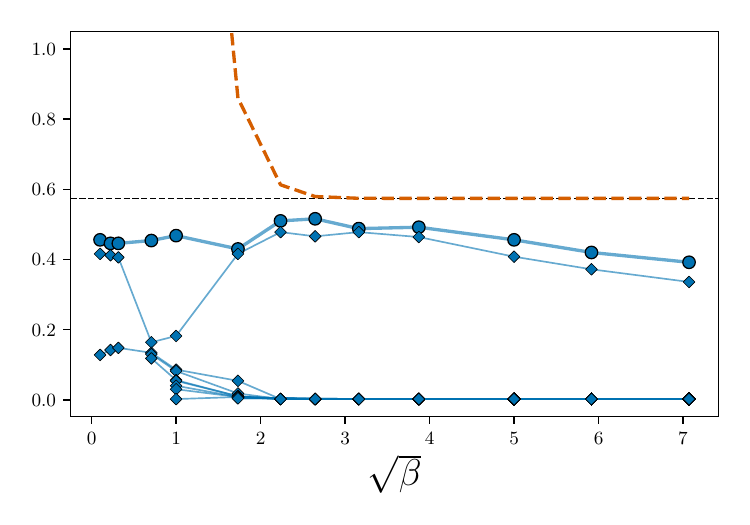}
    \end{subfigure}
    \hfill
    \begin{subfigure}[b]{0.32\linewidth}
        \centering
        \includegraphics[width=\linewidth]{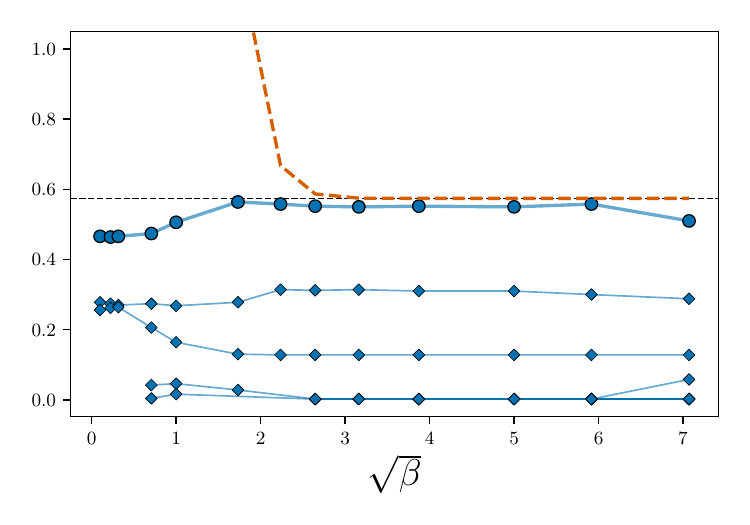}
    \end{subfigure}

    \vspace{0.5em}

    \begin{subfigure}[b]{0.32\linewidth}
        \centering
        \includegraphics[width=\linewidth]{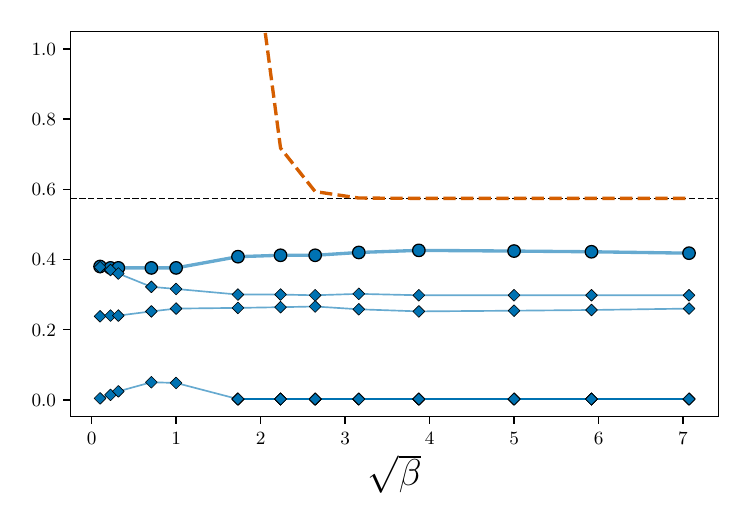}
    \end{subfigure}
    \hfill
    \begin{subfigure}[b]{0.32\linewidth}
        \centering
        \includegraphics[width=\linewidth]{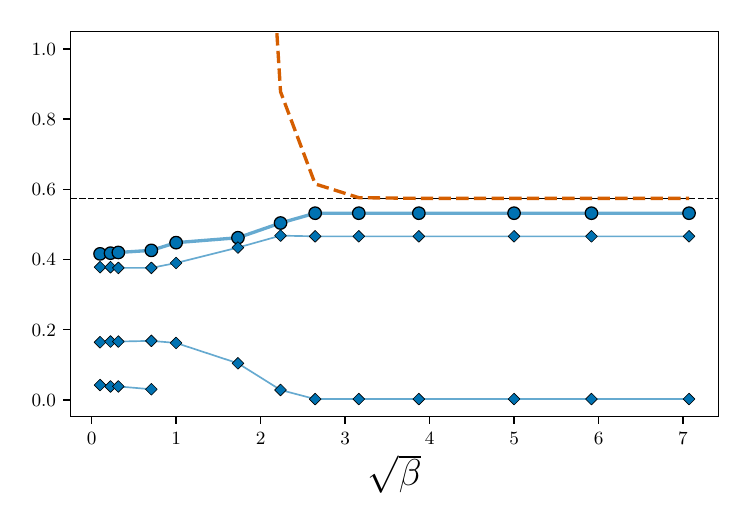}
    \end{subfigure}
    \hfill
    \begin{subfigure}[b]{0.32\linewidth}
        \centering
        \includegraphics[width=\linewidth]{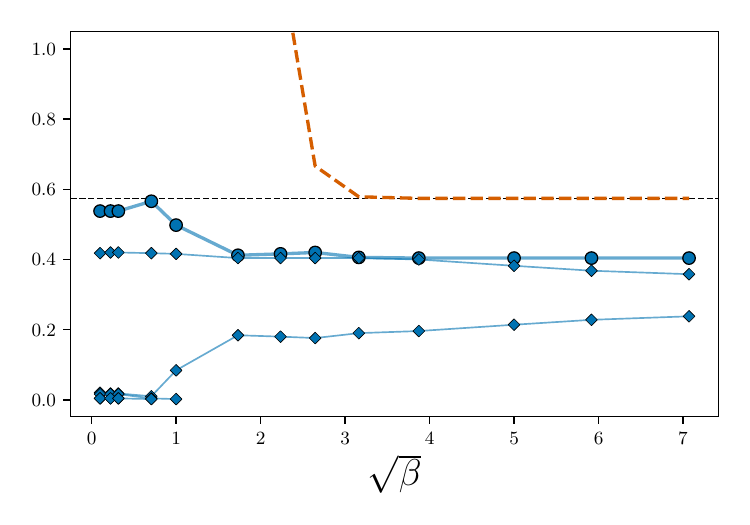}
    \end{subfigure}

    \caption{Cluster masses (in {\color{gblue}blue}, the largest being the thickest) at final time across $\sqrt{\beta}$. 
    The horizontal and \textcolor{red}{red} dashed lines represent the numerical term and the full upper bound in \eqref{eq: conclusion}, respectively.
    \textbf{Top row:} Gradient descent with ReLU ($d=4, 5, 7$).
    \textbf{Bottom row:} Same setup for higher dimensions ($d=10, 20, 50$).
    (See \Cref{sec:numerics} for setup).}
    \label{fig:all_masses_highdim}
\end{figure}

\section{Proofs}

\subsection{Preliminaries}

We denote the spherical harmonics in $L^2(\sigma_d)$ by $\{Y_{j\ell}\}_{\substack{j\ge0\\ 1\le \ell\le N_j}}$ where the dimension $N_j$ is given by
\begin{equation*} \label{eq:spharm}
N_j\coloneqq\binom{d+j-1}{j}-\binom{d+j-3}{j-2}.
\end{equation*}
Under the standard convention that $\binom{n}{k}=0$ for $k<0$, we have $N_0=1$ and $N_1=d$.

By the Funk--Hecke identity, the spherical harmonics are precisely the eigenfunctions:
\begin{equation}\label{eq:FH-basic}
\int e^{\beta x\cdot y}\,Y_{j\ell}(x)\diff\sigma_d(x) = \lambda_j(\beta)\,Y_{j\ell}(y).
\end{equation}
For $\beta > 0$, the eigenvalues $\lambda_j(\beta)$ are strictly positive and explicitly given by
$$\lambda_j(\beta) \coloneqq \Gamma\left(\frac{d}{2}\right) \left(\frac{2}{\beta}\right)^{\frac{d}{2}-1} I_{j+\frac{d}{2}-1}(\beta), \qquad j\geq 0,$$
where $I_\alpha$ is the modified Bessel function of the first kind of order $\alpha$.

\begin{lemma}\label{lem: quadpol}
For any $\beta>0$, the map $\mu \in \mathcal P(\S^{d-1})\mapsto f^\mu\in C^\infty(\S^{d-1})$, defined by
\begin{equation*}
f^\mu(x) \coloneqq \int e^{\beta x\cdot y}\diff \mu(y),
\end{equation*}
is injective. Moreover, if $f^\mu = \sum_{j=0}^{\infty} \sum_{\ell\in\llbracket1,N_j\rrbracket} \widehat f^\mu_{j\ell}\,Y_{j\ell}$ is its spherical harmonic expansion, the following properties hold:
\begin{enumerate}
\item[\textbf{(i)}] $\mu$ is absolutely continuous with respect to $\sigma_d$ with density $\frac{\diff\mu}{\diff\sigma_d} \in L^2(\sigma_d)$ if and only if
\begin{equation*}
\sum_{j=0}^{\infty}\sum_{\ell\in\llbracket1,N_j\rrbracket}\left|\frac{\widehat f^\mu_{j\ell}}{\lambda_j(\beta)}\right|^2 < \infty.
\end{equation*}
In this case, the density is given by the $L^2(\sigma_d)$-convergent series
\begin{equation}\label{eq:degreek}
\frac{\diff\mu}{\diff\sigma_d}(x) = \sum_{j=0}^{\infty}\sum_{\ell\in\llbracket1,N_j\rrbracket}\frac{\widehat f^\mu_{j\ell}}{\lambda_j(\beta)}Y_{j\ell}(x).
\end{equation}
In particular, $f^\mu$ is a polynomial of degree  $k$ if and only if $\mu$ has a density that is a polynomial of degree $k$.

\item[\textbf{(ii)}] $f^\mu$ is an even function if and only if $\mu(A)=\mu(-A)$ for every Borel set $A\subset\S^{d-1}$.
\end{enumerate}
\end{lemma}

\begin{proof}
We divide the proof into three parts.

\noindent\textbf{Injectivity.} Let $\mu_1,\mu_2\in\mathcal{P}(\S^{d-1})$ and assume
\[
\int e^{\beta x\cdot y}\diff\mu_1(y)=\int e^{\beta x\cdot y}\diff\mu_2(y)
\quad\text{for all }x\in\S^{d-1}.
\]
Set $\nu\coloneqq\mu_1-\mu_2$ and define $f^\nu(x)\coloneqq\int e^{\beta x\cdot y}\diff\nu(y)$; then $f^\nu\equiv0$.
For each $(j,\ell)$, using Fubini's theorem and \eqref{eq:FH-basic}, we compute the Fourier coefficients of $f^\nu$:
\begin{align}
0&=\widehat f^\nu_{j\ell}
=\int f^\nu(x)Y_{j\ell}(x)\diff\sigma_d(x)
=\int\left(\int e^{\beta x\cdot y}\,Y_{j\ell}(x)\diff\sigma_d(x)\right)\diff\nu(y)\nonumber
\\&=\lambda_j(\beta) \int Y_{j\ell}(y)\diff\nu(y).\label{eq:fourier.fnu}
\end{align}
Since $\lambda_j(\beta)>0$, it follows that $m_{j\ell}(\nu) \coloneqq \int Y_{j\ell}\diff\nu=0$ for all $j,\ell$. Hence, for any finite spherical polynomial $P=\sum a_{j\ell}Y_{j\ell}$, we have
\[
\int P\diff\nu=\sum a_{j\ell}\,m_{j\ell}(\nu)=0.
\]
By the Stone--Weierstrass theorem, spherical polynomials are uniformly dense in $C^0(\S^{d-1})$. Since $\nu$ is a finite signed measure, uniform approximation implies $\int h\diff\nu=0$ for all $h\in C^0(\S^{d-1})$. By the uniqueness of the Riesz representation theorem we deduce that $\nu=0$, and thus $\mu_1=\mu_2$.

For later use, note that applying the computation of \eqref{eq:fourier.fnu} to $\mu$ instead of $\nu$ yields the relation
\begin{equation}\label{eq:coeff-relation}
\widehat f^\mu_{j\ell}=\lambda_j(\beta)\,m_{j\ell},
\qquad\text{where}\quad m_{j\ell}\coloneqq\int Y_{j\ell}\diff\mu.
\end{equation}

\smallskip
\noindent\textbf{Proof of (i).} Assume $\sum_{j,\ell}\big|\widehat f^\mu_{j\ell}/\lambda_j(\beta)\big|^2<\infty$.
Since $\{Y_{j\ell}\}_{j,\ell}$ is an orthonormal basis of $L^2(\sigma_d)$, there exists a unique function $g\in L^2(\sigma_d)$ such that
\[
g(x)\coloneqq\sum_{j=0}^{\infty}\sum_{\ell\in \llbracket1,N_j\rrbracket}\frac{\widehat f^\mu_{j\ell}}{\lambda_j(\beta)}\,Y_{j\ell}(x)
\]
with convergence in $L^2(\sigma_d)$. In particular, $g\in L^1(\sigma_d)$, so $g\,\sigma_d$ defines a finite signed measure on $\S^{d-1}$. For any finite spherical polynomial $P=\sum a_{j\ell}Y_{j\ell}$, using \eqref{eq:coeff-relation},
\[
\int P\diff\mu=\sum a_{j\ell}m_{j\ell}
=\sum a_{j\ell}\frac{\widehat f^\mu_{j\ell}}{\lambda_j(\beta)}
=\int P\, g\diff\sigma_d.
\]
Since spherical polynomials are dense in $C^0(\S^{d-1})$, this identity extends to all $h\in C^0(\S^{d-1})$. By the uniqueness of the Riesz representation theorem, $\diff\mu=g\diff\sigma_d$, which yields \eqref{eq:degreek}. Conversely, if $\mu$ has a density $\rho\in L^2(\sigma_d)$ so that $\diff\mu=\rho\diff\sigma_d$, then $m_{j\ell}=\langle \rho,Y_{j\ell}\rangle_{L^2}$. Parseval's identity then gives
\[
\sum_{j,\ell}\left|\frac{\widehat f^\mu_{j\ell}}{\lambda_j(\beta)}\right|^2
=\sum_{j,\ell}|m_{j\ell}|^2=\|\rho\|_{L^2(\sigma_d)}^2<\infty.
\]
Finally, $f^\mu$ is a spherical polynomial of degree $k$ if and only if $\widehat f^\mu_{j\ell}=0$ for $j>k$ and $\widehat f^\mu_{k\ell}\neq0$, which, by \eqref{eq:degreek}, is equivalent to the density being a spherical polynomial of degree $k$.

\smallskip
\noindent\textbf{Proof of (ii).}
If $\mu(A)=\mu(-A)$ for every Borel set $A\subset\S^{d-1}$, then for every $x\in\S^{d-1}$,
\[
f^\mu(-x)=\int e^{\beta (-x)\cdot y}\diff\mu(y)
=\int e^{\beta x\cdot(-y)}\diff\mu(y)
=\int e^{\beta x\cdot y}\diff\mu(y)=f^\mu(x),
\]
so $f^\mu$ is even. 
Conversely, assume $f^\mu(-x)=f^\mu(x)$ for all $x\in\S^{d-1}$. Define $\mu^{-}\coloneqq(-\mathrm{Id})_{\#}\mu$. Then, for all $x\in\S^{d-1}$,
\begin{align*}
\int e^{\beta x\cdot y}\diff\mu^{-}(y) 
&= \int e^{\beta x\cdot (-y)}\diff\mu(y) 
= \int e^{\beta (-x)\cdot y}\diff\mu(y) \\
&= f^\mu(-x) 
= f^\mu(x) 
= \int e^{\beta x\cdot y}\diff\mu(y).
\end{align*}
By the injectivity proved above, we conclude that $\mu^{-}=\mu$, meaning $\mu(A)=\mu(-A)$ for every Borel set $A\subset\S^{d-1}$.
\end{proof}

\begin{remark}\label{rem: general-attention}
Fix a symmetric nonsingular matrix $B\in\R^{d\times d}$ and define, for $\mu\in\mathcal P(\S^{d-1})$,
\[
f_B(x)\coloneqq \int e^{x^\top B y}\diff\mu(y),\qquad x\in\S^{d-1}.
\]
Then the statements of injectivity and the parity characterization \textbf{(ii)} in \Cref{lem: quadpol} remain valid for $f_B$, with essentially the same proof and only the following modifications.

\smallskip
\noindent\textbf{Injectivity.}
Let $\mu_1,\mu_2\in\mathcal P(\S^{d-1})$ and assume $\int e^{x^\top B y}\diff\mu_1(y)=\int e^{x^\top B y}\diff\mu_2(y)$ for all $x\in\S^{d-1}$.
Set $\nu\coloneqq \mu_1-\mu_2$ and define 
$$
F_\nu(z)\coloneqq \int e^{z\cdot y}\diff\nu(y),\qquad z\in\R^d.
$$
Then $F_\nu\equiv0$ on the boundary of the ellipsoid $\Omega\coloneqq\{B^\top x\ :\ \|x\|<1\}$. Moreover, differentiating under the integral sign shows that $(\Delta-1)F_\nu=0$ on $\R^d$, since for any fixed $y\in\S^{d-1}$, the map $z \mapsto e^{z\cdot y}$ satisfies
\[
\Delta (e^{z\cdot y})=\|y\|^2 e^{z\cdot y}=e^{z\cdot y}.
\]
By uniqueness for the Dirichlet problem for $\Delta-1$ on $\Omega$, we obtain $F_\nu\equiv0$ in $\Omega$. Hence $F_\nu$ vanishes in a neighborhood of $0$, and therefore on all of $\R^d$ by real-analyticity.
Expanding at $0$, we get $\int P(y)\diff\nu(y)=0$ for every polynomial $P$. Since the restrictions of polynomials to $\S^{d-1}$ are uniformly dense in $C^0(\S^{d-1})$, it follows that $\nu=0$, and thus $\mu_1=\mu_2$.

\smallskip
\noindent\textbf{Parity.}
The proof that $f_B^\mu$ is even if and only if $\mu(A)=\mu(-A)$ for every Borel set $A\subset\S^{d-1}$ is exactly the same as in \Cref{lem: quadpol}\textbf{(ii)}, using the injectivity of $f_B^\mu$ established above.

\smallskip
\noindent We emphasize that \eqref{eq:coeff-relation} and the $L^2$ inversion criterion in \Cref{lem: quadpol}\textbf{(i)} rely on the isotropic (zonal) kernel $e^{\beta x\cdot y}$, failing as soon as $B$ introduces anisotropy.
\end{remark}

\begin{lemma}\label{lem: concavity}
Fix $d=2$ and $\beta>0$. Let $\mathsf{K}_\beta(\theta)\coloneqq e^{\beta\cos\theta}$ for $\theta\in(-\pi,\pi]$, and define 
\begin{equation} \label{eq: theta_c}
  \theta_c(\beta) \coloneqq \arccos\left(\frac{\sqrt{1+4\beta^2}-1}{2\beta}\right)\in(-\pi,\pi].
\end{equation}
Then $\mathsf{K}_\beta$ is strictly concave on $(-\theta_c(\beta),\theta_c(\beta))$, and
\[
\theta_c(\beta)=\beta^{-1/2}+O(\beta^{-3/2}) \quad \text{as }\beta\to\infty,\qquad
\theta_c(\beta)\to \frac{\pi}{2} \quad \text{as }\beta\to0^+.
\]
Moreover, for each $\lambda\in(0,1)$ there exists $\beta_0=\beta_0(\lambda)>0$ such that
\begin{equation}\label{eq:kernel-hess-bound-lambda}
\sup_{|\theta|\le \lambda\theta_c(\beta)} \mathsf{K}_\beta''(\theta)
\le
-e^{-\lambda^2/2}\,\frac{1-\lambda^2}{2}\,\beta e^\beta
\qquad\text{for all }\beta\ge \beta_0(\lambda),
\end{equation}
and there exists $\beta_1>0$ such that
\begin{equation}\label{eq:kernel-hess-bound-max}
\max_{\theta\in[\theta_c(\beta),\pi]}\mathsf{K}_\beta''(\theta)
\le 2\beta e^{\beta-3/2}
\qquad\text{for all }\beta\ge \beta_1.
\end{equation}
\end{lemma}

\begin{proof}

A direct computation gives
\begin{equation}\label{eq: K''}
\mathsf{K}_\beta''(\theta) = e^{\beta\cos\theta} (\beta^2\sin^2\theta - \beta\cos\theta).    
\end{equation}
Thus $\mathsf{K}_\beta''\left(\theta\right)=0$ is equivalent to $\beta\cos^2\theta+\cos\theta-\beta=0$, whose positive root is precisely
\begin{equation}\label{eq:critical-root}
\cos\theta_c\left(\beta\right) = \frac{\sqrt{1+4\beta^2}-1}{2\beta},
\end{equation}
so $\theta_c(\beta)\in(0,\pi/2)$ is well defined (the other root being $<-1$). Because $\mathsf{K}_\beta''(0) = -\beta e^\beta < 0$ and $\theta_c$ is the first positive root, $\mathsf{K}_\beta$ is strictly concave on $(-\theta_c(\beta), \theta_c(\beta))$. 

From \eqref{eq:critical-root}, we obtain
\[
\frac{\sqrt{1+4\beta^2}-1}{2\beta} = 1-\frac{1}{2\beta}+O\left(\beta^{-2}\right) \qquad \text{as }\beta\to\infty,
\]
yielding $\theta_c\left(\beta\right)=\beta^{-1/2}+O\left(\beta^{-3/2}\right)$ via the expansion $\cos\theta=1-\theta^2/2+O\left(\theta^4\right)$. Similarly,
\[
\frac{\sqrt{1+4\beta^2}-1}{2\beta} = \beta+O\left(\beta^3\right) \qquad \text{as }\beta\to0^+,
\]
hence $\cos\theta_c\left(\beta\right)\to0$ and $\theta_c\left(\beta\right)\to\pi/2$.  

To prove \eqref{eq:kernel-hess-bound-lambda}, we show that $\mathsf{K}_\beta''$ is strictly increasing on $[0, \theta_c(\beta)]$. Differentiating \eqref{eq: K''} gives
\begin{equation}\label{eq:K-third}
\mathsf{K}_\beta'''\left(\theta\right) = -\beta\sin\theta\, \mathsf{K}_\beta''\left(\theta\right) + e^{\beta\cos\theta} \beta\sin\theta\,\left(2\beta\cos\theta+1\right).
\end{equation}
For $0<\theta<\theta_c(\beta)$, we have $\sin\theta>0$, $\cos\theta>0$, and $\mathsf{K}_\beta''(\theta)<0$. Thus both terms on the right-hand side are positive, yielding $\mathsf{K}_\beta'''(\theta) > 0$.
Since $\mathsf{K}_\beta''$ is even, its supremum on $[-\lambda\theta_c, \lambda\theta_c]$ for every $\lambda\in(0,1)$ is attained at the boundaries. Using the expansion of $\theta_c(\beta)$, we have
\[
\cos\left(\lambda\theta_c\right) = 1 - \frac{\lambda^2}{2\beta} + O\left(\beta^{-2}\right)
\qquad\text{and}\qquad
\sin^2\left(\lambda\theta_c\right) = \frac{\lambda^2}{\beta} + O\left(\beta^{-2}\right).
\]
Substituting into \eqref{eq: K''} yields
\begin{align*}
    \mathsf{K}_\beta''\left(\lambda\theta_c\right)
&= e^{\beta\cos\left(\lambda\theta_c\right)}\left(\beta^2\sin^2\left(\lambda\theta_c\right) - \beta\cos\left(\lambda\theta_c\right)\right)\\
&= e^{\beta - \lambda^2/2}\,\left(1+O\left(\beta^{-1}\right)\right)
   \left(\beta\left(\lambda^2-1\right) + \frac{\lambda^2}{2} + O\left(1\right)\right)\\
&= -\left(1-\lambda^2\right)\,e^{\beta - \lambda^2/2}\,\beta\,\left(1+O\left(\beta^{-1}\right)\right).
\end{align*}
Therefore, for $\beta$ sufficiently large, we obtain the bound in \eqref{eq:kernel-hess-bound-lambda}.

To prove \eqref{eq:kernel-hess-bound-max}, we find the maximum of $\mathsf{K}_\beta''$ on $[\theta_c, \pi]$. Rearranging \eqref{eq:K-third}, the condition $\mathsf{K}_\beta'''(\theta)=0$ on $(0,\pi)$ reduces to $q_\beta(\cos\theta)=0$, where $q_\beta(t) \coloneqq \beta^2 t^2 + 3\beta t - \beta^2+1$. The unique root in $(-1, \cos\theta_c(\beta))$ is
\[
t_*=\frac{-3+\sqrt{4\beta^2+5}}{2\beta}.
\]
Let $\theta_*\in (\theta_c, \pi)$ be defined by $\cos\theta_* = t_*$. Since $\mathsf{K}_\beta'''(\theta)$ changes sign from positive to negative at $\theta_*$, the maximum is attained there. Because $q_\beta(t_*)=0$, we can express the quadratic term as $\beta^2(1-t_*^2) = 3\beta t_* + 1$. Substituting this into \eqref{eq: K''} simplifies the evaluation to
\[ 
\mathsf{K}_\beta''\left(\theta_*\right) = e^{\beta t_*} \left( 2\beta t_* + 1 \right). 
\] 
Finally, expanding $t_* = 1 - \frac{3}{2\beta} + \frac{5}{8\beta^2} + O(\beta^{-4})$ for large $\beta$ provides
\begin{align*} 
\mathsf{K}_\beta''\left(\theta_*\right) 
&= e^{\beta-\frac{3}{2}}\left(1+\frac{5}{8\beta}+O\left(\beta^{-2}\right)\right) \,2\beta\left(1-\frac{1}{\beta}+O\left(\beta^{-2}\right)\right) \\ 
&= 2\beta e^{\beta-\frac{3}{2}}\left(1-\frac{3}{8\beta}+O\left(\beta^{-2}\right)\right). 
\end{align*} 
Since $1-\frac{3}{8\beta} < 1$ for large $\beta$, this implies \eqref{eq:kernel-hess-bound-max} for all $\beta\ge\beta_1$.
\end{proof}

\subsection{Proof of \Cref{thm: circle} and \Cref{thm: circle.gelu}}
\label{sec: proof.1}

\begin{proof}[Proof of \Cref{thm: circle}]
Set $x(\theta)=(\cos\theta,\sin\theta)$ and $\tau(\theta)=(-\sin\theta,\cos\theta)$. For $\theta\in[0,2\pi)$ define
\begin{align*}
  f_\beta\left(\theta\right)&\coloneqq \int e^{\beta\cos\left(\phi-\theta\right)}\sin\left(\phi-\theta\right)\diff\mu(\phi),\\ 
  f_\upvartheta\left(\theta\right)&\coloneqq 
\sum_{j\in\llbracket1,2\rrbracket}\omega_j \,\left(a_j\cdot x(\theta)\right)_+\, a_j\cdot \tau(\theta),
\end{align*}
and set $g\coloneqq f_\beta+f_\upvartheta$, corresponding to the gradient $\nabla \frac{\delta \mathsf{E}_\beta}{\delta \mu}[\mu] + \mathsf{u}_\upvartheta$ expressed in angular coordinates.

Let $\mathscr Z\coloneqq \{\theta\in[0,2\pi): a_j\cdot x(\theta)=0 \text{ for some } j\}$. Each equation $a_j\cdot x(\theta)=0$ has at most two solutions, so $\mathscr Z$ is finite. On any connected component $I\subset[0,2\pi)\setminus\mathscr Z$, the active set
\[
J_I\coloneqq\{j:\ a_j\cdot x(\theta)>0\ \text{for all }\theta\in I\}
\]
is constant, and for $\theta\in I$ we have
\[
\left(a_j\cdot x(\theta)\right)_+=\mathbf1_{j\in J_I}\,a_j\cdot x(\theta),
\qquad
\frac{\diff}{\diff\theta}\left(a_j\cdot x(\theta)\right)_+=\mathbf1_{j\in J_I}\,a_j\cdot\tau(\theta).
\]
Thus $f_\upvartheta$ is real–analytic on each $I$. On the other hand, since 
\[
\mathsf K_\beta\left(\psi\right)\coloneqq e^{\beta\cos\psi}
=I_0(\beta)+2\sum_{n\ge1} I_n(\beta)\cos(n\psi),
\]
its Fourier series is absolutely convergent. Hence the convolution with the measure $\mu$,
\[
\left(\mathsf K_\beta*\mu\right)\left(\theta\right)\coloneqq\int \mathsf K_\beta\left(\theta-\phi\right)\diff\mu(\phi),
\]
is real-analytic on $\S^1$, and $f_\beta(\theta)=\frac{1}{\beta}\frac{\diff}{\diff\theta}(\mathsf K_\beta*\mu)(\theta)$ is real-analytic as well. All in all, we deduce that $g$ is continuous on $[0,2\pi)$ and real-analytic on each $I$.

For each component $I$, from \eqref{eq: steady.state} we get
\begin{equation}\label{eq:support-in-zero}
\supp\mu\cap I\ \subset\ \{\theta\in I:\ g(\theta)=0\}.
\end{equation}
Assume, for a contradiction, that $\supp\mu$ is infinite. Then $\supp\mu\cap I$ must be infinite for at least one of the finitely many components $I$. Define the global real-analytic function
\[
\widetilde g\left(\theta\right) \coloneqq f_\beta\left(\theta\right) + \sum_{j\in J_I}\omega_j\,\langle a_j,x(\theta)\rangle\,\langle a_j,\tau(\theta)\rangle,
\]
which satisfies $\widetilde g(\theta)=g(\theta)$ for all $\theta\in I$. Hence, by \eqref{eq:support-in-zero}, the zero set of $\widetilde g$ contains infinitely many points in $I$ and thus has an accumulation point in $\S^1$. By the identity theorem for real-analytic functions, it follows that $\widetilde g\equiv 0$ on $\S^1$. Equivalently,
\[
f_\beta\left(\theta\right) = -\sum_{j\in J_I}\omega_j\,\langle a_j,x(\theta)\rangle\,\langle a_j,\tau(\theta)\rangle
\qquad \text{for all }\theta \in \S^1.
\]
Integrating in $\theta$, and recalling $f_\beta=\frac{1}{\beta}(\mathsf K_\beta*\mu)'$, we find that $g=H'$ on $I$, for
\[
H\left(\theta\right)\coloneqq \frac{1}{\beta}\left(\mathsf K_\beta*\mu\right)\left(\theta\right)+\sum_{j}\frac{\omega_j}{2}\,\left( a_j\cdot x(\theta)\right)_+^{\,2}.
\]
Since $\widetilde g\equiv0$ on $\S^1$, we obtain the global identity
\begin{equation}\label{eq:K-equals-quad-global}
\left(\mathsf K_\beta*\mu\right)\left(\theta\right)=C_I-\beta\sum_{j\in J_I}\frac{\omega_j}{2}\,\left( a_j\cdot x(\theta)\right)^{\,2}
\qquad \left(\text{for }\theta\in[0,2\pi)\right)
\end{equation}
for some constant $C_I$.

By \Cref{lem: quadpol}\textbf{(i)}, \eqref{eq:K-equals-quad-global} implies that $\mu$ is absolutely continuous with an $L^2$-density which is a trigonometric polynomial of degree $\le 2$. As a nonnegative real-analytic function on $\S^1$, this density cannot vanish on an open arc unless it is identically zero; hence $\supp\mu=\S^1$.

Let $I'$ be a component adjacent to $I$, and let $\theta_*$ be the common boundary point. Since $\supp\mu=\mathbb S^1$, we have $\supp\mu\cap I'$ is infinite, and repeating the argument used on $I$ gives another global identity
\begin{equation}\label{eq:K-equals-quad-global-bis}
\left(\mathsf{K}_\beta*\mu\right)\left(\theta\right)
= C_{I'}-\beta\sum_{j\in J_{I'}}\frac{\omega_j}{2}\,\left( a_j\cdot x(\theta)\right)^{2}
\qquad \left(\text{for }\theta\in[0,2\pi)\right).
\end{equation}
Let $B\coloneqq J_I\triangle J_{I'}$ denote the symmetric difference. Clearly $B\neq\varnothing$, since $J_I\neq J_{I'}$ because at least one index changes sign when crossing $\theta_*$. Subtracting \eqref{eq:K-equals-quad-global-bis} from \eqref{eq:K-equals-quad-global} we obtain
\begin{equation}\label{eq:toggle-diff-general}
\frac{\beta}{2}\sum_{j\in B}\xi_j\,\omega_j\,\left(a_j\cdot x(\theta)\right)^{2}
\ \equiv\ C_I-C_{I'} \qquad \left(\text{for }\theta\in[0,2\pi)\right),
\end{equation}
where $\xi_j\in\{\pm 1\}$ records whether $j$ is added or removed when passing from $I$ to $I'$. For every $j\in B$ we have $a_j\cdot x(\theta_*)=0$, hence evaluating \eqref{eq:toggle-diff-general} at $\theta=\theta_*$ yields $C_I=C_{I'}$ and therefore
\begin{equation}\label{eq:quad-combo-zero}
\sum_{j\in B}\xi_j\,\omega_j\,\left(a_j\cdot x(\theta)\right)^{2}\ \equiv\ 0
\qquad \left(\text{for }\theta\in[0,2\pi)\right).
\end{equation}
The left-hand side of \eqref{eq:quad-combo-zero} is exactly the difference of $\left(\mathsf{v}_\upvartheta\circ x\right)\left(\theta\right)$ on $I$ and $I'$, which therefore vanishes identically on $[0,2\pi)$. Hence $\left(\mathsf{v}_\upvartheta\circ x\right)\left(\theta\right)$ coincides on $\overline{I\cup I'}$ with a single trigonometric polynomial of degree at most~$2$.

Since this argument applies to every adjacent pair of components, it follows by connectivity that $\mathsf{v}_\upvartheta\circ x$ is given in the full $\S^1$ by a single trigonometric polynomial of degree at most~$2$. In particular, $\mathsf{v}_\upvartheta$ is real-analytic on $\S^1$, contradicting our assumption. This shows that $\supp\mu$ is finite, and combining this with \eqref{eq:support-in-zero} yields that $\mu$ is purely atomic with finitely many atoms.

\end{proof}

\begin{remark}\label{rem:piecewise-poly}
   The conclusion of \Cref{thm: circle} holds for any $\sigma$ globally Lipschitz, piecewise polynomial (e.g., the leaky ReLU), provided $\mathsf{v}_\upvartheta$ is not real-analytic. 
    
    The proof applies \emph{mutatis mutandis}. On each component $I$, the field $f_\upvartheta$ now becomes a trigonometric polynomial of some finite degree. Consequently, if the support accumulates in $I$, analytic continuation implies that the global identity \eqref{eq:K-equals-quad-global} holds with a polynomial of some finite degree. \Cref{lem: quadpol}\textbf{(i)} then guarantees a global polynomial density, forcing $\supp\mu=\S^1$. The subtraction argument via \eqref{eq:toggle-diff-general} yields the contradiction exactly as before.
    
    The same extension applies to \Cref{thm: any.d}.
\end{remark}

\begin{proof}[Proof of \Cref{thm: circle.gelu}]

Let $\mu\in\mathcal{P}(\S^1)$ be a strict SOPD Wasserstein critical point of $\mathsf E_{\beta,\upvartheta}$. In particular, $\mathrm{Hess}_\mu\mathsf E_{\beta,\upvartheta}$ is well-defined at $\mu$. Write $x(\theta)=(\cos\theta,\sin\theta)$ and $\mathsf K_\beta(\phi)=e^{\beta\cos\phi}$, and define
\[
H\left(\theta\right)\coloneqq \frac{1}{\beta}\left(\mathsf K_\beta*\mu\right)\left(\theta\right)+\frac12\,\left(\mathsf v_\upvartheta\circ x\right)\left(\theta\right).
\]
Since $\sigma$ is real-analytic, the potential $\theta\mapsto \left(\mathsf v_\upvartheta\circ x\right)\left(\theta\right)$ is real-analytic on $\S^1$. The convolution term $(\mathsf K_\beta*\mu)$ is also real-analytic. Hence $H$ is real-analytic on $\S^1$.

Assume for contradiction that $\supp\mu$ is infinite. Then $\supp\mu$ has an accumulation point in $\S^1$. The first-order condition \eqref{eq: steady.state} forces $H'$ to vanish on $\supp\mu$. Since $H'$ is real-analytic and vanishes on a set with an accumulation point, the identity theorem implies $H'\equiv 0$ on $\S^1$. Consequently, $H''\equiv 0$ on $\S^1$, which gives the global cancellation:
\begin{equation}\label{eq:global-cancel-analytic}
\frac{1}{\beta}\,\left(\mathsf K_\beta''*\mu\right)\left(\theta\right)
+\frac12\,\partial_\theta^2\left(\mathsf v_\upvartheta\circ x\right)\left(\theta\right)\ =\ 0
\qquad \left(\text{for }\theta\in\S^1\right).
\end{equation}

To compute the second variation, let $\xi \in \mathsf T_\mu\mathcal P(\mathbb S^1)$. Since the Wasserstein Hessian is well-defined at $\mu$, 
\begin{align}
\mathrm{Hess}_\mu \mathsf E_{\beta,\upvartheta}(\xi,\xi)
&= \frac{1}{2\beta}\iint \mathsf K_\beta''\left(\theta-\phi\right)\left(\xi(\theta)-\xi(\phi)\right)^2\diff\mu(\theta)\diff\mu(\phi)\nonumber\\
&\quad+\frac12\int \partial_\theta^2\left(\mathsf v_\upvartheta\circ x\right)\left(\theta\right)\,\xi(\theta)^2\diff\mu(\theta).\label{eq:Hess-bilinear.std}
\end{align}
Expanding the square using the symmetry of $\mathsf K_\beta''$, we rewrite this as:
\begin{align}
\mathrm{Hess}_\mu \mathsf E_{\beta,\upvartheta}(\xi,\xi)
&=\int \left[\frac{1}{\beta}\left(\mathsf K_\beta''*\mu\right)\left(\theta\right)+\frac12\partial_\theta^2\left(\mathsf v_{\upvartheta}\circ x\right)\left(\theta\right)\right]\xi(\theta)^2\diff\mu(\theta)\nonumber\\
&\quad-\frac{1}{\beta}\iint \mathsf K_\beta''\left(\theta-\phi\right)\,\xi(\theta)\xi(\phi)\diff\mu(\theta)\diff\mu(\phi).\label{eq:Hess-bilinear}
\end{align}
By \eqref{eq:global-cancel-analytic}, the term in brackets vanishes identically. Thus, the Hessian reduces to a bilinear form:
\begin{equation}\label{eq:Hess-bilinear-analytic}
\mathrm{Hess}_\mu \mathsf E_{\beta,\upvartheta}(\xi,\xi)
=-\frac{1}{\beta}\iint \mathsf K_\beta''\left(\theta-\phi\right)\,\xi(\theta)\xi(\phi)\diff\mu(\theta)\diff\mu(\phi).
\end{equation}
We construct a sequence of tangent vectors along which \eqref{eq:Hess-bilinear-analytic} is arbitrarily small, thereby contradicting strict positivity. 

Fix a small arc $J_\delta\subset\S^1$ centered at an accumulation point of $\supp\mu$, with diameter $\le\delta$, so that $\mu(J_\delta)>0$.
Pick two disjoint subarcs $I_1,I_2\subset J_\delta$ with $\mu(I_i)>0$. Choose $u_i\in C_c^\infty(I_i)$ non-constant, view it as a smooth function on $\S^1$ by extending it by zero outside $I_i$, and set $\eta_i\coloneqq u_i'$.
Then $\eta_i\in \mathsf T_\mu\mathcal P(\S^1)$ (recall \eqref{eq:otto.tangent}) and $\eta_i$ is supported in $I_i$. Moreover, by choosing $u_i$ so that $u_i'$ does not vanish on a neighborhood of some point in $\supp\mu\cap I_i$, we ensure $\|\eta_i\|_{L^2(\mu)}>0$.

Set $m_i\coloneqq \int \eta_i\diff\mu$. If $(m_1,m_2)\neq (0,0)$, define
\[
\xi_0 \coloneqq m_2\,\eta_1 - m_1\,\eta_2,
\]
so that $\int \xi_0\diff\mu = m_2m_1 - m_1m_2 =0$. If $m_1=m_2=0$, simply set $\xi_0\coloneqq \eta_1$ (which already satisfies $\int \xi_0\diff\mu=0$).
In either case, since $\eta_1$ and $\eta_2$ have disjoint supports and $\xi_0$ is not identically zero, we have $\xi_0\not\equiv 0$.

Normalize $\xi_\delta \coloneqq \xi_0/\|\xi_0\|_{L^2(\mu)}$. Then
\[
\xi_\delta\in \mathsf T_\mu\mathcal P(\mathbb S^1),\qquad \supp \xi_\delta\subset J_\delta,\qquad\text{and}\qquad\|\xi_\delta\|_{L^2(\mu)}=1.
\] 
Moreover, by construction, $\int \xi_\delta \diff\mu = 0$. Evaluating \eqref{eq:Hess-bilinear-analytic} at $\xi=\xi_\delta$:
\[
\mathrm{Hess}_\mu \mathsf E_{\beta,\upvartheta}(\xi_\delta,\xi_\delta) = -\frac{1}{\beta}\iint_{J_\delta \times J_\delta} \mathsf K_\beta''\left(\theta-\phi\right)\,\xi_\delta(\theta)\xi_\delta(\phi)\diff\mu(\theta)\diff\mu(\phi).
\]
Since $\int \xi_\delta \diff\mu = 0$, adding a constant to the kernel does not change the integral. We replace $\mathsf K_\beta''(\theta-\phi)$ by $\mathsf K_\beta''(\theta-\phi) - \mathsf K_\beta''(0)$.
Let $\omega(\delta) \coloneqq \sup_{|h|\le\delta} \left|\mathsf K_\beta''(h)-\mathsf K_\beta''(0)\right|$ for $\delta\in[0,\pi)$. Then, 
\begin{align*}
\left|\mathrm{Hess}_\mu \mathsf E_{\beta,\upvartheta}(\xi_\delta,\xi_\delta)\right|
&\le \frac{1}{\beta} \iint \left|\mathsf K_\beta''\left(\theta-\phi\right)-\mathsf K_\beta''(0)\right| \left|\xi_\delta(\theta)\right| \left|\xi_\delta(\phi)\right| \diff\mu(\theta) \diff\mu(\phi)\\
&\le \frac{\omega(\delta)}{\beta} \|\xi_\delta\|_{L^1(\mu)}^2.
\end{align*}
Using Cauchy-Schwarz, $\|\xi_\delta\|_{L^1(\mu)}^2 \le \|\xi_\delta\|_{L^2(\mu)}^2 = 1$. Thus, the Hessian is bounded by $\omega(\delta)/\beta$, which tends to 0 as $\delta \to 0$.
This implies 
\[
\inf\left\{\mathrm{Hess}_\mu \mathsf E_{\beta,\upvartheta}(\xi,\xi) \ :\ \xi\in\mathsf T_\mu\mathcal P(\mathbb S^1),\ \|\xi\|_{L^2(\mu)}=1\right\}=0,
\]
contradicting \eqref{eq:strict2order}. Therefore 
$\supp\mu$ must be finite, and $\mu$ is purely atomic with finitely many atoms.

\end{proof}

\subsection{Proof of \Cref{thm: any.d}}

\begin{proof}[Proof of \Cref{thm: any.d}]

We split the proof into two parts.

\subsubsection*{Part I. Non-absolute continuity}
We reuse the definitions of $(\mathscr Z, I, J_I)$ introduced in \eqref{eq: defZd}--\eqref{eq: defJI}.  

Let $\mu\in\mathcal P(\S^{d-1})$ solve \eqref{eq: steady.state}. 
Assume first that $\sigma(s)=s_+$ and $\mathsf v_\upvartheta$ is not real-analytic on $\S^{d-1}$.  On $I$ the signs of $x\mapsto  a_j\cdot x$ are fixed, hence 
\begin{equation}\label{eq: quad.pol.I}
    \mathsf v_\upvartheta(x)
    =\sum_{j\in\llbracket 1,d\rrbracket}\omega_j\left( a_j\cdot x\right)_+^{2}
    =\sum_{j\in J_I}\omega_j\left( a_j\cdot x\right)^{2}\qquad \left(\text{for } x\in I\right).
\end{equation}
By using the Weierstrass $M$-test on the expansion 
\[
e^{\beta x\cdot y}=\sum_{m\ge0}\frac{\beta^m}{m!}\left(x\cdot y\right)^m,
\]
we deduce that 
$\frac{\delta\mathsf{E}_\beta}{\delta\mu}[\mu]$ is real-analytic on $\S^{d-1}$. Thus,  $H\coloneqq\frac{\delta\mathsf{E}_{\beta,\upvartheta}}{\delta\mu}[\mu]$ and $g\coloneqq\nabla H$ are both real-analytic on every $I$.


For each $I$, from \eqref{eq: steady.state} we have
\begin{equation}\label{eq: suppmu.zero.g}
    \supp\mu\cap I\subset\{x\in I: g(x)=0\}.     
\end{equation}
Assume for contradiction that $\sigma_d(\supp\mu)>0$. Since $\sigma_d(\mathscr Z)=0$, we deduce that $\sigma_d(\supp\mu\cap I)>0$ for some component $I$. 
Then by \eqref{eq: suppmu.zero.g}, the zero set of $g$ has positive measure in $I$, hence $g\equiv0$ on $I$ by analyticity \cite{Mityagin2020}. Thus $H$ is constant on $I$ and, using \eqref{eq: quad.pol.I},
\begin{equation}\label{eq:global-Id}
    \frac{\delta\mathsf{E}_\beta}{\delta\mu}[\mu](x)
    =C_I-\frac12\sum_{j\in J_I}\omega_j\left(a_j\cdot x\right)^2\qquad \left(\text{for } x\in I\right)
\end{equation}
for some $C_I\in\R$. Both sides are real-analytic on $\S^{d-1}$ and $I$ is open, so by the identity theorem this equality holds on $\S^{d-1}$.

By \Cref{lem: quadpol}\textbf{(i)}, it follows that $\mu$ is absolutely continuous with a density $\rho$ which is a spherical polynomial of degree $\le 2$; in particular, $\rho$ is real-analytic and not identically zero. Since the zero set of a nontrivial real-analytic function on $\S^{d-1}$ has null $\sigma_d$-measure, it follows that $\supp\mu=\S^{d-1}$.

 Since $\supp\mu=\S^{d-1}$, the stationarity condition \eqref{eq: steady.state} holds on all of $\S^{d-1}$. Hence $g=\nabla H=0$ on $\S^{d-1}$, and therefore $H$ is constant on $\S^{d-1}$. On the other hand, \eqref{eq:global-Id} already shows that $\frac{\delta\mathsf E_\beta}{\delta\mu}[\mu]$ is the restriction to $\S^{d-1}$ of a quadratic polynomial. It follows that
\[
\mathsf v_\upvartheta
=2H-2\frac{\delta\mathsf E_\beta}{\delta\mu}[\mu]
\]
is also the restriction to $\S^{d-1}$ of a quadratic polynomial. In particular, $\mathsf v_\upvartheta$ is real-analytic on $\S^{d-1}$, contradicting the assumption. Therefore $\sigma_d(\supp\mu)=0$.

Now assume that $\sigma$ is real-analytic and that $\mu$ is a strict SOPD Wasserstein critical point. If $\sigma_d(\supp\mu)>0$, then the real-analytic vector field $g=\nabla H$ vanishes on a set of positive $\sigma_d$-measure. Hence $g\equiv0$ on $\S^{d-1}$ by analyticity \cite{Mityagin2020}, and therefore $H$ is constant on $\S^{d-1}$. The same localization argument as in the proof of \Cref{thm: circle.gelu} then yields unit-norm tangent perturbations along which the corresponding Hessian values become arbitrarily small, contradicting \eqref{eq:strict2order}.

\subsubsection*{Part II. Atomicity}

Fix $\mu\in\mathcal P(\mathbb S^{d-1})$ and define
\[
g_{\beta,\upvartheta}(x) \coloneqq \nabla\frac{\delta \mathsf E_{\beta,\upvartheta}}{\delta \mu}[\mu](x)
= \int e^{\beta x\cdot y}\proj_x y \diff \mu(y)
+\sum_{j\in\llbracket1,d\rrbracket} \omega_j\sigma(a_j\cdot x)\proj_x a_j.
\]
If $\mu$ is stationary, then $\supp\mu\subset\{x\in\mathbb S^{d-1}\colon\ g_{\beta,\upvartheta}(x)=0\}$, by \eqref{eq: steady.state}.

Our goal is to show that, for a dense subset of parameters, all zeros of $g_{\beta,\upvartheta}$ are nondegenerate (and thus isolated). To achieve this, we invoke the parametric transversality theorem
\cite[Thm.~6.35]{Lee2012}. This theorem guarantees the result provided that the map $(x,\beta,\upvartheta)\mapsto g_{\beta,\upvartheta}(x)$ is transverse to the zero section.

Explicitly, we need to verify that for all $(x,\beta,\upvartheta)$ such that $g_{\beta,\upvartheta}(x)=0$, the differential map surjects onto the tangent space:
\begin{equation}\label{eq: cond.transversality}
\mathrm{Im}\;D_{(\beta,\upvartheta)} g_{\beta,\upvartheta}(x)+\mathrm{Im}\,D_x g_{\beta,\upvartheta}(x)
=\mathsf T_x\mathbb S^{d-1}.
\end{equation}
If \eqref{eq: cond.transversality} holds, the theorem implies that for almost every $(\beta,\upvartheta)$, and hence for a dense set of parameters, the zeros of $x\mapsto g_{\beta,\upvartheta}(x)$ are nondegenerate.

Assume first that $\sigma$ is real-analytic and $\sigma(s)\neq 0$ for all $s\neq0$. For fixed $\mu$, the map $(x,\beta,\upvartheta)\mapsto g_{\beta,\upvartheta}(x)$ is then real-analytic on $\mathbb S^{d-1}\times\R_{>0}\times(\R^{d+1})^d$ where \(\upvartheta=(a_j,\omega_j)_{j\in\llbracket 1,d\rrbracket}\) are the perceptron parameters. We restrict to the open dense parameter set where $a_j$ are linearly independent and $\omega_j\neq0$ for all $j$. On this subset, for every $x\in\mathbb S^{d-1}$ there exists $j\in\llbracket1,d\rrbracket$ such that $a_j\cdot x\neq0$, thus $\omega_j\sigma\left(a_j\cdot x\right)\neq 0$ thanks to the assumption on $\sigma$.

The differential with respect to $a_j$ evaluated at $h\in\R^d$ is given by
\[
D_{a_j} g_{\beta,\upvartheta}(x)[h]= \omega_j\left(\sigma'\left(a_j\cdot x\right)\left(x\cdot h\right)\proj_x a_j+\sigma\left(a_j\cdot x\right)\proj_x h\right).
\]
By restricting to $h\in\mathsf T_x\mathbb S^{d-1}$ so that $\proj_x h = h$ and $x\cdot h = 0$, we deduce that 
\[
D_{a_j} g_{\beta,\upvartheta}(x)[h]= \omega_j\sigma\left(a_j \cdot x\right)h.
\]
Since $\omega_j\sigma\left(a_j \cdot x\right) \neq 0$, the linear map $D_{a_j} g_{\beta,\upvartheta}(x)$ is onto $\mathsf T_x\mathbb S^{d-1}$. In particular, \eqref{eq: cond.transversality} holds. By the parametric transversality theorem \cite[Thm.~6.35]{Lee2012}, the set of parameters for which all zeros of $g_{\beta,\upvartheta}$ are nondegenerate is dense. Moreover, since $\mathbb S^{d-1}$ is compact, this set of parameters is also open. Finally, since nondegenerate zeros are isolated by the inverse function theorem, the compactness of $\mathbb{S}^{d-1}$ guarantees that $g_{\beta,\upvartheta}$ has only finitely many zeros.

Now assume $\sigma(s)=s_+$. Then $x\mapsto g_{\beta,\upvartheta}(x)$ is smooth on each component $I$ of $\mathbb S^{d-1}\setminus\mathscr Z$. As before, assume $a_j$ linearly independent and $\omega_j\neq0$. Fix a component $I$ with $J_I\neq\varnothing$ and let $x\in I$ with $g_{\beta,\upvartheta}(x)=0$. Choose any $j\in J_I$. Then for any $h\in\mathsf T_x\mathbb S^{d-1}$, the differential with respect to $a_j$ is given by
\[
D_{a_j} g_{\beta,\upvartheta}(x)[h] = \omega_j\sigma\left(a_j \cdot x\right)h = \omega_j \left(a_j \cdot x\right) h.
\]
Since $a_j \cdot x > 0$ on $I$ and $\omega_j\neq 0$, this map is a scaled isometry, hence surjective onto $\mathsf T_x\mathbb S^{d-1}$. By the parametric transversality theorem, for a dense set of parameters, the zeros in $I$ are nondegenerate and thus form a discrete set.

To handle the boundaries, we consider the submanifolds $F$ defined by the intersection of some hyperplanes $\{x: a_k\cdot x=0\}$ inside $\mathcal{A} \coloneqq \bigcup_j \{a_j\cdot x > 0\}$. Since the restriction of $g_{\beta,\upvartheta}$ to $F$ takes values in the larger space $\mathsf T\mathbb S^{d-1}$, we consider the projected field $g^F(x) \coloneqq \proj_{\mathsf T_x F} g_{\beta,\upvartheta}(x)$ for $x\in F$. Note that if $g_{\beta,\upvartheta}(x)=0$, then necessarily $g^F(x)=0$.
We apply transversality to the map $g^F: F \to \mathsf T F$. Let $x \in F$ be a zero of $g^F$. Since $x \in \mathcal{A}$, there exists $j$ such that $a_j \cdot x > 0$. For any $h \in \mathsf T_x F$, the differential with respect to $a_j$ acts as
\[
D_{a_j} g^F(x)[h] = \proj_{\mathsf T_x F}\left(\omega_j \left(a_j \cdot x\right) h\right) = \omega_j \left(a_j \cdot x\right) h.
\]
This map surjects onto $\mathsf T_x F$. Thus, for a dense set of parameters, the zeros of $g^F$ are isolated in $F$. By inclusion, the zeros of $g_{\beta,\upvartheta}$ in $F$ are also isolated.

Since $\mathscr Z$ consists of finitely many hyperplanes, $\mathcal{A}$ is the finite union of such $I$ and $F$. Consequently, the set of zeros of $g_{\beta,\upvartheta}$ in $\mathcal{A}$ is a finite union of discrete sets. If $\mu$ is stationary, then $\supp\mu\cap\mathcal{A}$ is countable.
\end{proof}

\begin{remark}\label{rem:genericity-vacuity}
\Cref{thm: any.d}\textbf{(ii)} does not strictly preclude the existence of non-atomic stationary measures; rather, it implies that such a measure $\mu$ can only be stationary if the parameters lie in the exceptional (non-generic) set associated with $\mu$.
A concrete example is the activation $\sigma(s)=s$, which yields a quadratic potential $\mathsf v_{\upvartheta}(x)=\sum_j \omega_j (a_j\cdot x)^2$.
If a stationary measure $\mu$ has support with non-empty interior, the stationarity condition forces the attention field to be locally (and thus globally) quadratic.
By \Cref{lem: quadpol}\textbf{(i)}, this implies $\mu$ admits a smooth density (a spherical polynomial of degree $\le 2$). Since such a measure is not finitely supported, the parameters allowing it must necessarily lie in the non-generic set specific to $\mu$.
\end{remark}

\subsection{Proofs of \Cref{thm: bound} and \Cref{cor: bound}}

\begin{proof}[Proof of \Cref{thm: bound}]
We prove a more general estimate for an arbitrary parameter $\lambda \in (0,1)$, from which the theorem follows by setting $\lambda=1/2$.

Restricting the energy functional $\mathsf{E}_{\beta,\upvartheta}$ to atomic measures of the form $\nu = \sum_{i\in\llbracket 1,N\rrbracket} m_i \delta_{\phi_i}$ (where $m_i$ are fixed as in \eqref{eq: atomic.thm.bound}), it becomes a function of the angular coordinates $\boldsymbol{\phi}=(\phi_1, \dots, \phi_N) \in (\mathbb{R}/2\pi\mathbb{Z})^N$:
\begin{equation*}
\mathsf{E}_{\beta,\upvartheta}\left(\boldsymbol{\phi}\right) 
= \frac{1}{2\beta} \sum_{i,j\in\llbracket 1,N\rrbracket} m_i m_j \mathsf{K}_\beta\left(\phi_i - \phi_j\right) 
+ \frac{1}{2}\sum_{i\in\llbracket 1,N\rrbracket} m_i \left(\mathsf{v}_\upvartheta \circ x\right)\left(\phi_i\right).
\end{equation*}
Since $\mu$ is a SOPD Wasserstein critical point, it satisfies \eqref{eq: steady.state} and \eqref{eq:2order}. The Wasserstein Hessian at $\mu$ coincides with the Euclidean Hessian of $\mathsf{E}_{\beta,\upvartheta}$ in the variables $\boldsymbol{\phi}$, evaluated at the configuration $\boldsymbol{\theta}$. In particular, for every $\xi\in\mathbb{R}^N$,
\begin{equation}\label{eq: Hess.pos.def}
 \nabla^2 \mathsf{E}_{\beta,\upvartheta}\left(\boldsymbol{\theta}\right)\left[\xi,\xi\right]
= \mathrm{Hess}_\mu \mathsf{E}_{\beta,\upvartheta}\left(\xi,\xi\right)\ \ge\ 0.
\end{equation}
Thus, the matrix $\nabla^2 \mathsf{E}_{\beta,\upvartheta}\left(\boldsymbol{\theta}\right)$, with entries defined by 
\[
\frac{\partial^2 \mathsf{E}_{\beta,\upvartheta}}{\partial \phi_i \partial \phi_j}\left(\boldsymbol{\theta}\right)
=\begin{cases}
    -\frac{1}{\beta} m_i m_j\,\mathsf{K}_\beta''\left(\theta_i-\theta_j\right),
\qquad&\text{for }i\neq j,\\[2mm]
\frac{1}{\beta} m_i\sum_{k\neq i} m_k \mathsf{K}_\beta''\left(\theta_i-\theta_k\right)
+\frac{1}{2}\,m_i\,\partial_\theta^2\left(\mathsf{v}_\upvartheta\circ x\right)\left(\theta_i\right),\qquad&\text{for }i= j,
\end{cases}
\]
is positive semi-definite.

Fix $\lambda\in(0,1)$, and consider a cluster of $n\in\llbracket 2,N\rrbracket$ atoms satisfying the distance condition
\begin{equation}\label{eq:dist.cond}
\min_{k\in\mathbb Z}\left|\theta_i-\theta_j+2\pi k\right| \le \frac{\lambda}{\sqrt{\beta}} \qquad \text{for all } i,j \in \llbracket1,n\rrbracket.
\end{equation}
We evaluate the quadratic form $\nabla^2 \mathsf{E}_{\beta,\upvartheta}\left(\boldsymbol{\theta}\right)\left[\xi,\xi\right]$ along the vector  $\xi\in\R^N$ defined by 
\[
\xi_i= \frac{1}{m_i}\quad \left(i<n\right),\qquad\xi_{n}=-\frac{n-1}{m_{n}},\qquad\xi_i = 0\quad \left(i > n\right).
\] 
Note that $\sum_{i\in\llbracket 1,n\rrbracket} m_i\xi_i=0$. We split
\begin{equation}\label{eq:Hess-split}
\nabla^2 \mathsf{E}_{\beta,\upvartheta}\left(\boldsymbol{\theta}\right)\left[\xi,\xi\right]
= \mathcal T_{\mathrm{in}} + \mathcal T_{\mathrm{out}} + \mathcal T_{\upvartheta},
\end{equation}
where
\begin{align}
\mathcal T_{\mathrm{in}}
&\coloneqq \frac{1}{\beta}\sum_{1\le i<j\le n} 
m_i m_j\,\mathsf K_\beta''\left(\theta_i-\theta_j\right)\left(\xi_i-\xi_j\right)^2,
\label{eq:T-in}\\[2mm]
\mathcal T_{\mathrm{out}}
&\coloneqq \frac{1}{\beta}\sum_{i\in\llbracket 1,n\rrbracket}\sum_{k>n} 
m_i m_k\,\mathsf K_\beta''\left(\theta_i-\theta_k\right)\xi_i^2,
\label{eq:T-out}\\[2mm]
\mathcal T_{\upvartheta}
&\coloneqq \frac{1}{2}\sum_{i\in\llbracket 1,n\rrbracket} 
m_i\,\partial_\theta^2\left(\mathsf v_\upvartheta\circ x\right)\left(\theta_i\right)\xi_i^2.
\label{eq:T-field}
\end{align}
Using the above choice of $\xi_i$ and the identity
\[
\sum_{1\le i<j\le n} m_i m_j\left(\xi_i-\xi_j\right)^2
= \sum_{i\in\llbracket 1,n\rrbracket} m_i\sum_{j\in\llbracket 1,n\rrbracket} m_j\xi_j^2
-\left(\sum_{i\in\llbracket 1,n\rrbracket} m_i\xi_i\right)^2,
\]
we can estimate
\begin{align*}
\mathcal T_{\mathrm{in}}
&\le \frac{1}{\beta}\max_{1\le i<j\le n}\mathsf K_\beta''\left(\theta_i-\theta_j\right)\sum_{i\in\llbracket 1,n\rrbracket} m_i\sum_{j\in\llbracket 1,n\rrbracket} m_j\xi_j^2.
\end{align*}
Using \eqref{eq:dist.cond} and the fact that $\lambda/\sqrt{\beta} \sim \lambda\theta_c(\beta)$ for large $\beta$, we can directly apply \eqref{eq:kernel-hess-bound-lambda} from \Cref{lem: concavity}. For all $\beta$ sufficiently large, this yields
\[
\max_{1\le i<j\le n}\mathsf K_\beta''\left(\theta_i-\theta_j\right)
\le \mathsf K_\beta''\left(\frac{\lambda}{\sqrt{\beta}}\right)
\le -\frac{1}{2}\left(1-\lambda^2\right)\beta\,e^{\beta-\lambda^2/2}.
\]
Substituting this bound into the definition of $\mathcal T_{\mathrm{in}}$, we obtain
\begin{equation}\label{eq:T-in-bound}
\mathcal T_{\mathrm{in}}\le -\frac{e^{\beta-\frac{\lambda^2}{2}}\left(1-\lambda^2\right)}{2}
\sum_{i\in\llbracket 1,n\rrbracket} m_i\sum_{j\in\llbracket 1,n\rrbracket} m_j\xi_j^2.
\end{equation}
Second, from \eqref{eq:kernel-hess-bound-max} we obtain, for large enough $\beta$, 
\begin{align}\label{eq:T-out-bound}
\mathcal T_{\mathrm{out}}
\le  2 e^{\beta-\frac{3}{2}}\sum_{k>n}m_k\sum_{i\in\llbracket 1,n\rrbracket} m_i\xi_i^2.
\end{align}
Finally, since $\sigma$ is globally Lipschitz and $\|x\|=\|x'\|=\|x''\|=1$, we can compute for a.e. $\theta$:
\begin{align*}
    \frac12\partial_\theta^2(\mathsf v_\upvartheta\circ x)(\theta)
&\!= \!\!\sum_{j\in\llbracket 1,2\rrbracket} \omega_j\left(\sigma'\left(a_j\cdot x(\theta)\right)(a_j\cdot x'(\theta))^2
+ \sigma\left(a_j\cdot x(\theta)\right)a_j\cdot x''(\theta)\right)\nonumber\\
&\!\le \sum_{j\in\llbracket1,2\rrbracket} |\omega_j|
\left(\|\sigma\|_{C^{0,1}}\|a_j\|^2 + (|\sigma(0)|+\|\sigma\|_{C^{0,1}}\|a_j\|)\,\|a_j\|\right)
\! \eqqcolon C_\upvartheta.
\end{align*}
Therefore
\begin{equation}\label{eq:T-field-bound}
\mathcal T_{\upvartheta}
\le C_\upvartheta\sum_{i\in\llbracket 1,n\rrbracket} m_i\xi_i^2.
\end{equation}
Combining \eqref{eq:Hess-split}, \eqref{eq:T-in-bound}, \eqref{eq:T-out-bound},  and \eqref{eq:T-field-bound}, we obtain an upper bound for the Hessian:
\begin{align*}
\nabla^2 \mathsf{E}_{\beta,\upvartheta}(\boldsymbol{\theta})[\xi,\xi]
\le 
\Biggl[
\tfrac{e^{\beta-\frac{\lambda^2}{2}}(\lambda^2-1)}{2}\sum_{i\in\llbracket 1,n\rrbracket} m_i
+ 2 e^{\beta-\frac{3}{2}}\sum_{k>n}m_k
+ C_\upvartheta
\Biggr] \sum_{j\in\llbracket 1,n\rrbracket} m_j\xi_j^2.
\end{align*}
For the second order condition \eqref{eq: Hess.pos.def} to hold, the right-hand side cannot be negative. Therefore:
\[
\frac{e^{\beta-\frac{\lambda^2}{2}}\left(1-\lambda^2\right)}{2}
\sum_{i\in\llbracket 1,n\rrbracket} m_i
\le 2 e^{\beta-\frac{3}{2}}\left(1-\sum_{i\in\llbracket 1,n\rrbracket} m_i\right)
+ C_\upvartheta.
\]
Solving for $\sum_{i\in\llbracket1, n\rrbracket} m_i$, we get
\[
\sum_{i\in\llbracket 1,n\rrbracket} m_i\le \frac{2e^{\beta-\frac{3}{2}}+C_\upvartheta}{2e^{\beta-\frac{3}{2}}+\frac{e^{\beta-\frac{\lambda^2}{2}}\left(1-\lambda^2\right)}{2}}
=\frac{4}{4+e^{\frac{3}{2}-\frac{\lambda^2}{2}}\left(1-\lambda^2\right)}+O\left(e^{-\beta}\right).
\]
Setting $\lambda = 1/2$, this bound yields \eqref{eq: conclusion}.

For the second part, assume that $\mu$ is itself a single cluster satisfying the pairwise distance condition (so $n=N$ and $\sum_{k>n}m_k=0$).
Then the upper bound for the Hessian gives:
\begin{equation*}
0 \le \nabla^2 \mathsf{E}_{\beta,\upvartheta}\left(\boldsymbol{\theta}\right)\left[\xi,\xi\right]
\le \left[-\frac{e^{\beta-\frac{\lambda^2}{2}}\left(1-\lambda^2\right)}{2} + C_\upvartheta\right] \sum_{j\in\llbracket 1,n\rrbracket} m_j\xi_j^2,
\end{equation*}
hence necessarily
\begin{equation*}
C_\upvartheta \ \ge\ \frac{e^{\beta-\frac{\lambda^2}{2}}\left(1-\lambda^2\right)}{2}.
\end{equation*}
With $\lambda=1/2$ this reads 
\begin{equation*}
C_\upvartheta \ge \frac{3}{8}e^{\beta-\frac{1}{8}}.
\end{equation*}
Under the standing assumptions $\|\sigma\|_{\mathscr{C}^{0,1}}=1$ and $\sigma(0)=0$, we have $C_\upvartheta = 2\sum_{j\in\llbracket 1,2\rrbracket} |\omega_j|\cdot\|a_j\|^2$. Therefore, if
\begin{equation*}
|\omega_1|\cdot\|a_1\|^2+|\omega_2|\cdot\|a_2\|^2 < \frac{3}{16}e^{-\frac{1}{8}} < 0.16547,
\end{equation*}
then $C_\upvartheta < \frac{3}{8}e^{-\frac{1}{8}} \le \frac{3}{8}e^{\beta-\frac{1}{8}}$ for all $\beta>0$. This excludes $\mathrm{supp}\,\mu$ being a single such cluster.

\end{proof}

\begin{remark}\label{rem: thm.bound.multid}
    The restriction of \Cref{thm: bound} to $d=2$ is technical: our proof uses the one-dimensional angular parametrization of $\S^1$ and the scalar concavity estimate in \Cref{lem: concavity}; extending it to $d\ge3$ would require a higher-dimensional Hessian estimate on small geodesic caps.
\end{remark}

\begin{proof}[Proof of \Cref{cor: bound}]
Let $\Lambda_\beta$ denote the exact mass bound derived in \Cref{thm: bound} for a cluster of diameter $\frac{1}{2\sqrt{\beta}}$, i.e.,
\[
\Lambda_\beta \coloneqq 0.5742+ O(e^{-\beta}).
\]
We define a covering of the support of $\mu$. For each of the arcs $I_j$ of length $|I_j| \le L$, we partition it into $K_j$ disjoint sub-intervals $J_{j,1}, \dots, J_{j, K_j}$, each of length at most $1/(2\sqrt{\beta})$. The minimal number of such intervals required is
\[
K_j \coloneqq \max\left\{1,\left\lceil 2|I_j|\sqrt{\beta}\right\rceil\right\}
\le 1+2|I_j|\sqrt{\beta}\le 1+2L\sqrt{\beta}.
\]
Consider any such sub-interval $J = J_{j,k}$. Let $n_\varepsilon(J)\coloneqq \#\left\{i: \theta_i\in J,\ m_i\ge \varepsilon\right\}$.
 We distinguish cases based on the number of atoms in $J$:
\begin{itemize}
    \item If $J$ contains at least two atoms, they form a cluster satisfying the pairwise distance condition of \Cref{thm: bound} (since the diameter of $J$ is $\leq 1/(2\sqrt{\beta})$). Thus, $\mu(J) \le \Lambda_\beta$.
    \item If $J$ contains 0 or 1 atom, then obviously $n_\varepsilon(J) \le 1$.
\end{itemize}
Combining these, if $\varepsilon \le \Lambda_\beta$, we have $1 \le \Lambda_\beta/\varepsilon$. Therefore, in the first case $n_\varepsilon(J) \cdot \varepsilon \le \mu(J) \le \Lambda_\beta$, and in the second case $n_\varepsilon(J) \le 1 \le \Lambda_\beta/\varepsilon$. Hence, for any sub-interval, $n_\varepsilon(J) \le \Lambda_\beta/\varepsilon$.

Since every atom with mass $\ge \varepsilon$ belongs to at least one sub-interval $J_{j,k}$, we have
\[
N_\varepsilon \le \sum_{j\in\llbracket 1,M\rrbracket} \sum_{k\in\llbracket 1,K_j\rrbracket}n_\varepsilon(J_{j,k}) \le \sum_{j\in\llbracket 1,M\rrbracket} K_j \frac{\Lambda_\beta}{\varepsilon}.
\]
Substituting the bound for $K_j$, we obtain:
\[
N_\varepsilon \le M \left(1 + 2L\sqrt{\beta}\right) \frac{\Lambda_\beta}{\varepsilon}.
\]
Finally, if $\varepsilon > \Lambda_\beta$, we use the trivial bound $N_\varepsilon \le 1/\varepsilon$. For $\beta$ large enough such that $M\left(1+2L\sqrt{\beta}\right)\Lambda_\beta \ge 1$, the bound in the statement is larger than $1/\varepsilon$ and thus remains valid.
\end{proof}

\subsection{Proof of \Cref{prop: min.max}}

\begin{proof}[Proof of \Cref{prop: min.max}]
Since $e^{\beta x\cdot y}\le e^{\beta}$ on $\S^{d-1}$, we have for any $\mu$,
\[
\mathsf E_{\beta,\upvartheta}[\mu]
=\frac{1}{2\beta}\iint e^{\beta x\cdot y}\diff \mu(x)\diff \mu(y)+\frac12\int \mathsf v_\upvartheta\diff\mu
\le \frac{e^{\beta}}{2\beta}+\frac12\max_{x\in\S^{d-1}}\mathsf v_\upvartheta(x),
\]
and equality holds iff $\mu=\delta_x$ with $x\in\argmax_{y\in\S^{d-1}}\sum_{j\in\llbracket1,d\rrbracket}\omega_j\,\upvarphi\left(a_j\cdot y\right).$
This proves the characterization of global maximizers.

Existence of the minimizer follows from compactness of $\mathcal P(\S^{d-1})$ and continuity of the integrands. To see uniqueness, note that the kernel $e^{\beta x\cdot y}$ is conditionally strictly positive definite: for any finite signed measure $\nu$ with $\nu(\S^{d-1})=0$,
\[
\iint e^{\beta x\cdot y}\diff\nu(x)\diff\nu(y)=\sum_{k\ge 1}\lambda_k(\beta)\sum_{\ell\in\llbracket1,N_k\rrbracket}\left( \int Y_{k\ell}\diff\nu\right)^{ 2}>0 \quad\text{if }\nu\neq 0.
\]
Thus $\mathsf E_\beta$ is strictly convex on $\mathcal P(\S^{d-1})$, see \cite[Prop.~2.1 \& Thm.~4.1]{BILYK2022126220}. Adding the linear term $\frac12\int \mathsf v_\upvartheta\diff\mu$ preserves strict convexity; thus the minimizer is unique.


Finally, $\mathsf E_\beta[R_\#\mu]=\mathsf E_\beta[\mu]$ for all $R\in O(d)$, and if moreover $Ra_j=a_j$ for all $j$, then $$\mathsf v_\upvartheta(Rx)=\sum_{j\in\llbracket1,d\rrbracket}\omega_j\,\upvarphi\left(a_j \cdot  Rx\right)
=\sum_{j\in\llbracket1,d\rrbracket}\omega_j\,\upvarphi\left(Ra_j \cdot  x\right)
=\mathsf v_\upvartheta(x).$$ Hence $\mathsf E_{\beta,\upvartheta}[R_\#\mu]=\mathsf E_{\beta,\upvartheta}[\mu]$ for all such $R$, and by uniqueness $R_\#\mu_\ast=\mu_\ast$.
\end{proof}

\begin{remark}\label{rem:strictSOPD-perceptron}
It is clear that the minimizer $\mu_\ast$ is SOPD. Under additional curvature from the perceptron potential, one can ensure strictness. For \(d=2\), \eqref{eq:Hess-bilinear} yields for all $\xi \in \mathsf T_\mu\mathcal P(\mathbb S^1)$:
\begin{align*}
\mathrm{Hess}_\mu \mathsf E_{\beta,\upvartheta}\left(\xi,\xi\right)
\ =&\int \left[\frac{1}{\beta}\left(\mathsf K_\beta''*\mu\right)\left(\theta\right)+\frac{1}{2}\partial_\theta^2\left(\mathsf v_{\upvartheta}\circ x\right)\left(\theta\right)\right]\xi\left(\theta\right)^2\diff\mu(\theta) \\
&-\frac{1}{\beta}\iint \mathsf K_\beta''\left(\theta-\phi\right)\xi\left(\theta\right)\xi\left(\phi\right)\diff\mu(\theta)\diff\mu(\phi) \\
\ge& \left(\frac{1}{2}\inf_{\theta\in\supp\mu}\partial_\theta^2\left(\mathsf v_\upvartheta\circ x\right)\left(\theta\right)-\frac{\|\mathsf K_\beta''\|_{\infty}}{\beta}\right) \|\xi\|_{L^2(\mu)}^2,
\end{align*}
where we used $$\iint \left(\xi(\theta)-\xi(\phi)\right)^2\diff\mu(\theta)\diff\mu(\phi) = 2\|\xi\|_{L^2(\mu)}^2 - 2\left(\int \xi\diff\mu\right)^2 \le 2\|\xi\|_{L^2(\mu)}^2.$$
Consequently, a sufficient condition for \(\mu\) to be strict SOPD is: there exists \(\kappa>0\) such that
\begin{equation}\label{eq:strictSOPD-sufficient}
\partial_\theta^2\left(\mathsf v_\upvartheta\circ x\right)\left(\theta\right) \ge \frac{2\|\mathsf K_\beta''\|_{\infty}}{\beta}+2\kappa
\qquad \left(\theta\in\supp\mu\right).
\end{equation}
If \(\mathsf K_\beta(\phi)=e^{\beta\cos\phi}\) then $\mathsf K_\beta''(\phi)=\left(\beta^2\sin^2\phi-\beta\cos\phi\right)e^{\beta\cos\phi},$ hence $\|\mathsf K_\beta''\|_{\infty}=\beta e^\beta$. Thus, for $\mathsf v_\upvartheta$ as in \eqref{eq: primitive.field}, condition \eqref{eq:strictSOPD-sufficient} simplifies to
\[
\sum_{j\in\llbracket 1,2\rrbracket} \omega_j\left(\sigma'\left(a_j\cdot x(\theta)\right)\left(a_j\cdot x'(\theta)\right)^2 + \sigma\left(a_j\cdot x(\theta)\right) a_j\cdot x''(\theta)\right)\ge e^\beta+\kappa,
\]
for $\theta\in\supp\mu$.

An analogous sufficient condition holds for $d\ge3$, replacing $\partial_\theta^2\left(\mathsf v_\upvartheta\circ x\right)$ with the minimum eigenvalue of $\nabla^2\mathsf v_\upvartheta$, and $\|\mathsf K_\beta''\|_{\infty}$ with the constant $C_{\beta,d}<\infty$ such that $\mathrm{Hess}_\mu \mathsf E_{\beta}(\xi,\xi)\ge -C_{\beta,d}\|\xi\|_{L^2(\mu)}^2$ for all $\mu$ and $\xi\in\mathsf T_\mu\mathcal P(\S^{d-1})$.
\end{remark}

\section{Normalized self-attention}\label{ss:normalized}

As noted in \Cref{s:setup}, the normalized attention field can be viewed as a weighted gradient. The sparsity results obtained for the unnormalized case extend to this setting, provided we impose a mild non-degeneracy condition on the perceptron weights.

\begin{proposition}\label{prop: unified.log}
Let $d \ge 2$, $\beta > 0$, fix $\sigma(s)=s_+$, and let $\mu\in\mathcal P(\S^{d-1})$ be a stationary measure for \eqref{eq:full-WGF-main}, i.e., satisfying \eqref{eq: fulltrans.stat}. Assume that $(\omega_j,a_j)_j$ satisfy the following non-degeneracy condition: for every index subset $J \subseteq \llbracket 1,d \rrbracket$, the matrix
\begin{equation}\label{eq: nondeg.normalized}
M\coloneqq\sum_{j\in J}\omega_j\,a_j a_j^\top - \sum_{j\notin J}\omega_j\,a_j a_j^\top
\end{equation}
is not a scalar multiple of the identity. Then $\sigma_d(\supp\mu)=0$; in particular, $\mu$ is singular with respect to $\sigma_d$. Moreover, if $d=2$, then $\mu$ is purely atomic with finitely many atoms.
\end{proposition}

\begin{proof}[Proof of \Cref{prop: unified.log}]
Because $\frac{\delta\mathsf{E}_\beta}{\delta\mu}[\mu]$ is strictly positive and real-analytic in $\S^{d-1}$, its logarithm is well-defined and real-analytic. On the other hand, on each connected component $I$ of the set $\S^{d-1}\setminus \mathscr Z$ we have
\begin{equation*}
\mathsf{v}_\upvartheta(x)=\sum_{j\in\llbracket 1,d\rrbracket}\omega_j\left( a_j\cdot x\right)_+^{2}=\sum_{j\in J_I}\omega_j\left( a_j\cdot x\right)^{2}\quad \text{for } x\in I,
\end{equation*}
where $\mathscr Z$ and $J_I$ are defined in \eqref{eq: defZd} and \eqref{eq: defJI}. Therefore
\[
H_{\log}\coloneqq \log \left(\frac{\delta\mathsf{E}_\beta}{\delta\mu}[\mu]\right)+\frac{1}{2}\mathsf{v}_\upvartheta,\qquad
g_{\log}\coloneqq\nabla H_{\log}
\]
are real-analytic on $I$. By \eqref{eq: fulltrans.stat}, we get $\supp\mu\subset\{g_{\log}=0\}$.

\medskip 

\noindent (a) Suppose, for contradiction, that $\sigma_d(\supp\mu\cap I)>0$ for some component $I$. As $g_{\log}$ is real-analytic on $I$ and vanishes on a set of positive measure, we have $g_{\log}\equiv0$ on $I$, hence $H_{\log}$ is constant on $I$: there exists $C_I\in\R$ such that
\[
\log \left(\frac{\delta\mathsf{E}_\beta}{\delta\mu}[\mu](x)\right)
= C_I-\frac{1}{2}\sum_{j\in J_I}\omega_j\left(a_j\cdot x\right)^2,\qquad x\in I.
\]
Define the global real-analytic vector field
\[
\widetilde g_{\log}(x)\coloneqq\nabla \left(\log \left(\frac{\delta\mathsf{E}_\beta}{\delta\mu}[\mu](x)\right)+\frac{1}{2}\sum_{j\in J_I}\omega_j\left(a_j\cdot x\right)^2\right).
\]
Since $\widetilde g_{\log}=g_{\log}\equiv0$ on $I$, the identity theorem for real-analytic functions on $\S^{d-1}$ implies $\widetilde g_{\log}\equiv0$ on $\S^{d-1}$. Therefore
\begin{equation}\label{eq:log-identity-global}
\frac{\delta\mathsf{E}_\beta}{\delta\mu}[\mu](x)
=\exp \left(C_I-\frac{1}{2}\sum_{j\in J_I}\omega_j\left(a_j\cdot x\right)^2\right),\qquad x\in\S^{d-1}.
\end{equation}
The right-hand side of \eqref{eq:log-identity-global} is even, so $\mu(A)=\mu(-A)$ for every Borel set $A\subset\S^{d-1}$, by \Cref{lem: quadpol}\textbf{(ii)}. Thus $\supp\mu$ meets the antipodal component $I'$ of $I$, whose active set is $J_{I'}=\llbracket 1,d \rrbracket\setminus J_I$. Repeating the argument on $I'$ yields
\[
\frac{\delta\mathsf{E}_\beta}{\delta\mu}[\mu](x)
=\exp \left(C_{I'}-\frac{1}{2}\sum_{j\in J_{I'}}\omega_j\left(a_j\cdot x\right)^2\right),\qquad x\in\S^{d-1}.
\]
Equating the two global expressions gives, for all $x\in\S^{d-1}$,
\begin{equation}\label{eq:ratio-const-correct}
\sum_{j\in J_I}\omega_j\left(a_j\cdot x\right)^2-\sum_{j\notin J_I}\omega_j\left(a_j\cdot x\right)^2
\equiv 2\left(C_I-C_{I'}\right).
\end{equation}
This means that $x^\top M x$ is constant on $\S^{d-1}$ for $M$ defined as in \eqref{eq: nondeg.normalized}, which forces $M$ to be a scalar multiple of the identity, contradicting our assumption. Thus $\sigma_d(\supp\mu\cap I)=0$ for every $I$. Since $\sigma_d(\partial I)=0$, we conclude $\sigma_d(\supp\mu)=0$.

\smallskip
\noindent (b) Let $d=2$ and suppose $\supp\mu$ is infinite. Then $\supp\mu\cap I$ has an accumulation point in $\S^1$ for some $I$. The same identity theorem applied to $\widetilde g_{\log}$ yields \eqref{eq:log-identity-global} and \eqref{eq:ratio-const-correct}, leading to the same contradiction. By compactness of $\S^1$, the support must be finite.
\end{proof}

\begin{remark}\label{rem:unifiedlog-B}
The argument above also applies to normalized attention with a symmetric nonsingular matrix $B\in\R^{d\times d}$ by replacing $\mathsf E_\beta$ with $$\mathsf E_B[\mu]\coloneqq \frac{1}{2}\iint e^{x^\top B y}\diff\mu(x)\diff\mu(y).$$ Since $\frac{\delta\mathsf E_B}{\delta\mu}[\mu]$ is strictly positive and real-analytic, the inclusion $\supp\mu\subset\{\nabla H_{\log}=0\}$ still leads to the global identity \eqref{eq:log-identity-global}. By \Cref{rem: general-attention}, this implies that $\mu(A)=\mu(-A)$ for every Borel set $A\subset\S^{d-1}$, which yields the same contradiction via \eqref{eq:ratio-const-correct}. A similar argument extends \Cref{thm: circle,thm: any.d} to $\mathsf E_B$ by removing the exponential from the analogous identities.
\end{remark}

\bibliographystyle{plain}
\bibliography{biblio}

	\bigskip
    \bigskip

\begin{minipage}[!h]{.5\textwidth}
  {\footnotesize{\bf Antonio \'Alvarez-L\'opez}\par
  Departamento de Matem\'aticas\par
  Universidad Aut\'onoma de Madrid\par
  C/ Francisco Tom\'as y Valiente 7, \par
  28049 Madrid, Spain\par
 \par
  e-mail: \href{mailto:blank}{\textcolor{blue}{\scriptsize antonio.alvarezl@uam.es}}
  }
\end{minipage}
\begin{minipage}[!h]{.5\textwidth}
{\footnotesize{\bf Borjan Geshkovski}\par
  Laboratoire Jacques-Louis Lions\par
  Inria \& Sorbonne Université\par
  4 Place Jussieu\par
  75005 Paris, France\par
 \par
  e-mail: \href{mailto:borjan@mit.edu}{\textcolor{blue}{\scriptsize borjan.geshkovski@inria.fr}}
  }
\end{minipage}%

\begin{center}
\begin{minipage}[!h]{.5\textwidth}
  {\footnotesize{\bf Domènec Ruiz-Balet}\par
  Departament de Matemàtiques i Informàtica\par
  Universitat de Barcelona\par
  Gran Via de les Corts Catalanes 585\par
  08007 Barcelona, Spain\par
 \par
  e-mail: \href{mailto:blank}{\textcolor{blue}{\scriptsize domenec.ruizibalet@ub.edu}}
  }
\end{minipage}%
\end{center}

\end{document}